\documentclass[pdflatex,sn-basic]{sn-jnl}

\usepackage{graphicx}%
\usepackage{multirow}%
\usepackage{amsmath,amssymb,amsfonts}%
\usepackage{amsthm}%
\usepackage[title]{appendix}%
\usepackage{xcolor}%
\usepackage{textcomp}%
\usepackage{manyfoot}%
\usepackage{booktabs}%
\usepackage{array}%
\usepackage{algorithm}%
\usepackage{algorithmicx}%
\usepackage{algpseudocode}%
\usepackage{listings}%

\usepackage{framed}
\usepackage{picinpar}

\theoremstyle{thmstyleone}%
\newtheorem{theorem}{Theorem}

\newtheorem{lemma}[theorem]{Lemma}
\newtheorem{remark}{Remark}%
\newtheorem{definition}{Definition}%
\newtheorem{assumption}{Assumption}%

\begin{document}

\title[A Mechanism Study of Delayed Loss Spikes in \\ Batch-Normalized Linear Models]{A Mechanism Study of Delayed Loss Spikes in \\ Batch-Normalized Linear Models}


\author[1]{\fnm{Peifeng} \sur{Gao}}\email{gaopeifeng@connect.hku.hk}

\author[2]{\fnm{Wenyi} \sur{Fang}}\email{fangwenyi3@huawei.com}

\author[2]{\fnm{Yang} \sur{Zheng}}\email{zhengyang31@huawei.com}

\author[1]{\fnm{Difan} \sur{Zou}}\email{dzou@cs.hku.hk}

\affil[1]{\orgdiv{School of Computing and  Data Science}, \orgname{The University of Hong Kong}
}

\affil[2]{\orgdiv{RAMS Technologies Lab}, \orgname{Huawei Technologies Ltd}}


\abstract{
    Delayed loss spikes have been reported in neural-network training, but existing theory mainly explains earlier non-monotone behavior caused by overly large fixed learning rates. We study one stylized hypothesis: normalization can postpone instability by gradually increasing the effective learning rate during otherwise stable descent. To test this hypothesis at theorem level, we analyze batch-normalized linear models. Our flagship result concerns whitened square-loss linear regression, where we derive explicit no-rising-edge and delayed-onset conditions, bound the waiting time to directional onset, and show that the rising edge self-stabilizes within finitely many iterations. Combined with a square-loss decomposition, this yields a concrete delayed-spike mechanism in the whitened regime. For logistic regression, under highly restrictive active-margin assumptions, we prove only a supporting finite-horizon directional precursor in a knife-edge regime, with an optional appendix-only loss lower bound under an extra non-degeneracy condition. The paper should therefore be read as a stylized mechanism study rather than a general explanation of neural-network loss spikes. Within that scope, the results isolate one concrete delayed-instability pathway induced by batch normalization.
}

\keywords{Loss Spikes, Normalization Method, Linear Model, Large Learning Rate}



\maketitle
\hypersetup{hypertexnames=false}

\section{Introduction}

Delayed loss spikes have been reported in large-scale neural-network training; see, for example, \citet{spike_GLM}, \citet{spike_PaLM}, and \citet{Spike_A_theory_on}. These events can require many additional iterations to recover and, in severe cases, are followed by collapse or restart from an earlier checkpoint. Recent mitigation-oriented work, including \citet{kumar2025zclip} and \citet{wang2025adagc}, further underscores that spike suppression remains a practical concern in modern large-model pretraining. We use these papers only as operational motivation, not as direct theoretical predecessors and not as empirical validation of the mechanism analyzed below. The goal of this paper is narrower: to test, in a tractable normalized linear setting, whether batch normalization can already create one delayed-instability pathway.

One influential line of work is the Edge of Stability (EoS) phenomenon \citep{First_EOS}, which studies optimization trajectories that remain stable even when the learning rate exceeds classical local-smoothness thresholds. In this regime, the loss can oscillate while still decreasing on longer timescales. Subsequent theory explains such behavior for carefully chosen scalar objectives \citep{eos_minimalist,eos_monotonically_decreases_the_sharpness} or under restrictive structural assumptions \citep{ahn2022understanding,ma2022beyond}. These results provide important insight into non-monotone training dynamics, but they do not explain the delayed spike pattern reported in large-scale practice, where instability often appears only after the loss has already become relatively small.

A more directly related line of work studies gradient descent with learning rates beyond inverse smoothness in simple linear or neural-network models \citep{eos_logistic_regression,andriushchenko2023sgd,wu2024large,lu2023benign}. These results show that stable, and sometimes accelerated, convergence can still occur. However, the instability predicted there typically appears near the beginning of training; see, in particular, \citet{eos_logistic_regression}. By contrast, the empirical spikes that motivate this paper can emerge much later, after a long period of apparently stable descent. Existing theory in this direction also does not identify explicit trigger conditions for delayed spikes or characterize their subsequent shape.

Our starting point is the empirical observation that many reported loss spikes arise in models equipped with normalization layers, such as weight normalization \citep{WN}, batch normalization \citep{BN}, and layer normalization \citep{LN}. At the same time, prior theory shows that normalization can induce effective learning-rate auto-tuning or, more broadly, scale-adaptive dynamics; see \citet{arora2018theoretical}, \citet{hoffer2018norm}, and \citet{morwani2022inductive}. This suggests a plausible mechanism for delayed instability: normalization can stabilize early training while gradually amplifying the effective learning rate until directional instability is triggered.

To investigate this mechanism, we study the simplest setting in which it can be analyzed cleanly: linear models with batch normalization trained by full-batch gradient descent. The scope of our results is intentionally asymmetric. Our main theorem line treats whitened square-loss linear regression and gives a comparatively strong characterization of the delayed-instability mechanism, including explicit no-rising-edge and delayed-onset conditions, waiting-time control, and finite-time self-stabilization of the rising edge. Through the square-loss decomposition in Lemma \ref{lemma_Decompositionofrisk}, this yields a concrete delayed-spike interpretation in the whitened regime. The logistic-regression analysis is narrower: under additional restrictive active-margin assumptions, it proves only a finite-horizon directional precursor, with an optional appendix-only non-degeneracy condition used solely for a loss lower bound. The paper should therefore be read as a mechanism study with one flagship theorem regime and one supporting extension. Relative to prior theory, the closest formal comparator for the flagship linear result is \citet{arora2018theoretical}, whose theorem is broader on generic batch-normalized convergence whereas ours is narrower but more explicit on delayed onset and self-stabilization. By contrast, \citet{eos_logistic_regression} mainly provides a mechanism contrast because its instability is early and fixed-scale, while \citet{implicit_bias_BN} provides a scope contrast because our logistic result does not strengthen the asymptotic max-margin guarantee.

The main contributions are as follows:




\begin{itemize}
    \item For whitened batch-normalized linear regression, we identify explicit sufficient conditions for no-rising-edge and delayed-onset regimes, bound the waiting time and duration of the \textit{Rising Edge}, and show that the instability self-stabilizes because directional divergence drives the effective learning rate back down. Via Lemma \ref{lemma_Decompositionofrisk}, this yields a concrete delayed-spike mechanism in the whitened square-loss setting.

    \item We derive a scale-invariant directional convergence/divergence lemma for normalized parameterizations. This is the common technical bridge in both analyses, but the resulting theorem payloads remain asymmetric: a full delayed-onset/self-stabilization theorem line in the whitened linear regime and only a supporting finite-horizon precursor statement in logistic regression.

    \item For batch-normalized logistic regression, we do not prove a full spike theorem. Instead, under highly restrictive active-margin assumptions, we characterize a supporting finite-horizon directional precursor: the directional error first contracts to a small threshold, and on a later positive-alignment branch there is an explicit exit threshold beyond which the directional rising edge cannot persist. This identifies how margin, conditioning, learning rate, and initialization interact in that narrow stylized regime.
\end{itemize}

\noindent\textbf{Scope and limitations.}
The strongest theorem in this paper is the whitened square-loss linear-regression result, which gives a full directional delayed-onset/self-stabilization picture in a deliberately stylized regime together with a square-loss spike interpretation. The logistic-regression analysis is narrower: the active-margin assumption enforces a knife-edge max-margin geometry, and Assumption \ref{Non_degenerate_data} appears only in the appendix as an optional bridge from direction to a logistic-loss lower bound. The numerical section is included only as a qualitative mechanism illustration, not as comprehensive validation in modern nonlinear architectures. Section \ref{ProblemSetup} introduces the notation, Sections \ref{Main_Results} and \ref{ProofOverview} present the main theorem line and proof intuition, and the appendices contain the full technical details.


\section{Related Work}

\noindent\textbf{Theoretical Studies on Normalization.}  
Normalization layers such as batch normalization \citep{BN} and layer normalization \citep{LN} substantially reshape optimization dynamics. For batch normalization, \citet{arora2018theoretical}, \citet{hoffer2018norm}, and \citet{morwani2022inductive} highlight effective learning-rate auto-tuning and related scale-adaptive effects, while \citet{kohler2019exponential} and \citet{cai2019quantitative} quantify how normalization alters optimization geometry and convergence rates. Recent theory further studies mean-field signal propagation, representation orthogonalization, depth scaling, expressivity, and layer-normalization-induced nonlinearity \citep{yang2019mean_field_bn,daneshmand2021bn_orthogonalizes,meterez2024bn_without_gradient_explosion,burkholz2024bn_universal,ni2024layernorm_nonlinearity}. Recent systems papers such as \citet{wang2025sdd} and \citet{zhuo2025hybridnorm} revisit normalization and scale control from a stability perspective at large-model scale. We use these works mainly as context. The closest formal comparator for our flagship theorem is \citet{arora2018theoretical}, but our focus is different: rather than asymptotic convergence under large steps, we study a transient delayed-instability mechanism in a narrower whitened linear regime.

On the logistic side, \citet{implicit_bias_BN} study batch-normalized logistic regression and prove convergence toward the max-margin direction; our supporting logistic theorem does not strengthen that asymptotic result, but instead isolates a narrower finite-horizon directional regime under stronger structural assumptions.

\noindent
\textbf{Edge of Stability.} 
Edge of Stability (EoS) \citep{First_EOS} refers to the regime in which optimization remains stable even when the learning rate exceeds the classical threshold $2/\lambda_{\max}$, with $\lambda_{\max}$ denoting the top Hessian eigenvalue. Instead of diverging immediately, the loss may oscillate while continuing to decrease overall. A substantial follow-up literature explains this behavior for specific models and objectives \citep{eos_minimalist, chen2022gradient, ahn2023learning, even2023s, ma2022beyond, ahn2022understanding, eos_Self_stabilization, wang2022analyzing}. More recent work connects large-step dynamics to deep-network EoS theory, curvature-aware tuning, learning-rate transfer across scales, outlier sensitivity, and large-step implicit regularization \citep{arora2022understanding_eos,roulet2024stepping_edge,noci2024super_consistency,rosenfeld2024outliers_opposing_signals,qiao2024stable_minima}. These works clarify important mechanisms behind non-monotone training, but the analyzed objectives are often far from modern normalized networks, or else rely on assumptions that are difficult to verify in practice.

\citet{eos_logistic_regression} are especially relevant here: they show that logistic regression can converge to the max-margin direction under arbitrary constant learning rates, while large learning rates generate oscillatory loss behavior. Our work differs mainly in mechanism: their oscillations arise early in training, whereas the empirical spikes that motivate this paper typically appear only after a long stable phase \citep{Spike_Spike_No_More, Spike_A_theory_on}. We therefore focus on whether batch normalization can create a delayed-instability mechanism that is absent from plain logistic regression with a fixed scale.

\section{Problem Setup}\label{ProblemSetup}
Because we study both linear regression and logistic regression, we first introduce a unified notation. 
Let $n$ denote the number of samples and $d$ the feature dimension. 

\noindent\textbf{Dataset.}
The training dataset is $S = \{(\mathbf{x}_i, y_i)\}_{i=1}^{n}$, 
where $\mathbf{x}_i \in \mathbb{R}^{d}$ is the feature vector and $y_i \in \{1, -1\}$ is the corresponding label. 
We use the following matrix notation:
\begin{equation*}
\begin{aligned}
  & \mathbf{X} = [\mathbf{x}_1, \cdots, \mathbf{x}_n] \in \mathbb{R}^{d \times n}; \ 
  \mathbf{y} = [y_1, \cdots, y_n] \in \mathbb{R}^{n \times 1}; 
  \\ & 
  \tilde{\mathbf{X}} = \mathbf{X} \text{diag}\left(\mathbf{y}\right); \ 
  \mathbf{\Sigma} = \frac{1}{n} \tilde{\mathbf{X}} \tilde{\mathbf{X}}^T, \ 
  \boldsymbol{\mu} = \frac{1}{n} \tilde{\mathbf{X}} \mathbf{1}_{n},
\end{aligned}
\end{equation*}
where $\mathbf{1}_{n}$ denotes the $n$-dimensional all-ones vector. 
We also use the inner product and norm induced by $\mathbf{\Sigma}$:
for any $\mathbf{a}, \mathbf{b} \in \mathbb{R}^{d}$, we define
\begin{equation*}
\begin{aligned}
  \langle \mathbf{a}, \mathbf{b} \rangle_{\mathbf{\Sigma}} = \mathbf{a}^T \mathbf{\Sigma} \mathbf{b}; \ 
  \Vert \mathbf{a} \Vert_{\mathbf{\Sigma}}^2 = \mathbf{a}^T \mathbf{\Sigma} \mathbf{a}. 
\end{aligned}
\end{equation*}
Unless marked by the subscript $\mathbf{\Sigma}$, the directional quantities used for geometric intuition in the main text are Euclidean. We introduce $\mathbf{\Sigma}$-weighted variants only when the covariance geometry is essential to a bound.
\noindent\textbf{ Linear Model and Risk.}
Following \citet{implicit_bias_BN}, we consider the batch-normalized linear model
\begin{equation*}
\begin{aligned}
  \text{logit} \left(\mathbf{x}_i; \mathbf{w}, \alpha \right) = 
  \alpha \cdot \frac{\langle \mathbf{x}_i, \mathbf{w} \rangle }{\Vert \mathbf{w} \Vert_{\mathbf{\Sigma}}},
  \ \mathbf{w} \in \mathbb{R}^{d}, 
  \ \alpha \in \mathbb{R},
\end{aligned}
\end{equation*}
where $\mathbf{w}$ is the linear parameter and $\alpha$ is the batch-normalization scale. 
We use matrix notation throughout the analysis. For the logits of all samples, we write
\begin{equation*}
\begin{aligned}
  \left[\begin{array}{ccccccc}
    \text{logit} \left(\mathbf{x}_1; \mathbf{w}, \alpha \right),
    \cdots,
    \text{logit} \left(\mathbf{x}_n; \mathbf{w}, \alpha \right)
  \end{array}\right]^T
  =
  \alpha \cdot \frac{\mathbf{X}^T \mathbf{w}}{\Vert \mathbf{w} \Vert_{\mathbf{\Sigma}}}
\end{aligned}
\end{equation*}
We consider losses $\ell \in \{\ell_{log}, \ell_{squ}\}$, where
$\ell_{log} \left(\cdot\right) := \log\left(1 + \exp \left(-(\cdot)\right)\right)$.
Since the labels take values in $\{1, -1\}$, we define
$\ell_{squ} \left(\cdot\right) := {(1 - (\cdot))^2}/{2}$
so that the same notation covers both logistic regression and a square-loss linear surrogate. We write $\boldsymbol{\ell}\left(\cdot\right)$ for the element-wise version of the loss. The empirical risk is
\begin{equation*}
\begin{aligned}
  \mathcal{R} \left( \mathbf{w}, \alpha \right)
  =
  \frac{1}{n} \mathbf{1}_{n}^T \boldsymbol{\ell} \left( \alpha \cdot
    \frac{\tilde{\mathbf{X}}^T \mathbf{w}}{\Vert \mathbf{w} \Vert_{\mathbf{\Sigma}}} 
  \right),
\end{aligned}
\end{equation*}
\noindent\textbf{Gradient Descent.}
We study full-batch gradient descent with separate step sizes for $\mathbf{w}$ and $\alpha$. The corresponding gradients are
\begin{equation*}
\begin{aligned}
  & \nabla_{\mathbf{w}} \mathcal{R} \left(\mathbf{w}, \alpha\right) 
  = 
  \frac{\alpha}{n \Vert \mathbf{w} \Vert_{\mathbf{\Sigma}}}
  \left(
    \mathbf{I} - \frac{\mathbf{\Sigma} \mathbf{w} \mathbf{w}^{T}}{\Vert \mathbf{w} \Vert_{\mathbf{\Sigma}}^{2}}
  \right)
  \tilde{\mathbf{X}} 
  \boldsymbol{\ell^\prime}\left(
    \frac{
    \alpha \tilde{\mathbf{X}}^T \mathbf{w}
    }{\Vert \mathbf{w} \Vert_\mathbf{\Sigma}}
  \right);
  \\ & 
  \frac{\partial \mathcal{R}}{\partial \mathbf{\alpha}} \left(\mathbf{w}, \alpha\right)
  = 
  \frac{1}{n}
  \left(
    \frac{
      \tilde{\mathbf{X}}^{T} \mathbf{w}
      }{ \Vert \mathbf{w} \Vert_\mathbf{\Sigma}}
  \right)^T 
  \boldsymbol{\ell^\prime}\left(
    \frac{
    \alpha \tilde{\mathbf{X}}^T \mathbf{w}
    }{\Vert \mathbf{w} \Vert_\mathbf{\Sigma}}
  \right),
\end{aligned}
\end{equation*}
where $\boldsymbol{\ell^{\prime}}$ denotes the vector of element-wise derivatives of $\ell$.
We study the parameter sequence $(\mathbf{w}_t, \alpha_t)$ generated by
\begin{equation}\label{gradient_updata}
\begin{aligned}
  \mathbf{w}_{t+1} \leftarrow \mathbf{w}_{t} - \eta 
  \nabla_{\mathbf{w}} \mathcal{R}_t; \ \ 
  \alpha_{t+1} \leftarrow \alpha_{t} - \eta_{\alpha} 
  \frac{\partial \mathcal{R}_t}{\partial \alpha},
\end{aligned}
\end{equation}
starting from $(\mathbf{w}_0, \alpha_0)$, where $\mathcal{R}_t := \mathcal{R} \left(\mathbf{w}_t, \alpha_t\right)$ for brevity. Since $\nabla_{\mathbf{w}} \mathcal{R}_t$ always lies in $\text{span}\left(\mathbf{X}\right) := \text{span} \left\{\mathbf{x}_1, \cdots, \mathbf{x}_n\right\}$, only the component of $\mathbf{w}_t$ inside $\text{span}\left(\mathbf{X}\right)$ affects the dynamics. We therefore assume that the initialization lies in $\mathcal{X} := \text{span} \left(\mathbf{X}\right) \backslash \{0\}$, and, without loss of generality, take $\alpha_0 > 0$. 



\section{Main Results}\label{Main_Results}
We begin with the directional quantities that drive the paper's mechanism. In the whitened linear-regression regime, sharp directional divergence yields explicit no-rising-edge and delayed-onset conditions together with finite-time self-stabilization; Lemma \ref{lemma_Decompositionofrisk} then translates that event into the square-loss spike interpretation emphasized in the introduction. The logistic analysis below is included only to test whether a related directional precursor can still be proved beyond square loss under substantially stronger assumptions. Here, ``direction'' refers to the reference direction that gradient descent should ultimately approach in order to reduce the loss. We denote this direction by $\hat{\mathbf{w}}$: for linear regression it is the least-squares direction, while for logistic regression it is the max-margin SVM direction \citep{implicit_bias_BN,eos_logistic_regression}. Roughly speaking, when $\mathbf{w}_t$ remains aligned with $\hat{\mathbf{w}}$ and the scale $\alpha_t$ does not vary too abruptly, the loss continues to decrease. A rapid directional departure from $\hat{\mathbf{w}}$ can instead trigger a sudden loss increase. To quantify this effect, we introduce $\rho_t$ and $\rho_t^{\perp}$:
\begin{equation*}
\begin{aligned}
  & \rho_t 
  := \rho \left( \mathbf{w}_t \right) 
  := \frac{\langle \hat{\mathbf{w}}, \mathbf{w}_t \rangle}{\Vert \mathbf{w}_t \Vert};
  \quad
  \rho_t^{\perp} 
  := \rho^{\perp} \left( \mathbf{w}_t \right) 
  := \left\Vert
    \hat{\mathbf{w}} - \rho_t \frac{\mathbf{w}_t}{\Vert \mathbf{w}_t \Vert}
  \right\Vert.
\end{aligned} 
\end{equation*}
The left panel of Figure \ref{fig_demo} illustrates that $\rho \left( \mathbf{w} \right)$ is the component of $\hat{\mathbf{w}}$ along $\mathbf{w}$, whereas $\rho^{\perp} \left( \mathbf{w} \right)$ is the orthogonal residual. By the Pythagorean theorem,
$\left(\rho^{\perp} \left( \mathbf{w} \right)\right)^2 + \rho \left( \mathbf{w} \right)^2 = \Vert \hat{\mathbf{w}} \Vert^2$ for every $\mathbf{w}$. In particular, $\rho_t^{\perp} / \rho_t$ is the tangent of the angle between $\mathbf{w}_t$ and $\hat{\mathbf{w}}$, so the convergence of $\rho_t^{\perp} / \rho_t$ to $0$ means that $\mathbf{w}_t$ becomes increasingly aligned with the reference direction. We use these quantities first in linear regression and then, in a weaker form, in logistic regression. For ease of exposition, we introduce two system states according to the trend of $\rho_t^{\perp} / \rho_t$, as illustrated on the right of Figure \ref{fig_demo}:
\begin{itemize}
    \item \textit{Rising Edge}: the time span over which $\rho_t^{\perp} / \rho_t$ increases;
    \item \textit{Falling Edge}: the time span over which $\rho_t^{\perp} / \rho_t$ decreases.
\end{itemize}
\begin{figure}[t]
  \centering
  \begin{minipage}{0.38\columnwidth}
    \centering
    \includegraphics[width=\columnwidth]{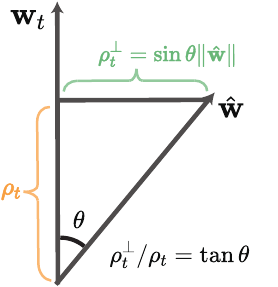}
  \end{minipage}\hfill
  \begin{minipage}{0.58\columnwidth}
    \centering
    \includegraphics[width=\columnwidth]{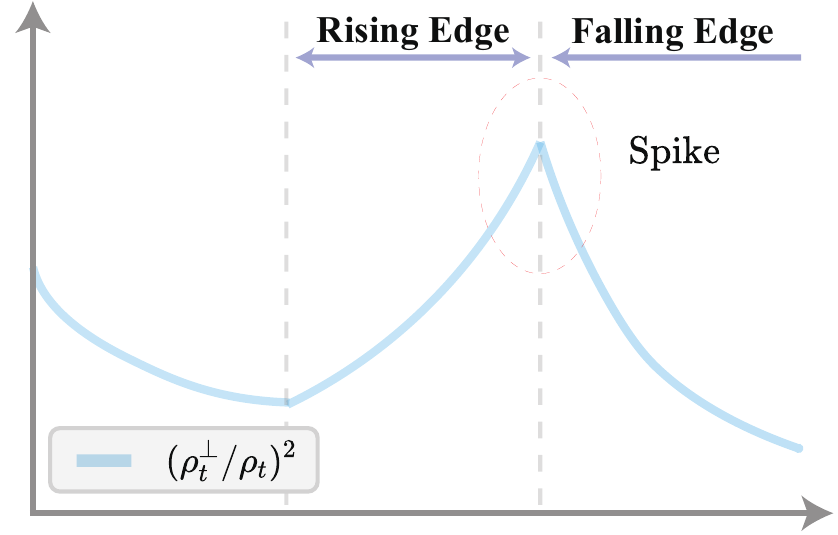}
  \end{minipage}
  \caption{\textbf{Left}: $\rho_t^{\perp}$ and $\rho_t^{\perp}/\rho_t$ measure the directional deviation between $\hat{\mathbf{w}}$ and $\mathbf{w}_t$. \textbf{Right}: an example trajectory of $\rho_t^{\perp} / \rho_t$, where the \textit{Rising Edge} and \textit{Falling Edge} are defined by its monotonic trend.}
  \label{fig_demo}
\end{figure}
Our goal is to characterize this directional onset/self-stabilization mechanism in detail, including its trigger condition and the durations of the rising and falling phases, and then interpret it in square loss through Lemma \ref{lemma_Decompositionofrisk}.

\subsection{Whitened Square-Loss Linear Regression}\label{Linear_Regression_Model}

We begin with whitened square-loss linear regression because it is the paper's strongest and most complete theorem regime. The reader can refer to Appendix \ref{Detailed_Proof_for_Linear_Regression} for full proofs. To simplify the analysis, we consider whitened data, so that the empirical covariance satisfies $\mathbf{\Sigma}=\mathbf{I}$. Relative to large-step convergence theory for normalized objectives and early-oscillation results without normalization, the present whitened setting is narrower but yields a more explicit delayed-instability package: no-rising-edge and delayed-onset conditions, a waiting-time bound, and finite-time self-stabilization of the rising edge. We start with a decomposition of the risk into directional components.
\begin{lemma}[Decomposition of Mean Square Loss]\label{lemma_Decompositionofrisk}
  Let $\ell$ be $\ell_{squ}$ and let $\hat{\mathbf{w}}$ be the least squares solution, 
  that is, $\hat{\mathbf{w}} := \mathbf{\Sigma}^{-1} \boldsymbol{\mu}$. Suppose $\mathbf{\Sigma} = I$. 
   Then, the following holds:
  \begin{equation*}
  \begin{aligned}
    \inf_{\mathbf{w}, \alpha} \mathcal{R}(\mathbf{w}, \alpha) = 
    1 - \left\Vert \hat{\mathbf{w}} \right\Vert^{2};
    \quad 
    \mathcal{R}_t = 
    (\alpha_t - \rho_t)^{2} + (\rho_t^{\perp})^2 + 
    1 - \left\Vert \hat{\mathbf{w}} \right\Vert^{2}.
  \end{aligned}
  \end{equation*}
\end{lemma}
This lemma reduces the square-loss question to the directional growth of $\rho_t^{\perp}$, or equivalently of $\rho_t^{\perp}/\rho_t$. It also shows that the square loss is already close to its optimum when both $|\alpha_t-\rho_t|$ and $\rho_t^{\perp}$ are small. We therefore analyze a near-alignment regime, encoded below by $\rho_{t_0}^{\perp}/\rho_{t_0} \le 1/\sqrt{3}$ and $0 < \alpha_{t_0} < \rho_{t_0}$, and ask when that regime can still transition into a delayed rising edge.


\begin{theorem}[Delayed Onset of the Rising Edge]\label{Occurance_of_Spike_ls}
  Let $\ell = \ell_{squ}$ be the square loss and $\hat{\mathbf{w}} = \mathbf{\Sigma}^{-1} \boldsymbol{\mu}$. Suppose $\mathbf{\Sigma} = \mathbf{I}$. 
  Consider the gradient descent (\ref{gradient_updata}) 
  for $t > t_0$ with $\eta_{\alpha} \in (0, 1)$, where $t_0$ is such that 
  $\rho_{t_0}^{\perp} / \rho_{t_0} \le 1/\sqrt{3}$ and $0 < \alpha_{t_0} < \rho_{t_0}$. The following results hold:
  \begin{enumerate}
    \item \textbf{Condition of No Rising Edge.} 
    If $\frac{\eta}{\| \mathbf{w}_{t_0} \|^2} < \frac{2}{\| \hat{\mathbf{w}} \|^2}$, then no rising edge occurs for any $t \ge t_0$.
    \item \textbf{Condition of Delayed Onset.}
    Let $k_{t_0} := \alpha_{t_0}/\rho_{t_0}$ and define
    \begin{equation*}
    \begin{aligned}
      C_{t_0} & :=
      \min \left(
        \frac{1}{\alpha_{t_0} \rho_{t_0}},
        \frac{3}{16\Vert \hat{\mathbf{w}} \Vert^2}
        \frac{\eta_\alpha}{e^{2} \left( 1 - k_{t_0} \right)}
      \right), \\
      \Delta T_0 & :=
      \left\lfloor
        \frac{1}{\eta_\alpha}
        \ln
        \left(
          \frac{\eta_\alpha \left(1-k_{t_0}\right)\Vert \mathbf{w}_{t_0}\Vert^2}{
            4 \eta \Vert \hat{\mathbf{w}} \Vert^2 \left(\rho_{t_0}^{\perp}/\rho_{t_0}\right)^2
          }
        \right)
        + 1
      \right\rfloor.
    \end{aligned}
    \end{equation*}
    If $\eta$ satisfies $\frac{8 }{\Vert \hat{\mathbf{w}} \Vert^2} 
        < 
        \frac{\eta}{\Vert \mathbf{w}_{t_0} \Vert^{2}} 
        \le 
        C_{t_0}$, then a rising edge starts within at most $\Delta T_0$ iterations.  
    Formally speaking, there exists $t_1 \in (t_0, t_0 + \Delta T_0]$ such that 
    \begin{equation*}
    \begin{aligned}
      & \frac{\rho_{t+1}^{\perp}}{\rho_{t+1}} \le \frac{\rho_{t}^{\perp}}{\rho_{t}} \ \ \forall t \in [t_0, t_1)
      \ \ \text{and} \ \ 
      \frac{\rho_{t_1+1}^{\perp}}{\rho_{t_1+1}} \ge \frac{\rho_{t_1}^{\perp}}{\rho_{t_1}}.
    \end{aligned}
    \end{equation*}
  \end{enumerate} 
\end{theorem}
Theorem \ref{Occurance_of_Spike_ls} has three immediate implications:
\begin{itemize}
    \item A delayed rising edge occurs when the effective learning rate $\eta/\|\mathbf{w}_{t_0}\|^2$ exceeds a threshold; otherwise the directional ratio continues to decrease.
    \item The waiting time scales as $-\log (\rho_{t_0}^{\perp} / \rho_{t_0})$, so delayed onset is more pronounced when the iterate is already better aligned with $\hat{\mathbf{w}}$. In the whitened square-loss model, Lemma \ref{lemma_Decompositionofrisk} shows that such near alignment is compatible with already-improved loss once $\alpha_{t_0}$ tracks $\rho_{t_0}$, even though the theorem itself is stated in geometric variables.
    \item The gap between the sufficient no-rising-edge threshold $2/\|\hat{\mathbf{w}}\|^2$ and the sufficient delayed-onset threshold $8/\|\hat{\mathbf{w}}\|^2$ is a proof-level gap rather than a separate claimed phase transition.
\end{itemize}
Given the delayed-onset condition in Theorem \ref{Occurance_of_Spike_ls}, we next characterize the subsequent rising edge for iterations $t \ge t_1$.
\begin{theorem}[Finite-Time Self-Stabilization of the Rising Edge]\label{Divergence_lr}
  Under the delayed-onset condition in Theorem \ref{Occurance_of_Spike_ls}, starting from the onset time $t_1$, the \textit{Rising Edge} lasts for at most
    \begin{equation*}
  \begin{aligned}
    \Delta T_1 = 
    \left\lceil
      \frac{1}{4}
      \frac{\Vert \hat{\mathbf{w}} \Vert^4 }{\alpha_{t_1}^2}
        \left(
          1 / (\rho_{t_1}^{\perp})^2 - 1 / \rho_{t_1}^2
        \right)^2
      \right\rceil
        +
        \left\lceil
      \frac{1}{4}
      \frac{\Vert \hat{\mathbf{w}} \Vert^2}{\rho_{t_1}^2}
      \left(
        \rho_{t_1} / \rho_{t_1}^{\perp} - \rho_{t_1}^{\perp} / \rho_{t_1}
      \right)^2
    \right\rceil
  \end{aligned}
  \end{equation*}
  iterations and turn to the \textit{Falling Edge}. 
  Specifically, there exists a $t_2 \in (t_1, t_1 + \Delta T_1]$ such that 
  \begin{equation*}
  \begin{aligned}
    \frac{\rho_{t+1}^{\perp}}{\rho_{t+1}} \ge \frac{\rho_{t}^{\perp}}{\rho_{t}} \ \ \forall t \in [t_1, t_2)
    \ \ \text{and} \ \ 
    \frac{\rho_{t_2+1}^{\perp}}{\rho_{t_2+1}} \le \frac{\rho_{t_2}^{\perp}}{\rho_{t_2}},
  \end{aligned}
  \end{equation*}
 Moreover, define a time $\phi \in [t_1, t_2]$ as the first iteration when $\alpha_t$ 
  catches up with $\rho_t$, \textit{i.e.}, the time such that 
  \( \alpha_t \le \rho_t \ \forall t \in [t_1, \phi] \) and \( \alpha_t \ge \rho_t \ \forall t \in (\phi, t_2) \). 
  Then, it holds that:
  \begin{equation*}
  \begin{aligned}
\forall t \in [t_1, \phi], \ \ 
    (\rho_{t}^{\perp}/\rho_{t})^2 
    \le 
    1 - 
    \frac{ 2 \rho_{t_1}^{\perp} \alpha_{t_1}}{\Vert \hat{\mathbf{w}} \Vert^2}
    \sqrt{t - t_1};
    \quad 
    \forall t \in (\phi, t_2], \ \ 
    (\rho_{t}^{\perp}/\rho_{t})^2 
    \le 
    1 - 
    \frac{ 2 \rho_{t_1}^{\perp}}{
      \Vert \hat{\mathbf{w}} \Vert
    }
    \sqrt{t - \phi}.
  \end{aligned}
  \end{equation*}
\end{theorem}
Theorem \ref{Divergence_lr} shows that the rising phase terminates after finitely many iterations and then returns to a falling phase, so the growth of $\rho_t^{\perp}$ remains bounded. In particular, the mechanism self-stabilizes within at most $\Delta T_1$ iterations and the peak value of $\rho_t^\perp/\rho_t$ never exceeds $1$. As discussed further in Section \ref{Proof_Overview_lr}, this happens because directional divergence activates a BN-induced negative-feedback adjustment that rapidly decreases the effective learning rate during the \textit{Rising Edge}; see Figure \ref{fig1}.

\subsection{A Supporting Logistic-Regression Analysis}\label{Logistic_Regression_Model}
We now turn to a narrower logistic-regression analysis. This subsection records one stylized finite-horizon directional precursor under stronger assumptions. The reader can refer to Appendix \ref{Detailed_Proof_for_Logistic_Regression} for the corresponding proofs. Following \citet{implicit_bias_BN}, we consider the overparameterized setting, where the number of features exceeds the number of training examples.
\begin{assumption}[Overparameterization]\label{Overparameterization}
  Assume $n < d$ and $\text{rank}\{\mathbf{X}\} = n$. 
\end{assumption}
\begin{assumption}[Logistic setup]\label{Logistic_Setup}
  The dataset is linearly separable, so the hard-margin SVM problem in Definition \ref{margin} is feasible. In addition, the initialization satisfies $\mathbf{w}_0 \in \text{span}(\mathbf{X})$.
\end{assumption}
Under Assumption \ref{Logistic_Setup}, each gradient update remains in $\text{span}(\mathbf{X})$, so $\mathbf{w}_t \in \text{span}(\mathbf{X})$ for all $t \ge 0$.
We now introduce the Support Vector Machine (SVM) solution, which serves as the 
reference convergence direction of $\mathbf{w}_t$ in the gradient descent dynamics. 
\begin{definition}[Support Vector Machine solution]\label{margin}
  Let $\hat{\mathbf{w}}$ be the SVM solution and define margin as $\gamma := 1 / \Vert \hat{\mathbf{w}} \Vert$:
  \begin{equation*}
  \begin{aligned}
    \hat{\mathbf{w}} := \arg \min_{\mathbf{w} \in \mathbb{R}^{d}} \|\mathbf{w}\|_2, \ 
    \text{s.t.} \quad y_i \cdot \langle \mathbf{x}_i, \mathbf{w} \rangle \geq 1, \  i \in [n].
  \end{aligned}
  \end{equation*}
\end{definition}
\begin{assumption}[Active-margin data]\label{Active_margin_data}
  The SVM solution in Definition \ref{margin} satisfies
  \[
    y_i \cdot \langle \mathbf{x}_i, \hat{\mathbf{w}} \rangle = 1,
    \qquad \forall i \in [n].
  \]
\end{assumption}
This is a highly restrictive special-case assumption: every training sample lies exactly on the max-margin boundary, so all examples are support vectors. We use it only to obtain a clean finite-horizon directional argument in the logistic analysis.
Under Assumption \ref{Logistic_Setup}, $\mathbf{w}_t \in \text{span}(\mathbf{X})$ for all $t$, so we also use the extremal eigenvalues of $\mathbf{\Sigma}$ restricted to
\[
  \mathcal{X} := \text{span}(\mathbf{X}) \backslash \{0\},
\]
namely
\begin{equation*}
\begin{aligned}
  \lambda_{\max}
  &:=
  \sup_{\mathbf{w} \in \mathcal{X}}
  \frac{\Vert \mathbf{w} \Vert_{\mathbf{\Sigma}}^2}{\Vert \mathbf{w} \Vert^2},
  \\
  \lambda_{\min}
  &:=
  \inf_{\mathbf{w} \in \mathcal{X}}
  \frac{\Vert \mathbf{w} \Vert_{\mathbf{\Sigma}}^2}{\Vert \mathbf{w} \Vert^2}.
\end{aligned}
\end{equation*}
The next lemma upper bounds the risk in terms of the parameter direction and scale.
\begin{lemma}[Upper Bound of Logistic Loss]\label{risk_upperbound_logit_reg}
  Let $\ell$ be $\ell_{log}$ and $\hat{\mathbf{w}}$ be the SVM solution as presented in Definition \ref{margin}. 
  Suppose Assumptions \ref{Overparameterization} and \ref{Logistic_Setup} hold.
  Consider the gradient descent (\ref{gradient_updata}), 
  for all $t \ge 0$, if $\rho_t > 0$ and $\alpha_t > 0$, it holds that 
  \begin{equation*}
  \begin{aligned}
    \mathcal{R}_t \le 
    \ell \left(\alpha_t\right) +
    \alpha_t \left| \ell^{\prime} \left(
        \left[ 1 -  C_0 \gamma \cdot \rho_t^{\perp} \right] \alpha_t
      \right)
    \right|
    \cdot C_0 \gamma \cdot \rho_t^{\perp},
  \end{aligned}
  \end{equation*}
  where $C_0$ is a data-dependent constant. 
\end{lemma}
Lemma \ref{risk_upperbound_logit_reg} shows that, on the positive-alignment branch, better directional alignment translates into a smaller logistic loss. A converse bridge is available in the appendix under an additional non-degeneracy condition, but the main theorem below remains purely directional. We therefore state it here in qualitative form and defer the full constant package to Appendix \ref{Detailed_Proof_for_Logistic_Regression}.
\begin{theorem}[Finite-Horizon Directional Precursor for BN Logistic Regression (qualitative form)]\label{logis_spike}
  Let $\ell$ be $\ell_{log}$ and $\hat{\mathbf{w}}$ be the SVM solution as defined in Definition \ref{margin}. 
  Suppose Assumptions \ref{Overparameterization}, \ref{Logistic_Setup}, and \ref{Active_margin_data} hold and $\lambda_{\max} > 1$. There exist explicit constants
  \[
    C_{\mathrm{low}}, \quad C_{\mathrm{high}}, \quad C_{\alpha}, \quad T_0, \quad \Theta_{\downarrow}, \quad \tan_{\min}, \quad \Theta_{\uparrow},
  \]
  depending only on $\alpha_0$, $\gamma$, $\lambda_{\min}$, $\lambda_{\max}$, and the initialization, such that if
  \[
    0 < \alpha_0 \le \frac{1}{3} \log \left( \lambda_{\max} \right), \qquad
    C_{\mathrm{low}} \gamma \le \frac{\eta}{\Vert \mathbf{w}_0 \Vert^2} \le C_{\mathrm{high}} \gamma^{-1}, \qquad
    \eta_{\alpha} \le C_{\alpha} \frac{\Vert \mathbf{w}_0 \Vert^2}{\eta \gamma},
  \]
  then the following hold for $t \in [0, T_0)$:
  \begin{enumerate}
    \item \textbf{Monotonic decrease before the small-ratio regime.}
    As long as $\rho_t > 0$ and $\left( {\rho_t^{\perp}}/{\rho_t} \right)^2 \ge \Theta_{\downarrow}$, the quantity $\rho_t^{\perp}$ keeps decreasing.
    \item \textbf{Entry into a small-ratio regime.}
    There exists a $t_0 < T_0$ such that
    $\left(\rho_{t_0}^{\perp}/\rho_{t_0}\right)^2 \le \tan_{\min}^2 = {\gamma^2 \lambda_{\min}}/{(8 \lambda_{\max}^2)}$.
    \item \textbf{Conditional exit threshold on a later positive-alignment rising branch.}
    If, after item (2), the directional dynamics enters a \textit{Rising Edge} segment on the positive-alignment branch $\rho_t > 0$, then every iterate on that segment satisfying
    $\left(\rho_t^{\perp} / \rho_t\right)^2 \ge \Theta_{\uparrow}$
    already meets the convergence condition of Lemma \ref{Direction_Convegence_and_Divergence}, so the next iterate leaves the \textit{Rising Edge}.
  \end{enumerate}
  The full appendix restatement gives the exact formulas for $C_{\mathrm{low}}$, $C_{\mathrm{high}}$, $C_{\alpha}$, $T_0$, $\Theta_{\downarrow}$, and $\Theta_{\uparrow}$.
\end{theorem}
\noindent The theorem is intentionally narrow: it applies only in parameter regimes where the lower and upper bounds above are simultaneously satisfiable, so it should not be read as a generic description of batch-normalized logistic regression or as an improvement over the asymptotic guarantee of \citet{implicit_bias_BN}.
Theorem \ref{logis_spike} has two main qualitative takeaways:
\begin{itemize}
    \item A larger learning rate $\eta$ and a smaller initialization norm $\|\mathbf w_0\|$ make the small-ratio regime easier to reach. The exact threshold formula is recorded in the appendix restatement.
    \item The delayed directional instability also depends on both the margin and the condition number of the data covariance matrix. Since $\tan_{\min}^2 = {\gamma^2 \lambda_{\min}}/{(8 \lambda_{\max}^2)}$, smaller margins and worse conditioning push the dynamics deeper into the low-directional-error regime before a positive-alignment rising branch can start.
\end{itemize}
Taken together, the two theorem lines support a limited common message: the same BN-induced auto-rate-tuning mechanism can drive directional instability, but only the whitened linear analysis yields an explicit delayed-onset and self-stabilization theorem with a square-loss interpretation. The logistic theorem identifies one narrower pathway. In particular, whenever the positive-alignment branch in item (3) is realized, the corresponding directional exit threshold in logistic regression scales with $\eta / (\gamma \Vert \mathbf{w}_0 \Vert^2)$, whereas for linear regression the peak directional ratio along the rising edge is bounded by a constant-order quantity.


\section{Proof Overview}\label{ProofOverview}

We first explain, at a high level, how batch normalization reshapes the direction dynamics of gradient descent. We then specialize this mechanism to the whitened square-loss linear-regression analysis and the narrower logistic-regression analysis. The goal of this section is a roadmap rather than a derivation: it tells the reader which lemma drives which transition, while the appendix carries the constant tracking and proof details. Full proofs are deferred to Appendix \ref{Detailed_Proof_for_Linear_Regression} and Appendix \ref{Detailed_Proof_for_Logistic_Regression}.
\subsection{Directional Convergence Induced by BN}\label{Directional_Convergence}
We begin at a level that is agnostic to the specific task and objective, and focus only on the direction dynamics induced by normalization. 
The following lemma describes when a normalized model moves toward or away from a reference direction during gradient descent.
\begin{lemma}[Directional Convergence and Divergence]\label{Direction_Convegence_and_Divergence}
  Suppose there exists a reference direction $\hat{\mathbf{w}}$. 
  Consider gradient descent (\ref{gradient_updata}) on an objective function 
  $\mathcal{R} \left(\mathbf{w}, \alpha\right)$, where $\mathbf{w}$
  is parameterized by normalization and $\alpha$ represents the scaling factor of the normalization.
  Assume further that the objective is scale-invariant in $\mathbf{w}$, namely
  \[
    \mathcal{R}(k \mathbf{w}, \alpha) = \mathcal{R}(\mathbf{w}, \alpha),
    \qquad \forall k > 0.
  \]
  We have the following direction convergence condition: 
  if there exists a $t \ge 0$ such that 
  \begin{equation*}
  \begin{aligned}
  \rho_t > 0 \text{ and }\ 
    \frac{\eta {\rho}_{t}}{\Vert \mathbf{w}_{t} \Vert}
    \left\Vert
      \nabla_{\mathbf{w}} \mathcal{R}_t
    \right\Vert^2
    \le 
    - 2
    \left\langle 
      \hat{\mathbf{w}},
      \nabla_{\mathbf{w}} \mathcal{R}_t
    \right\rangle,
  \end{aligned}
  \end{equation*}
  it holds that $\left( {\rho}_{t+1}^{\perp} \right)^2 \le \left( {\rho}_{t}^{\perp}\right)^2$, 
  and the following direction divergence condition: 
  if there exists a $t \ge 0$ such that 
  \begin{equation*}
  \begin{aligned}
  0 < \rho_t^{\perp} / \rho_t \le 1, \ \ \alpha_t > 0
  \ \ \text{and} \ \ 
    \frac{\eta }{\Vert \mathbf{w}_t \Vert} 
    \left\Vert 
      \nabla_{\mathbf{w}} \mathcal{R}_t
    \right\Vert
    \ge
    \frac{2  {\rho}_t {\rho}_t^{\perp}}{
      {\rho}_t^2 
      - \left({\rho}_t^{\perp}\right)^2
    }
  \end{aligned},
  \end{equation*}
  it holds that $\left({\rho}_{t+1}^{\perp}\right)^2 \ge \left({\rho}_{t}^{\perp}\right)^2$. 
\end{lemma}
\begin{remark}
  Recall that smaller $\rho_t^{\perp}$ means that the direction of $\mathbf{w}_t$ is closer to $\hat{\mathbf{w}}$. The lemma therefore gives generic convergence and divergence conditions for the direction dynamics of any scale-invariant normalized model. It is not restricted to batch normalization, but it does rely essentially on the scale-invariance assumption stated above. Once that assumption holds, the conclusion does not depend on the particular training objective or the specific choice of reference direction. 
\end{remark}
Lemma \ref{Direction_Convegence_and_Divergence} is one of the main technical tools of the paper; its proof is given in Appendix \ref{Proof_of_Lemma_direction}. In the next two subsections, we use it to explain how normalization-induced directional divergence can generate delayed instability after a sustained well-aligned phase.

\subsection{Whitened Square-Loss Linear Regression}\label{Proof_Overview_lr}
For linear regression, the square-loss gradient admits an especially transparent reformulation in terms of $\rho_t$ and $\rho_t^{\perp}$:
\begin{equation*}
\begin{aligned}
  & \nabla_{\mathbf{w}} \mathcal{R}_t = 
  - \frac{\alpha_t}{\Vert \mathbf{w}_{t} \Vert}
  \left(
    \mathbf{I}
    -
    \frac{\mathbf{w} \mathbf{w}^T}{\Vert \mathbf{w} \Vert^2}
  \right)\hat{\mathbf{w}} ;
  \quad
  \frac{\partial \mathcal{R}_t}{\partial \alpha} = - \left(\rho_{t} - \alpha_t\right).
\end{aligned}
\end{equation*}
Thus $\nabla_{\mathbf{w}} \mathcal{R}_t$ is collinear with the projection of $\hat{\mathbf{w}}$ onto $\text{span}^{\perp} \{\mathbf{w}_t\}$. This exact structure is what makes the linear-regression analysis much sharper than the logistic one: combined with Lemma \ref{Direction_Convegence_and_Divergence}, it yields a nearly closed dynamical system for the direction ratio and the effective learning rate.
\begin{figure}[!htbp]
  \includegraphics[width=0.95\columnwidth]{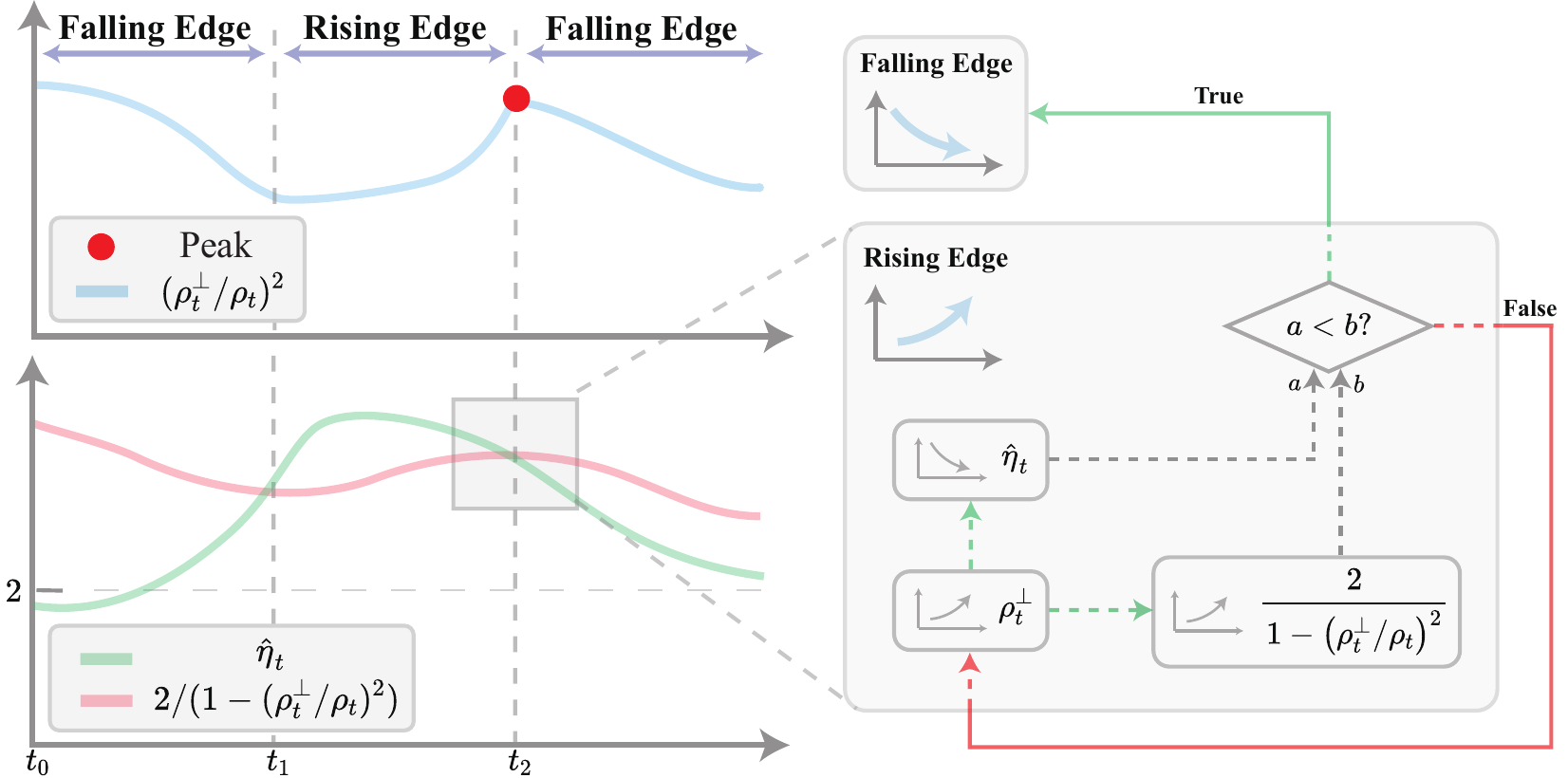}
  \caption{Schematic dynamics for batch-normalized linear regression. The relative magnitude of $2/(1 - (\rho_t^{\perp} / \rho_t)^2)$ and $\hat{\eta}_t$ determines whether the system is in a \textit{Falling Edge} or a \textit{Rising Edge}. During the \textit{Rising Edge}, the growth of $\rho_t^{\perp}$ activates a negative-feedback loop that eventually drives the dynamics back to the \textit{Falling Edge}.}
  \label{fig1}
\end{figure}
\begin{lemma}[The Dynamics of BN Linear Regression]\label{lemma:Dynamics_of_BN_Linear_Regression}
  Let $\ell = \ell_{squ}$, $\hat{\mathbf{w}} = \mathbf{\Sigma}^{-1} \boldsymbol{\mu}$ and $\mathbf{\Sigma} = \mathbf{I}$. 
  Consider the gradient descent (\ref{gradient_updata}), 
  it holds that 
  \begin{equation*}
  \begin{aligned}
    & (1).
    \frac{\rho_{t+1}^{\perp}}{\rho_{t+1}} = 
    \frac{|\hat{\eta}_t - 1|}{1 + \hat{\eta}_t \left({\rho_t^{\perp}}/{\rho_t}\right)^2}
    \frac{\rho_t^{\perp}}{\rho_t}; 
    \\ & (2). \alpha_{t+1}  = \alpha_t + \eta_{\alpha} \left( \rho_{t} - \alpha_t \right);
    \\ & (3). \Vert \mathbf{w}_{t+1} \Vert^2 = 
    \Vert \mathbf{w}_{t} \Vert^2
    + 
    \frac{\eta^2 \alpha_t^2}{\Vert \mathbf{w}_{t} \Vert^2} \left(\rho_t^{\perp}\right)^2,
  \end{aligned}
  \end{equation*}
  where $\hat{\eta}_t$ is effective learning rate, defined as $\hat{\eta}_{t} := {\eta \alpha_t \rho_{t}}/{\Vert \mathbf{w}_{t} \Vert^{2}}$. 
\end{lemma}
\begin{remark}
  Lemma \ref{lemma:Dynamics_of_BN_Linear_Regression} gives an almost minimal dynamical system for the batch-normalized linear model. It retains the core BN-specific feedback loop while stripping away complications that are irrelevant for the spike mechanism.
\end{remark}
The key question is therefore how the dynamics switches between the \textit{Falling Edge} and the \textit{Rising Edge}. Lemma \ref{lemma:Dynamics_of_BN_Linear_Regression} answers this explicitly: if $\hat{\eta}_t > {2}/{(1 - (\rho_t^{\perp} / \rho_t)^2)}$, then $\rho_t^{\perp} / \rho_t$ increases; otherwise it decreases. The proof roadmap for Theorems \ref{Occurance_of_Spike_ls} and \ref{Divergence_lr} has three steps.

\noindent\textbf{Step 1: enter the delayed-onset regime.}
During the \textit{Falling Edge}, $\rho_t^{\perp}$ decreases while $\rho_t$ increases because $\rho_t^2 + (\rho_t^{\perp})^2 = \Vert \hat{\mathbf{w}} \Vert^2$. Item (2) of Lemma \ref{lemma:Dynamics_of_BN_Linear_Regression} shows that $\alpha_t$ tracks $\rho_t$, while item (3) shows that $\Vert \mathbf{w}_t \Vert$ grows only slowly once the iterate is already well aligned. Consequently $\hat{\eta}_t$ rises while the threshold ${2}/{(1 - (\rho_t^{\perp} / \rho_t)^2)}$ falls, so the trajectory must eventually cross into the \textit{Rising Edge}; Theorem \ref{Occurance_of_Spike_ls} turns this into the no-rising-edge, delayed-onset, and waiting-time statements.

\noindent\textbf{Step 2: force the rising edge to end.}
Once the trajectory enters the \textit{Rising Edge}, $\rho_t^{\perp}$ increases while $\rho_t$ decreases. The same BN dynamics now creates negative feedback: $\alpha_t$ follows the decreasing $\rho_t$, while $\Vert \mathbf{w}_t \Vert$ keeps growing. In the appendix this is quantified by the growth estimate
\[
\Vert \mathbf{w}_t \Vert^2 
\geq c \cdot \left(\int_{\tau=0}^{t} \left(\rho_{\tau}^{\perp}\right)^2 \, d\tau\right)^{1/2}
\geq c \cdot \rho_0^{\perp} \cdot \sqrt{t},
\]
which eventually pushes $\hat{\eta}_t$ back below the divergence threshold. Lemma \ref{Existence_of_Falling_Edge_lr} formalizes this return to the \textit{Falling Edge}.

\noindent\textbf{Step 3: bound the rising-edge shape.}
With the two transition directions in hand, Theorem \ref{Divergence_lr} bounds how long the \textit{Rising Edge} can last and why the peak directional ratio remains bounded. Figure \ref{fig1} summarizes this feedback loop at the level of mechanism.

\subsection{Supporting Logistic-Regression Analysis}\label{proof_overview_Logistic_Regression}
Under Assumption \ref{Logistic_Setup}, the iterate $\mathbf{w}_t$ evolves only inside $\text{span} \left(\mathbf{X}\right)$, so the same restricted eigenvalues $\lambda_{\min}$ and $\lambda_{\max}$ from Section \ref{Logistic_Regression_Model} control the proof. The roadmap has three steps: first control how directional error perturbs the logits, then convert the abstract directional lemma into explicit finite-horizon inequalities, and finally show that the trajectory can reach the small-ratio regime before any later positive-alignment rising branch must exit.
\begin{lemma}\label{logit_bound}
  Let $\ell$ be $\ell_{log}$ and $\hat{\mathbf{w}}$ be the SVM solution as presented in Definition \ref{margin}. Suppose Assumptions \ref{Overparameterization} and \ref{Logistic_Setup} hold. Then, on the positive-alignment branch $\rho_t>0$, the appendix proves explicit inequalities showing that both $\Vert \mathbf{w}_t \Vert_{\mathbf{\Sigma}}$ and every signed logit $y_i\langle \mathbf{w}_t,\mathbf{x}_i\rangle$ remain within controlled $O(\rho_t^\perp)$ perturbations of their aligned reference values.
\end{lemma}
Lemma \ref{logit_bound} supplies the first ingredient: once $\rho_t^\perp$ is small, the logits behave like a controlled perturbation of the aligned reference scale $\alpha_t$. To use Lemma \ref{Direction_Convegence_and_Divergence}, we then need matching gradient bounds.
\begin{lemma}\label{inner_product_bound}
  Let $\ell$ be $\ell_{log}$ and $\hat{\mathbf{w}}$ be the SVM solution as defined in Definition \ref{margin}. Suppose Assumptions \ref{Overparameterization}, \ref{Logistic_Setup}, and \ref{Active_margin_data} hold. Then the appendix proves an explicit lower bound on the aligned gradient component $-\langle \hat{\mathbf{w}}, \nabla_{\mathbf{w}} \mathcal{R}_t\rangle$, showing that it is nonnegative and scales quadratically in $\rho_t^{\perp}$ on the positive-alignment branch.
\end{lemma}
\begin{lemma}\label{Gradient_Norm_Bound}
  Let $\ell$ be $\ell_{log}$ and $\hat{\mathbf{w}}$ be the SVM solution as defined in Definition \ref{margin}. Suppose Assumptions \ref{Overparameterization}, \ref{Logistic_Setup}, and \ref{Active_margin_data} hold. Then the appendix proves matching upper and lower bounds on $\Vert \nabla_{\mathbf{w}} \mathcal{R}_t \Vert$, again in terms of $\alpha_t$, $\Vert \mathbf{w}_t \Vert_{\mathbf{\Sigma}}$, and $\rho_t^{\perp}$.
\end{lemma}
Lemmas \ref{Gradient_Norm_Bound} and \ref{inner_product_bound} provide the second ingredient by translating the abstract directional conditions of Lemma \ref{Direction_Convegence_and_Divergence} into explicit finite-horizon inequalities. The only reader-facing threshold we keep in the main text is the divergence criterion
\begin{equation}\label{main_text_spike_logistic}
\begin{aligned}
  \frac{\lambda_{\min}}{4}
  \cdot
  \frac{\alpha_t}{e^{\alpha_t}}
  \cdot 
  \frac{\eta {\rho}_t}{
    \Vert \mathbf{w}_t \Vert_{\mathbf{\Sigma}}
    \Vert \mathbf{w}_t \Vert
  }
  \ge
  \frac{2}{
    1 - \left({\rho}_t^{\perp} / \rho_t\right)^2
  }
\end{aligned}
\end{equation}
When both the margin $\gamma$ and the direction error $\rho_t^{\perp}$ are small, Lemma \ref{logit_bound} turns (\ref{main_text_spike_logistic}) into an effective threshold of order $1/\gamma^2$. This is why smaller margins make the directional transition easier to satisfy within the finite-horizon regime of Theorem \ref{logis_spike}. The corresponding convergence inequality, derived in the appendix from the same two gradient bounds, yields the explicit exit threshold in item (3): once a positive-alignment rising branch grows past that threshold, the next iterate must leave the branch.

The third ingredient is entry into the small-ratio regime itself.
\begin{lemma}\label{First_Phase}
  Let $\ell$ be $\ell_{log}$ and $\hat{\mathbf{w}}$ be the SVM solution as defined in Definition \ref{margin}. Suppose Assumptions \ref{Overparameterization}, \ref{Logistic_Setup}, and \ref{Active_margin_data} hold and $\lambda_{\max} > 1$. Then, under the explicit parameter conditions recorded in the appendix restatement, there exists a finite horizon $T_0$ and an iterate $t_0<T_0$ such that $\left(\rho_{t_0}^{\perp} / \rho_{t_0}\right)^2 \le \tan_{\min}^2$, while $\alpha_t$ stays within a controlled constant-factor window throughout $[0,T_0)$.
\end{lemma}
Lemma \ref{First_Phase} supplies the final step: it shows that the theorem's parameter regime is strong enough to force entry into the small-ratio window before finite-horizon control is lost. This yields item (2) of Theorem \ref{logis_spike}. Item (3) then follows by combining the directional transition condition (\ref{main_text_spike_logistic}) with Lemma \ref{Direction_Convegence_and_Divergence} along the positive-alignment branch. If Assumption \ref{Non_degenerate_data} is also imposed, Lemma \ref{risk_lowerbound_logit_reg} gives an optional lower bound on the logistic loss, but the theorem itself remains directional.

\section{Qualitative Mechanism Illustration}\label{sec:experiment}
\begin{figure*}[t]
  \centerline{\includegraphics[width=0.95\textwidth]{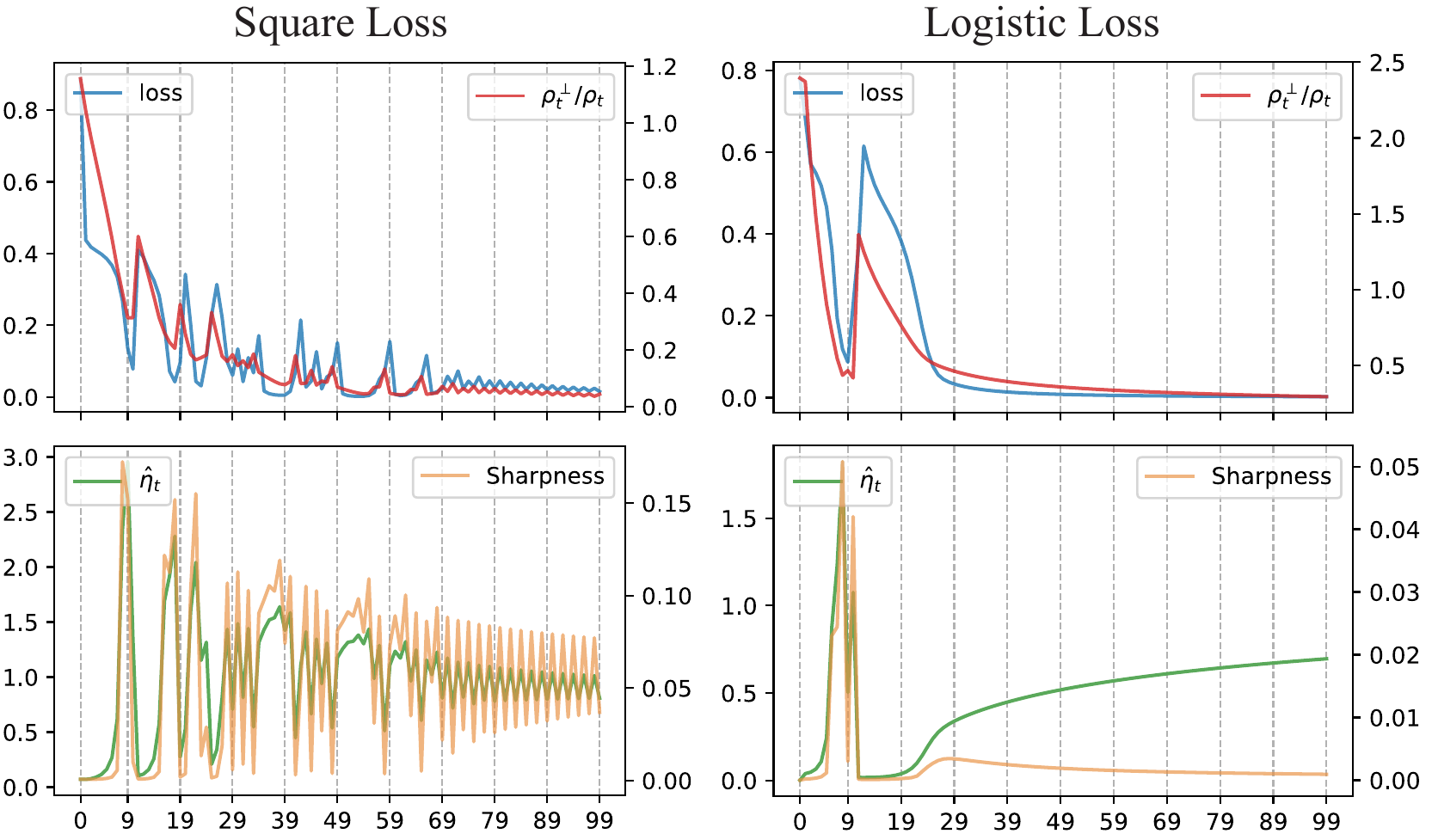}}
  \caption{Synthetic mechanism illustrations for square loss (left column) and logistic loss (right column). Each column reports the loss trajectory together with the corresponding effective-learning-rate and sharpness trends over full-batch gradient-descent iterations. The square-loss panel illustrates the theorem-backed delayed-rising-edge mechanism in the whitened linear setting together with its square-loss interpretation. The logistic panel is included only as a qualitative directional illustration in a narrow stylized regime; it is not evidence of a theorem-backed generic logistic spike claim.}
  \label{Fig2}
\end{figure*}
To illustrate the directional mechanism highlighted by the theory, we run two synthetic full-batch experiments on deliberately ill-conditioned data. Concretely, we slice a Hilbert matrix, apply random rotations, and add Gaussian noise to obtain a small-margin, poorly conditioned instance. The dataset contains $10$ samples with feature dimension $20$, matching the overparameterized setting, and we train a batch-normalized linear model with square loss and logistic loss.

As shown in Figure \ref{Fig2}, the synthetic trajectories are qualitatively consistent with the directional mechanism analyzed in the main text, but the interpretation is intentionally asymmetric. In the square-loss panel, the visible delayed square-loss increase coincides with rapid growth in $\rho_t^{\perp} / \rho_t$ and with the effective-learning-rate mechanism analyzed in Sections \ref{Main_Results} and \ref{ProofOverview}. In the logistic panel, we use the same diagnostics only to illustrate that a related directional transition can be seen in one stylized instance; this panel should not be read as theorem-level evidence for a generic logistic spike claim. We also record the effective learning rate $\hat{\eta}_t = \alpha_t \cdot \rho_t / \Vert \mathbf{w}_t \Vert_{\mathbf{\Sigma}}^2$, whose increase and later decrease track the directional mechanism discussed in Section \ref{ProofOverview}. The sharpness curves are included only as auxiliary diagnostics rather than as theorem-level evidence.

These experiments use full-batch gradient descent on synthetic data and can be run without specialized hardware, but the current draft still lacks a public code package, fixed random seeds, and a full hyperparameter table. This section should therefore be read only as a qualitative mechanism illustration in a stylized setting, not as the main scientific support for the paper.

\section{Conclusion}

This paper isolates a concrete delayed-instability mechanism in batch-normalized linear models. In the whitened square-loss regime, batch normalization can induce a delayed rising edge through directional divergence, and that rising edge self-stabilizes rather than causing immediate runaway behavior. Via the square-loss decomposition, this yields the paper's delayed-spike interpretation in the whitened regime. The logistic-regression analysis is weaker: under additional restrictive structural assumptions, it yields only a finite-horizon directional precursor together with an explicit exit threshold on a positive-alignment branch. Relative to prior work on batch-normalization auto-rate tuning, asymptotic implicit bias, and large-step logistic-regression oscillations, our contribution is therefore not broader generality but a sharper mechanism-level picture in a tractable normalized setting. The strongest theorem here relies essentially on the whitened linear regime, and it remains open which parts of that picture survive without whitening or in nonlinear architectures.

\begin{appendices}

\section{Proof of Lemma \ref{Direction_Convegence_and_Divergence}}\label{Proof_of_Lemma_direction}

We first prove a directional lemma for scale-invariant normalized parameterizations. 
\begin{lemma}\label{Orthogonal}
  Let $\mathcal{R} \left(\mathbf{w}\right): \mathbb{R}^{d} \rightarrow \mathbb{R}$.
  Suppose that $\mathcal{R} \left(\mathbf{w}\right) = \mathcal{R} \left(k \cdot \mathbf{w}\right)$ for every $k > 0$.
  Then $\langle \mathbf{w}, \nabla_\mathbf{w} \mathcal{R} \left( \mathbf{w} \right)\rangle = 0$.
\end{lemma}
\begin{proof}
  We have
  $\forall k > 0, \mathcal{R}\left(\mathbf{w}\right) = \mathcal{R}\left(k \mathbf{w}\right)$. 
  Take derivatives with respect to $k$ on both sides to obtain
  $
    \left\langle 
       \nabla_{\mathbf{w}} \mathcal{R}\left(k \cdot \mathbf{w}\right),
      \mathbf{w}
    \right\rangle =  0.
  $
  By setting $k = 1$, we prove that $
  \left\langle 
      \nabla_{\mathbf{w}} \mathcal{R}\left(\mathbf{w}\right), \mathbf{w}
  \right\rangle = 0$. 
\end{proof}

\newtheorem*{restate_Direction_Convegence_and_Divergence}{Lemma \ref{Direction_Convegence_and_Divergence} (Directional Convergence and Divergence)}
\definecolor{shadecolor}{rgb}{0.92,0.92,0.92}
\begin{shaded}
  \begin{restate_Direction_Convegence_and_Divergence}
    Suppose there exists a reference direction $\hat{\mathbf{w}}$. 
  Consider gradient descent (\ref{gradient_updata}) on an objective function 
  $\mathcal{R} \left(\mathbf{w}, \alpha\right)$, where $\mathbf{w}$
  is parameterized by normalization and $\alpha$ represents the scaling factor of the normalization.
  Assume further that the objective is scale-invariant in $\mathbf{w}$, namely
  \[
    \mathcal{R}(k \mathbf{w}, \alpha) = \mathcal{R}(\mathbf{w}, \alpha),
    \qquad \forall k > 0.
  \]
  We have the following direction convergence condition: 
  if there exists a $t \ge 0$ such that 
  \begin{equation*}
  \begin{aligned}
  \rho_t > 0 \text{ and }\ 
    \frac{\eta {\rho}_{t}}{\Vert \mathbf{w}_{t} \Vert}
    \left\Vert
      \nabla_{\mathbf{w}} \mathcal{R}_t
    \right\Vert^2
    \le 
    - 2
    \left\langle 
      \hat{\mathbf{w}},
      \nabla_{\mathbf{w}} \mathcal{R}_t
    \right\rangle,
  \end{aligned}
  \end{equation*}
  it holds that $\left( {\rho}_{t+1}^{\perp} \right)^2 \le \left( {\rho}_{t}^{\perp}\right)^2$, 
  and the following direction divergence condition: 
  if there exists a $t \ge 0$ such that 
  \begin{equation*}
  \begin{aligned}
  0 < \rho_t^{\perp} / \rho_t \le 1, \ \ \alpha_t > 0
  \ \ \text{and} \ \ 
    \frac{\eta }{\Vert \mathbf{w}_t \Vert} 
    \left\Vert 
      \nabla_{\mathbf{w}} \mathcal{R}_t
    \right\Vert
    \ge
    \frac{2  {\rho}_t {\rho}_t^{\perp}}{
      {\rho}_t^2 
      - \left({\rho}_t^{\perp}\right)^2
    }
  \end{aligned},
  \end{equation*}
  it holds that $\left({\rho}_{t+1}^{\perp}\right)^2 \ge \left({\rho}_{t}^{\perp}\right)^2$. 
  \end{restate_Direction_Convegence_and_Divergence}
\end{shaded}
\begin{proof}
  We prove the first result. 
  \begin{equation*}
  \begin{aligned}
    \left(
      {\rho}_{t+1}^{\perp}
    \right)^2 
    & = 
    \left\Vert 
      \hat{\mathbf{w}} - 
      \frac{{\rho}_{t+1}}{\Vert \mathbf{w}_{t+1} \Vert}
      \mathbf{w}_{t+1}
    \right\Vert^2
    \\ & \le 
    \left\Vert 
      \hat{\mathbf{w}} - 
      \frac{{\rho}_{t}}{\Vert \mathbf{w}_{t} \Vert}
      \mathbf{w}_{t+1}
    \right\Vert^2
    \\ & =
    \left\Vert 
      \hat{\mathbf{w}} - \frac{{\rho}_{t}}{\Vert \mathbf{w}_{t} \Vert} \mathbf{w}_{t}
      +
      \frac{\eta {\rho}_{t}}{\Vert \mathbf{w}_{t} \Vert}
      \nabla_{\mathbf{w}} \mathcal{R} \left(\mathbf{w}_t, \alpha_t\right)
    \right\Vert^2
    \\ & =
    \left(
      {\rho}_{t}^{\perp}
    \right)^2 
    +
    \frac{\eta^2 {\rho}_{t}^2}{\Vert \mathbf{w}_{t} \Vert^2}
    \left\Vert
      \nabla_{\mathbf{w}} \mathcal{R} \left(\mathbf{w}_t, \alpha_t\right)
    \right\Vert^2
    +
    \frac{2 \eta {\rho}_{t}}{\Vert \mathbf{w}_{t} \Vert}
    \left(
      \hat{\mathbf{w}} - 
      \frac{{\rho}_{t}}{\Vert \mathbf{w}_{t} \Vert} 
      \mathbf{w}_{t}
    \right)^T
    \nabla_{\mathbf{w}} \mathcal{R} \left(\mathbf{w}_t, \alpha_t\right)
\end{aligned}
\end{equation*}
where 
the inequality is since $\mathbf{w} \cdot \langle \mathbf{w}, \hat{\mathbf{w}}\rangle / \Vert \mathbf{w} \Vert^2$ 
is the projection of $\hat{\mathbf{w}}$ onto $\text{span}\{\mathbf{w}\}$ under the
Euclidean inner product $\langle\cdot, \cdot\rangle$
  and 
the last equation is because $\langle \mathbf{w}_t, \nabla_{\mathbf{w}} \mathcal{R} \left(\mathbf{w}_t, \alpha_t\right) \rangle = 0$ by Lemma \ref{Orthogonal}. And then, we have
\begin{equation*}
\begin{aligned}
\left(
      {\rho}_{t+1}^{\perp}
    \right)^2 
    & \le
    \left(
      {\rho}_{t}^{\perp}
    \right)^2 
    +
    \frac{\eta^2 {\rho}_{t}^2}{\Vert \mathbf{w}_{t} \Vert^2}
    \left\Vert
      \nabla_{\mathbf{w}} \mathcal{R} \left(\mathbf{w}_t, \alpha_t\right)
    \right\Vert^2
    +
    \frac{2 \eta {\rho}_{t}}{\Vert \mathbf{w}_{t} \Vert}
    \left\langle 
    \hat{\mathbf{w}},
    \nabla_{\mathbf{w}} \mathcal{R} \left(\mathbf{w}_t, \alpha_t\right)
    \right\rangle
    \\ & =
    \left(
      {\rho}_{t}^{\perp}
    \right)^2 
    +
    \frac{\eta {\rho}_{t}}{\Vert \mathbf{w}_{t} \Vert}
    \left(
      \frac{\eta {\rho}_{t}}{\Vert \mathbf{w}_{t} \Vert}
      \left\Vert
        \nabla_{\mathbf{w}} \mathcal{R} \left(\mathbf{w}_t, \alpha_t\right)
      \right\Vert^2
      +
      2
      \left\langle 
      \hat{\mathbf{w}},
      \nabla_{\mathbf{w}} \mathcal{R} \left(\mathbf{w}_t, \alpha_t\right)
      \right\rangle
    \right)
    \\ & \le 
    \left(
      {\rho}_{t}^{\perp}
    \right)^2,
  \end{aligned}
  \end{equation*}
  Next, we prove the second result. 
  Since $\left({\rho}_{t}^{\perp}\right)^2 + {\rho}_{t}^2 = \Vert \hat{\mathbf{w}}\Vert^2$, 
  we check if $\left(
    {\rho}_{t+1}^{\perp}
  \right)^2 \ge \left(
    {\rho}_{t}^{\perp}
  \right)^2$ by verifying if $\left(
    {\rho}_{t+1}
  \right)^2 - \left(
    {\rho}_{t}
  \right)^2 \le 0$.
  To simplify the calculation, we decompose $\hat{\mathbf{w}}$ orthogonally into two components: one in the $\text{span}\left\{\mathbf{w}_t\right\}$ and the other
  in a direction orthogonal to it: 
  \begin{equation*}
  \begin{aligned}
    \mathbf{e}_1 := \mathbf{w}_t / \Vert \mathbf{w}_t \Vert; \quad 
    \mathbf{e}_2 := \left(\hat{\mathbf{w}} - \rho_t \mathbf{e}_1 \right) / \rho_t^{\perp}.
  \end{aligned}
  \end{equation*}
  It is easy to see that
  \begin{equation*}
  \begin{aligned}
     \langle \mathbf{e}_1, \mathbf{e}_2 \rangle = 0; \quad 
     \mathbf{w}_t = \Vert \mathbf{w}_t \Vert \mathbf{e}_1; \quad
    \hat{\mathbf{w}} = \rho_t \mathbf{e}_1 + \rho_t^{\perp} \mathbf{e}_2.
  \end{aligned}
  \end{equation*}
  Recall that 
  \begin{equation*}
  \begin{aligned}
    \mathbf{w}_{t+1} = \mathbf{w}_{t} - \eta 
    \nabla_{\mathbf{w}} \mathcal{R}_t
    = \Vert \mathbf{w}_t \Vert \mathbf{e}_1 - \eta 
    \nabla_{\mathbf{w}} \mathcal{R}_t
  \end{aligned}
  \end{equation*}
  By Lemma \ref{Orthogonal}, we know $\langle \mathbf{w}_t, \nabla_{\mathbf{w}} \mathcal{R}_t \rangle  = 0$.
  We then calculate $\langle \mathbf{w}_{t+1}, \hat{\mathbf{w}} \rangle^2$ and $\Vert \mathbf{w}_{t+1} \Vert^2$. 
  \begin{equation*}
  \begin{aligned}
    \langle \mathbf{w}_{t+1}, \hat{\mathbf{w}} \rangle^2 
    & = 
    \langle 
      \Vert \mathbf{w}_t \Vert \mathbf{e}_1 - \eta \nabla_{\mathbf{w}} \mathcal{R}_t, 
      \rho_t \mathbf{e}_1 + \rho_t^{\perp} \mathbf{e}_2
    \rangle^2 
    \\ & = 
    \left(
      \rho_t \Vert \mathbf{w}_t \Vert
      -
      \eta \rho_t^{\perp} \left\langle \nabla_{\mathbf{w}} \mathcal{R}_t, \mathbf{e}_2\right\rangle
    \right)^2
    \\ & = 
    \rho_t^2 \Vert \mathbf{w}_t \Vert^2
    +
    \eta^2 \left(\rho_t^{\perp}\right)^2 \left\langle \nabla_{\mathbf{w}} \mathcal{R}_t, \mathbf{e}_2 \right\rangle^2
    - 2 \eta \rho_t^{\perp}
    \rho_t \Vert \mathbf{w}_t \Vert
    \left\langle \nabla_{\mathbf{w}} \mathcal{R}_t, \mathbf{e}_2\right\rangle;
    \\
    \Vert \mathbf{w}_{t+1} \Vert^2 
    & = 
    \Vert \mathbf{w}_{t} \Vert^2
    + 
    \eta^2 \left\Vert \nabla_{\mathbf{w}} \mathcal{R}_t \right\Vert^2 
    - 
    2 \eta \left\langle \mathbf{w}_{t}, \nabla_{\mathbf{w}} \mathcal{R}_t\right\rangle 
    \\ & = 
    \Vert \mathbf{w}_{t} \Vert^2
    + 
    \eta^2 \left\Vert \nabla_{\mathbf{w}} \mathcal{R}_t \right\Vert^2,
  \end{aligned},
  \end{equation*}
  Then we check the sign of $\rho_{t+1}^2 - \rho_{t}^2$:
  \begin{equation*}
  \begin{aligned}
    & \Vert \mathbf{w}_t \Vert^2
    \Vert \mathbf{w}_{t+1} \Vert^2
    \left(
      \rho_{t+1}^2 - \rho_{t}^2
    \right)
    \\ = & 
    \Vert \mathbf{w}_t \Vert^2 \langle \mathbf{w}_{t+1}, \hat{\mathbf{w}} \rangle^2
    - 
    \langle \mathbf{w}_{t}, \hat{\mathbf{w}} \rangle^2 \Vert \mathbf{w}_{t+1} \Vert^2
    \\ = &
    \Vert \mathbf{w}_t \Vert^2 
    \left(
      \rho_t^2 \Vert \mathbf{w}_t \Vert^2
      +
      \eta^2 \left(\rho_t^{\perp}\right)^2 \left\langle \nabla_{\mathbf{w}} \mathcal{R}_t, \mathbf{e}_2 \right\rangle^2
      - 2 \eta \rho_t^{\perp}
      \rho_t \Vert \mathbf{w}_t \Vert
      \left\langle \nabla_{\mathbf{w}} \mathcal{R}_t, \mathbf{e}_2\right\rangle
    \right)
    \\ & \ \ \ \ \ \ -
    \rho_t^2 \Vert \mathbf{w}_t \Vert^2
    \left(
      \Vert \mathbf{w}_{t} \Vert^2
      + 
      \eta^2 \left\Vert \nabla_{\mathbf{w}} \mathcal{R}_t \right\Vert^2 
    \right)
    \\  =&
    \eta \Vert \mathbf{w}_t \Vert^2 
    \left(
      \eta \left(\rho_t^{\perp}\right)^2 \left\langle \nabla_{\mathbf{w}} \mathcal{R}_t, \mathbf{e}_2 \right\rangle^2
      - 2  \rho_t^{\perp}
      \rho_t \Vert \mathbf{w}_t \Vert
      \left\langle \nabla_{\mathbf{w}} \mathcal{R}_t, \mathbf{e}_2\right\rangle
      -
      \rho_t^2 
      \eta \left\Vert \nabla_{\mathbf{w}} \mathcal{R}_t \right\Vert^2 
    \right)
    \\  \le &
    \eta \Vert \mathbf{w}_t \Vert^2 
    \left(
      - \eta 
      \left(
        \rho_t^2 
        -
        \left(\rho_t^{\perp}\right)^2 
      \right)
      \left\Vert \nabla_{\mathbf{w}} \mathcal{R}_t \right\Vert^2 
      + 2  \rho_t^{\perp}
      \rho_t \Vert \mathbf{w}_t \Vert
      \Vert \nabla_{\mathbf{w}} \mathcal{R}_t \Vert
    \right) 
    \\  =&
    \eta \Vert \mathbf{w}_t \Vert^2 
    \Vert \nabla_{\mathbf{w}} \mathcal{R}_t \Vert
    \left(
      - \eta \left\Vert \nabla_{\mathbf{w}} \mathcal{R}_t \right\Vert \left(
        \rho_t^2 
        - \left(\rho_t^{\perp}\right)^2
      \right)
      + 
      2  \rho_t^{\perp}
      \rho_t \Vert \mathbf{w}_t \Vert
    \right)
  \end{aligned}
  \end{equation*}
  Now, let's examine the signs within the parentheses. 
  Recall $0 < \rho_t^{\perp} / \rho_t \le 1$, it holds that 
  \begin{equation*}
  \begin{aligned}
    - \eta 
    \left(
      \rho_t^2 
      -
      \left(\rho_t^{\perp}\right)^2 
    \right)
    \left\Vert \nabla_{\mathbf{w}} \mathcal{R}_t \right\Vert^2 \le 0
  \end{aligned}
  \end{equation*}
Hence the relevant case here is $0 < \rho_t^{\perp} / \rho_t \le 1$. 
  Since $1 \ge \rho_t^{\perp} / \rho_t > 0$, 
  it holds that ${\rho}_t^2 - \left({\rho}_t^{\perp}\right)^2 > 0$,
  and in particular $\rho_t > 0$. Therefore, we can verify
  \begin{equation*}
  \begin{aligned}
    \frac{\eta \Vert \nabla_{\mathbf{w}} \mathcal{R}_t \Vert }{\Vert \mathbf{w}_t \Vert} 
    \ge 
    \frac{2  {\rho}_t {\rho}_t^{\perp}}{
      {\rho}_t^2 
      - \left({\rho}_t^{\perp}\right)^2
    }
  \end{aligned}
  \end{equation*}
  ensures that 
  \begin{equation*}
  \begin{aligned}
    - \eta \left\Vert \nabla_{\mathbf{w}} \mathcal{R}_t \right\Vert \left(
      \rho_t^2 
      - \left(\rho_t^{\perp}\right)^2
    \right)
    + 
    2  \rho_t^{\perp}
    \rho_t \Vert \mathbf{w}_t \Vert
    \le 0
  \end{aligned}.
  \end{equation*}
  This means $\rho_{t+1}^2 - \rho_{t}^2 \le 0$. 
\end{proof}

\section{Detailed Proof for Linear Regression}\label{Detailed_Proof_for_Linear_Regression}

\newtheorem*{restate_lemma_Decompositionofrisk}{Lemma \ref{lemma_Decompositionofrisk} (Decomposition of Mean Square Loss)}
\definecolor{shadecolor}{rgb}{0.92,0.92,0.92}
\begin{shaded}
  \begin{restate_lemma_Decompositionofrisk}
    Let $\ell = \ell_{squ}$, $\hat{\mathbf{w}} = \mathbf{\Sigma}^{-1} \boldsymbol{\mu}$ and $\mathbf{\Sigma} = \mathbf{I}$. 
    Then, the following holds:
    \begin{equation*}
    \begin{aligned}
      \mathcal{R}_t & = 
      (\alpha_t - \rho_t)^{2} + (\rho_t^{\perp})^2 + 
      1 - \left\Vert \hat{\mathbf{w}} \right\Vert^{2}
    \end{aligned}
    \end{equation*}
  \end{restate_lemma_Decompositionofrisk}
\end{shaded}
\begin{proof}
  Since $\mathbf{\Sigma} = \mathbf{I}$, we have 
  \begin{equation*}
  \begin{aligned}
    \rho_t = \frac{\langle \hat{\mathbf{w}}, \mathbf{w}_t \rangle_\mathbf{\Sigma}}{\Vert \mathbf{w}_t \Vert_{\mathbf{\Sigma}}};
    \quad 
    \rho_t^{\perp} = \left\Vert
    \hat{\mathbf{w}} - \rho_t \frac{\mathbf{w}_t}{\Vert \mathbf{w}_t \Vert_\mathbf{\Sigma}}
  \right\Vert_{\mathbf{\Sigma}}
  \end{aligned}
  \end{equation*}
  Therefore, we have
  \begin{equation*}
  \begin{aligned}
    \mathcal{R}_t & = 
    \frac{1}{n} \left\Vert \alpha_t \frac{\tilde{\mathbf{X}}^{T} \mathbf{w}_t}{
      \left\Vert \mathbf{w}_t \right\Vert_{\mathbf{\Sigma}}
    } - \mathbf{1}_{n} \right\Vert^{2}
    \\ & =
    \alpha_t^{2} + 1 - \frac{2\alpha_t}{n} \frac{\mathbf{w}_t^{T} \tilde{\mathbf{X}} \mathbf{1}}{\left\Vert \mathbf{w}_t \right\Vert_{\mathbf{\Sigma}}}
    \\ & =
    \alpha_t^{2} + 1 - 2 \alpha_t \frac{\mathbf{w}_t^{T} \mathbf{u}}{\left\Vert \mathbf{w}_t \right\Vert_{\mathbf{\Sigma}}}
    \\ & =
    \alpha_t^{2} + 1 - 2 \alpha_t \frac{\langle \mathbf{w}_t, \hat{\mathbf{w}} \rangle_{\mathbf{\Sigma}}}{\left\Vert \mathbf{w}_t \right\Vert_{\mathbf{\Sigma}}}
    \\ & =
    \alpha_t^{2} + 1 - 2 \alpha_t \rho_t
    \\ & = 
    (\alpha_t - \rho_t)^{2} + 1 - \rho_t^{2} 
    \\ & = 
    (\alpha_t - \rho_t)^{2} + \left\Vert \hat{\mathbf{w}} \right\Vert_{\mathbf{\Sigma}}^{2} - \rho_t^{2} 
    + 1 - \left\Vert \hat{\mathbf{w}} \right\Vert_{\mathbf{\Sigma}}^{2}
    \\ & = 
    (\alpha_t - \rho_t)^{2} + 
    \left( \rho_t^{\perp} \right)^2
    + 1 - \left\Vert \hat{\mathbf{w}} \right\Vert_{\mathbf{\Sigma}}^{2}
  \end{aligned}
  \end{equation*}
  Note that in $\mathbf{\Sigma}$-inner product space, Pythagorean theorem also holds, 
  which means $\rho_t^2 + \left( \rho_t^\perp \right)^2 = 
  \Vert \hat{\mathbf{w}}\Vert_{\mathbf{\Sigma}}^2$. It is easy to verify that a solution is reached when \( \mathbf{w} \) is colinear with \(\hat{\mathbf{w}} \), which indicates that \( \rho_t = \alpha_t = \Vert \hat{\mathbf{w}} \Vert \) and \( \rho_t^{\perp} = 0 \). Therefore, we have $\inf_{\mathbf{w}, \alpha} \mathcal{R} (\mathbf{w}, \alpha) = 1 - \Vert \hat{\mathbf{w}} \Vert_{\mathbf{\Sigma}}^2$. This completes the proof. 
\end{proof}

\newtheorem*{restate_Dynamics_of_BN_Linear_Regression}{Lemma \ref{lemma:Dynamics_of_BN_Linear_Regression} (The Dynamics of BN Linear Regression)}
\definecolor{shadecolor}{rgb}{0.92,0.92,0.92}
\begin{restate_Dynamics_of_BN_Linear_Regression}
  Let $\ell = \ell_{squ}$, $\hat{\mathbf{w}} = \mathbf{\Sigma}^{-1} \boldsymbol{\mu}$ and $\mathbf{\Sigma} = \mathbf{I}$. 
  Consider the gradient descent (\ref{gradient_updata}), 
  it holds that 
  \begin{equation*}
  \begin{aligned}
    & (1).
    \frac{\rho_{t+1}^{\perp}}{\rho_{t+1}} = 
    \frac{|\hat{\eta}_t - 1|}{1 + \hat{\eta}_t \left(\frac{\rho_t^{\perp}}{\rho_t}\right)^2}
    \frac{\rho_t^{\perp}}{\rho_t}; \quad
    \\ & (2). \alpha_{t+1}  = \alpha_t + \eta_{\alpha} \left( \rho_{t} - \alpha_t \right); \quad
    \\ & (3). \Vert \mathbf{w}_{t+1} \Vert^2 = 
      \Vert \mathbf{w}_{t} \Vert^2
      + 
      \frac{\eta^2 \alpha_t^2}{\Vert \mathbf{w}_{t} \Vert^2} \left(\rho_t^{\perp}\right)^2,
  \end{aligned}
  \end{equation*}
  where $\hat{\eta}_t$ is effective learning rate, defined as $\hat{\eta}_{t} := {\eta \alpha_t \rho_{t}}/{\Vert \mathbf{w}_{t} \Vert^{2}}$. 
\end{restate_Dynamics_of_BN_Linear_Regression}

\pagebreak
\begin{proof}
  \noindent \textbf{(1).}
  We first reformulate the gradient $\nabla_\mathbf{w} \mathcal{R}_t$:
  \begin{equation*}
  \begin{aligned}
    - \nabla_\mathbf{w} \mathcal{R}_t
    & =
    \frac{\alpha}{n \Vert \mathbf{w} \Vert_{\mathbf{\Sigma}}}
    \left(
      \mathbf{I} - \frac{\mathbf{\Sigma} \mathbf{w} \mathbf{w}^{T}}{\Vert \mathbf{w} \Vert_{\mathbf{\Sigma}}^{2}}
    \right)
    \tilde{\mathbf{X}} 
    \boldsymbol{\ell^\prime}\left(
      \frac{
      \alpha \tilde{\mathbf{X}}^T \mathbf{w}
      }{\Vert \mathbf{w} \Vert_\mathbf{\Sigma}}
    \right)
    \\ & =
    \frac{\alpha}{n \Vert \mathbf{w} \Vert_{\mathbf{\Sigma}}}
    \left(
      \mathbf{I} - \frac{\mathbf{\Sigma} \mathbf{w} \mathbf{w}^{T}}{\Vert \mathbf{w} \Vert_{\mathbf{\Sigma}}^{2}}
    \right)
    \tilde{\mathbf{X}} 
    \left(
      \mathbf{1}_{n}
      -
      \alpha \cdot
      \frac{
       \tilde{\mathbf{X}}^T \mathbf{w}
      }{\Vert \mathbf{w} \Vert_\mathbf{\Sigma}}
    \right)
    \\ & =
    \frac{\alpha_t}{\Vert \mathbf{w}_{t} \Vert_{\mathbf{\Sigma}}}
    \left(
      \mathbf{I} - \frac{\mathbf{\Sigma} \mathbf{w}_{t} \mathbf{w}_{t}^{T}}{\Vert \mathbf{w}_{t} \Vert_{\mathbf{\Sigma}}^{2}}
    \right)
    \mathbf{u}
    \\ & = 
    \frac{\alpha_t}{\Vert \mathbf{w}_{t} \Vert_{\mathbf{\Sigma}}}
    \mathbf{\Sigma}
    \left(
      \mathbf{\Sigma}^{-1} \mathbf{u} - \frac{\mathbf{w}^{T} \mathbf{u} \mathbf{w}_{t}}{\Vert \mathbf{w}_{t} \Vert_{\mathbf{\Sigma}}^{2}}
    \right)
    \\ & = 
    \frac{\alpha_t}{\Vert \mathbf{w}_{t} \Vert_{\mathbf{\Sigma}}}
    \mathbf{\Sigma}
    \left(
      \mathbf{\Sigma}^{-1} \mathbf{u} - \frac{\mathbf{w}^{T} \mathbf{\Sigma} \mathbf{\Sigma}^{-1} \mathbf{u} \mathbf{w}_{t}}{\Vert \mathbf{w}_{t} \Vert_{\mathbf{\Sigma}}^{2}}
    \right)
    \\ & = 
    \frac{\alpha_t}{\Vert \mathbf{w}_{t} \Vert_{\mathbf{\Sigma}}}
    \mathbf{\Sigma}
    \left(
      \hat{\mathbf{w}} - \frac{\left\langle \mathbf{w}, \hat{\mathbf{w}}\right\rangle_\mathbf{\Sigma} }{\Vert \mathbf{w}_{t} \Vert_{\mathbf{\Sigma}}^{2}}
      \mathbf{w}_{t}
    \right)
    \\ & = 
    \frac{\alpha_t}{\Vert \mathbf{w}_{t} \Vert}
    \left(
      \hat{\mathbf{w}} - \frac{\left\langle \mathbf{w}, \hat{\mathbf{w}}\right\rangle }{\Vert \mathbf{w}_{t} \Vert^{2}}
      \mathbf{w}_{t}
    \right)
  \end{aligned}
  \end{equation*}
  Therefore, we have 
  \begin{equation*}
  \begin{aligned}
    & \Vert \nabla_\mathbf{w} \mathcal{R}_t \Vert =  \frac{| \alpha_t |}{\Vert \mathbf{w}_{t} \Vert} \rho_t^{\perp};
    \\ & 
    - \left\langle \nabla_\mathbf{w} \mathcal{R}_t, \hat{\mathbf{w}} \right\rangle
    = 
    \frac{\alpha_t}{\Vert \mathbf{w}_{t} \Vert}
    \left\langle 
      \hat{\mathbf{w}} - \frac{\left\langle \mathbf{w}, \hat{\mathbf{w}}\right\rangle }{\Vert \mathbf{w}_{t} \Vert^{2}}
      \mathbf{w}_{t}
    , \hat{\mathbf{w}} \right\rangle
    = 
    \frac{\alpha_t}{\Vert \mathbf{w}_{t} \Vert} \left( \Vert \hat{\mathbf{w}} \Vert^2 - \rho_t^2 \right)
    = 
    \frac{\alpha_t}{\Vert \mathbf{w}_{t} \Vert} \left( \rho_t^{\perp} \right)^2.
  \end{aligned}
  \end{equation*}
  Next, We decompose $\hat{\mathbf{w}}$ orthogonally into two components: one in the $\text{span}\left\{\mathbf{w}_t\right\}$ and the other
  in a direction orthogonal to it:
  \begin{equation*}
  \begin{aligned}
    \mathbf{e}_1 = \frac{\mathbf{w}_t}{\Vert \mathbf{w}_t \Vert};
    \quad
    \mathbf{e}_2 = 
    \left(
      \hat{\mathbf{w}} - \rho_t \mathbf{e}_1
    \right) / \rho_t^{\perp}.
  \end{aligned}
  \end{equation*}
  Recall the gradient descent update:
  \begin{equation*}
  \begin{aligned}
    \mathbf{w}_{t+1} = \mathbf{w}_{t} - \eta 
    \nabla_{\mathbf{w}} \mathcal{R}_t
    = \Vert \mathbf{w}_t \Vert \mathbf{e}_1 - \eta 
    \nabla_{\mathbf{w}} \mathcal{R}_t
  \end{aligned}
  \end{equation*}
  We have 
  \begin{equation*}
  \begin{aligned}
    \langle \mathbf{w}_{t+1}, \hat{\mathbf{w}} \rangle
    & = 
    \langle 
      \Vert \mathbf{w}_t \Vert \mathbf{e}_1 - \eta \nabla_{\mathbf{w}} \mathcal{R}_t, 
      \rho_t \mathbf{e}_1 + \rho_t^{\perp} \mathbf{e}_2
    \rangle
    \\ & = 
    \rho_t \Vert \mathbf{w}_t \Vert
    -
    \eta \rho_t^{\perp} \left\langle \nabla_{\mathbf{w}} \mathcal{R}_t, \mathbf{e}_2\right\rangle
    \\ & = 
    \rho_t \Vert \mathbf{w}_t \Vert
    -
    \eta \left\langle \nabla_{\mathbf{w}} \mathcal{R}_t, \hat{\mathbf{w}} - \rho_t \mathbf{e}_1 \right\rangle
    \\ & = 
    \rho_t \Vert \mathbf{w}_t \Vert
    -
    \eta \left\langle \nabla_{\mathbf{w}} \mathcal{R}_t, \hat{\mathbf{w}} \right\rangle
    \\ & = 
    \rho_t \Vert \mathbf{w}_t \Vert
    +
    \frac{\eta \alpha_t}{\Vert \mathbf{w}_{t} \Vert} \left( \rho_t^{\perp} \right)^2
  \end{aligned}
  \end{equation*}
  and 
  \begin{equation*}
  \begin{aligned}
    \Vert \mathbf{w}_{t+1} \Vert^2 
    & = 
    \Vert \mathbf{w}_{t} \Vert^2
    + 
    \eta^2 \left\Vert \nabla_{\mathbf{w}} \mathcal{R}_t \right\Vert^2 
    - 
    2 \eta \left\langle \mathbf{w}_{t}, \nabla_{\mathbf{w}} \mathcal{R}_t\right\rangle 
    \\ & = 
    \Vert \mathbf{w}_{t} \Vert^2
    + 
    \eta^2 \left\Vert \nabla_{\mathbf{w}} \mathcal{R}_t \right\Vert^2 
    \\ & = 
    \Vert \mathbf{w}_{t} \Vert^2
    + 
    \frac{\eta^2 \alpha_t^2}{\Vert \mathbf{w}_{t} \Vert^2} \left(\rho_t^{\perp}\right)^2,
  \end{aligned}
  \end{equation*}
  where we apply $\left\langle \mathbf{w}_{t}, \nabla_{\mathbf{w}} \mathcal{R}_t\right\rangle = 0$ and 
  $\left\langle \mathbf{e}_{1}, \nabla_{\mathbf{w}} \mathcal{R}_t\right\rangle = 0$. 
  Then, by definition of $\rho_t^\perp$ and $\rho_t$, we have 
  \begin{equation*}
  \begin{aligned}
    \left(\frac{\rho_{t+1}^{\perp}}{\rho_{t+1}}\right)^2
    & = 
    \frac{\Vert \hat{\mathbf{w}} \Vert^2 - \rho_{t+1}^2}{\rho_{t+1}^2}
    \\ & = 
    \frac{\Vert \mathbf{w}_{t+1} \Vert^2 \Vert \hat{\mathbf{w}} \Vert^2 - \left\langle \mathbf{w}_{t+1}, \hat{\mathbf{w}} \right\rangle^2}
    {\left\langle \mathbf{w}_{t+1}, \hat{\mathbf{w}} \right\rangle^2}
    \\ & = 
    \frac{\left(
      \Vert \mathbf{w}_{t} \Vert^2
      + 
      \frac{\eta^2 \alpha_t^2}{\Vert \mathbf{w}_{t} \Vert^2} \left(\rho_t^{\perp}\right)^2
    \right) \left(\rho_t^2 + \left(\rho_t^{\perp}\right)^2\right) - 
    \left(    \rho_t \Vert \mathbf{w}_t \Vert
    +
    \frac{\eta \alpha_t}{\Vert \mathbf{w}_{t} \Vert} \left( \rho_t^{\perp} \right)^2\right)^2}
    {
      \left(    \rho_t \Vert \mathbf{w}_t \Vert
    +
    \frac{\eta \alpha_t}{\Vert \mathbf{w}_{t} \Vert} \left( \rho_t^{\perp} \right)^2\right)^2}
    \\ & = 
    \left(\rho_t^{\perp}\right)^2
    \frac{
      \left(
        \frac{\eta^2 \alpha_t^2}{\Vert \mathbf{w}_t \Vert^2} \rho_t^2 
        +
        \Vert \mathbf{w}_t \Vert^2  
        - 
        2 \eta \alpha_t \rho_t  
      \right)
    }{
      \left(    \rho_t \Vert \mathbf{w}_t \Vert
    +
    \frac{\eta \alpha_t}{\Vert \mathbf{w}_{t} \Vert} \left( \rho_t^{\perp} \right)^2\right)^2
    }
    \\ & = 
    \left(\frac{\rho_t^{\perp}}{\rho_t}\right)^2
    \frac{
      \left(
        \hat{\eta}_t - 1
      \right)^2
    }{
      \left( 1
    +
    \hat{\eta}_t \left( \rho_t^{\perp} / \rho_t\right)^2\right)^2
    }
  \end{aligned}
  \end{equation*}

  \noindent \textbf{(2).}
  Then we consider the update of $\alpha_t$, which is much simpler:
  \begin{equation*}
  \begin{aligned}
    \frac{\partial \mathcal{R}_t}{\partial \mathbf{\alpha}} & = 
    \frac{1}{n}
    \left(
      \frac{
        \tilde{\mathbf{X}}^{T} \mathbf{w}_t
        }{ \Vert \mathbf{w}_t \Vert_\mathbf{\Sigma}}
    \right)^T 
    \boldsymbol{\ell^\prime}\left(
      \frac{
      \alpha_t \tilde{\mathbf{X}}^T \mathbf{w}_t
      }{\Vert \mathbf{w}_t \Vert_\mathbf{\Sigma}}
    \right)
    \\ & = 
    \frac{1}{n}
    \left(
      \frac{
        \tilde{\mathbf{X}}^{T} \mathbf{w}_t
        }{ \Vert \mathbf{w}_t \Vert_\mathbf{\Sigma}}
    \right)^T 
    \left(
      \alpha_t \cdot \frac{
        \tilde{\mathbf{X}}^{T} \mathbf{w}_t
        }{ \Vert \mathbf{w}_t \Vert_\mathbf{\Sigma}} - \mathbf{1}_{n}
    \right)
    \\ & = 
    \alpha_t - \frac{\mathbf{w}_t^T \mathbf{u}}{\Vert \mathbf{w}\Vert_{\mathbf{\Sigma}}}
  \end{aligned}
  \end{equation*}
  Therefore, we have 
  \begin{equation*}
  \begin{aligned}
    \alpha_{t+1} 
    & = 
    \alpha_t - \eta_{\alpha} \frac{\partial \mathcal{R}_t}{\partial \mathbf{\alpha}} 
    = 
    \alpha_t
    -
    \eta_{\alpha}
    \left(
      \alpha_t - \frac{\mathbf{w}_{t}^{T} \mathbf{u}}{\Vert \mathbf{w}_{t} \Vert_{\mathbf{\Sigma}}}
    \right)
    = 
    \alpha_t
    +
    \eta_{\alpha}
    \left(
      \rho_{t} - \alpha_t
    \right)
  \end{aligned}
  \end{equation*}
\end{proof}

\newtheorem*{restate_Occurance_of_Spike_ls}{Theorem \ref{Occurance_of_Spike_ls} (Delayed Onset of the Rising Edge)}
\begin{restate_Occurance_of_Spike_ls}
  Let $\ell = \ell_{squ}$, $\hat{\mathbf{w}} = \mathbf{\Sigma}^{-1} \boldsymbol{\mu}$. Suppose $\mathbf{\Sigma} = \mathbf{I}$. 
  Consider the gradient descent (\ref{gradient_updata}) 
  for $t > t_0$ with $\eta_{\alpha} \in (0, 1)$, where $t_0$ is such that 
  $\rho_{t_0}^{\perp} / \rho_{t_0} \le 1/\sqrt{3}$ and $0 < \alpha_{t_0} < \rho_{t_0}$. The following results hold:
  \begin{enumerate}
    \item \textbf{Condition of No Rising Edge.} 
    If $\frac{\eta}{\| \mathbf{w}_{t_0} \|^2} < \frac{2}{\| \hat{\mathbf{w}} \|^2}$, there shall be no rising edge when $t \ge t_0$.
    \item \textbf{Condition of Delayed Onset.}
    If $\eta$ satisfies $\frac{8 }{\Vert \hat{\mathbf{w}} \Vert^2} 
    < 
    \frac{\eta}{\Vert \mathbf{w}_{t_0} \Vert^{2}} 
    \le 
    C$, 
    the dynamics shall stay in \textit{Falling Edge} for at most $\Delta T_0$ iterations, 
    and then enters \textit{Rising Edge} state. 
    Formally speaking, there exists $t_1 \in (t_0, t_0 + \Delta T_0]$ such that 
    \begin{equation*}
    \begin{aligned}
      & \frac{\rho_{t+1}^{\perp}}{\rho_{t+1}} \le \frac{\rho_{t}^{\perp}}{\rho_{t}} \ \ \forall t \in [t_0, t_1), 
      \ \ 
      \frac{\rho_{t_1+1}^{\perp}}{\rho_{t_1+1}} \ge \frac{\rho_{t_1}^{\perp}}{\rho_{t_1}}
    \end{aligned}
    \end{equation*}
    where 
    \begin{equation*}
    \begin{aligned}
    & C = 
    \min \left(
      \frac{1}{\alpha_{t_0} \rho_{t_0}},
      \frac{3}{16\Vert \hat{\mathbf{w}} \Vert^2}
      \frac{\eta_\alpha}{e^{2} \left( 1 - \alpha_{t_0}/\rho_{t_0} \right)}
    \right);
    \\ & \Delta T_0 \le 
      \left\lfloor
        \ln
        \left(
          \frac{ \eta_{\alpha} \eta \left( 1 - k \right) \Vert \mathbf{w}_{t_0} \Vert^2}{4 
          \eta^{2} \Vert \hat{\mathbf{w}} \Vert^2 \left(\rho_{t_0}^{\perp} / \rho_{t_0}\right)^2 }
        \right) 
        / \eta_\alpha
        + 1
      \right\rfloor.
    \end{aligned}
    \end{equation*}
  \end{enumerate} 
\end{restate_Occurance_of_Spike_ls}
\begin{proof} We first prove the no-rising-edge condition, then the delayed-onset condition.

\noindent\textbf{Condition of No Rising Edge} 
We first prove $\alpha_t \le \Vert \hat{\mathbf{w}} \Vert \ \forall t \ge t_0$ by induction. 
When $t = t_0$, it holds that $\alpha_t \le \rho_t \le \Vert \hat{\mathbf{w}} \Vert$. 
Then we assume $\alpha_t \le \Vert \hat{\mathbf{w}} \Vert$ for $t \ge t_0$ and consider $\alpha_{t+1}$:
\begin{equation*}
\begin{aligned}
  \alpha_{t+1} = \alpha_{t} + \eta_{\alpha} \left( \rho_t - \alpha_t \right) = 
  \left(1 - \eta_{\alpha}\right) \alpha_t + \eta_\alpha \rho_t
\end{aligned}
\end{equation*}
Since $\eta_{\alpha} \in (0, 1)$, we observe that $\alpha_{t+1}$ is, in fact, the convex combination of 
$\rho_t$ and $\alpha_t$, both of which are smaller than $\Vert \hat{\mathbf{w}} \Vert$. 
Therefore, $\alpha_{t+1} \le \Vert \hat{\mathbf{w}} \Vert$. 
By Lemma \ref{lemma:Dynamics_of_BN_Linear_Regression}, if 
$\hat{\eta}_{t}$ is always smaller than $\frac{2}{1 - (\rho_t^{\perp} / \rho_t)^2}$, 
the rising edge can never start. We verify the value of $\hat{\eta}_{t}$. 
For any $t \ge 0$, we have 
\begin{equation*}
\begin{aligned}
  \hat{\eta}_{t} & = 
  \frac{\eta \alpha_{t} \rho_{t}}{\Vert \mathbf{w}_{t} \Vert^{2}} 
  \le 
  \frac{\eta \Vert \hat{\mathbf{w}} \Vert^2}{\Vert \mathbf{w}_{t} \Vert^{2}} 
  \le 
  \frac{\eta \Vert \hat{\mathbf{w}} \Vert^2}{\Vert \mathbf{w}_{t_0} \Vert^{2}} 
  \le 
  2
  \le \frac{2}{1 - \left(\rho_{t}^{\perp}/\rho_{t}\right)^2},
\end{aligned}
\end{equation*}
where the second inequality is by $\Vert \mathbf{w}_{t} \Vert \ge \Vert \mathbf{w}_{t_0} \Vert \ \forall t_0 \le t$ 
and the third inequality is by the upper bound of $\eta$. 

\noindent\textbf{Condition of Delayed Onset.}
  By the upper bound of $\eta$, we can verify that following inequalities hold:
  \begin{equation*}
  \begin{aligned}
    \hat{\eta}_{t_0} & = 
    \frac{\eta \alpha_{t_0} \rho_{t_0}}{\Vert \mathbf{w}_{t_0} \Vert^{2}} 
    \le 
    \frac{2}{1 - \left(\rho_{t_0}^{\perp}/\rho_{t_0}\right)^2},
  \end{aligned}
  \end{equation*}
  which ensures that $\rho_t^{\perp} / \rho_t$ is decreasing when $t = t_0$. 

  Without loss of generality, we treat the time interval $[t_0, t_0 + \Delta T_0)$ as $[0, \Delta T_0)$ by defining 
  $t \leftarrow t - t_0$. This index shift does not affect the correctness of the proof. 
  By the upper bound of $\eta$ in condition (1), we have $\hat{\eta}_0 \le \frac{2}{1 - \left({\rho_0^{\perp}}/{\rho_0}\right)^2 }$, 
  which guarantees that dynamics is still in \textit{Falling Edge} state when $t = 0$. 
  Next, we assume that $\rho_t^{\perp}$ will be decreasing for all $t \ge 0$, then derive a contradiction to identify the delayed onset of the rising edge.
  By this assumption, we have $\forall t \ge 0$, $\rho_{t+1}^{\perp} \le \rho_{t}^{\perp}$ and $\rho_{t+1} \ge \rho_{t}$. 
  By the update of $\alpha_{t}$ and $\eta_{\alpha} \in (0, 1)$, we have 
  \begin{equation*}
  \begin{aligned}
    \alpha_{t+1} & = \alpha_{t} + \eta_{\alpha} \left(
      \rho_t - \alpha_{t}
    \right)
    \\ & \ge 
    \alpha_{t} + \eta_{\alpha} \left(
      \rho_0 - \alpha_{t}
    \right)
  \end{aligned}
  \end{equation*}
  Take the negative of both sides and add $\rho_0$ to obtain
  \begin{equation*}
  \begin{aligned}
    \rho_0 - \alpha_{t+1} & \le 
    \rho_0 -\alpha_{t} - \eta_{\alpha} \left(
      \rho_0 - \alpha_{t}
    \right)
    \\ & \le 
    \left(1 - \eta_{\alpha} \right)\left(\rho_0 -\alpha_{t}\right) 
    \\ & \le 
    \exp(-\eta_{\alpha}) \left(\rho_0 -\alpha_{t}\right)
     \\ & \le 
    \exp(-\eta_{\alpha}(t+1)) \left(\rho_0 -\alpha_{0}\right)
  \end{aligned}
  \end{equation*}
  Then we obtain the lower bound of $\alpha_t$:
  \begin{equation*}
  \begin{aligned}
    \forall t \ge 0, \ \ \alpha_t \ge \rho_0 - e^{-\eta_{\alpha}t} \left(\rho_0 -\alpha_{0}\right)
  \end{aligned}
  \end{equation*}
  Next, we calculate the upper bound of $\Vert \mathbf{w}_t \Vert$:
  \begin{equation*}
  \begin{aligned}
    \forall t \ge 0, \ \ \Vert \mathbf{w}_t \Vert^{2} & = \Vert \mathbf{w}_0 \Vert^{2} + 
    \eta^{2} \sum_{\tau=0}^{t-1} \frac{\alpha_\tau^2 \left(\rho_{\tau}^{\perp}\right)^2}{\Vert \mathbf{w}_\tau \Vert^2}
    \\ & \le 
    \Vert \mathbf{w}_0 \Vert^{2} + 
    \frac{\eta^{2} \Vert \hat{\mathbf{w}} \Vert^2 }{\Vert \mathbf{w}_0 \Vert^2} \sum_{\tau=0}^{t-1} \left(\rho_{\tau}^{\perp}\right)^2
    \\ & \le 
    \Vert \mathbf{w}_0 \Vert^{2} + 
    \frac{\eta^{2} \Vert \hat{\mathbf{w}} \Vert^2 }{\Vert \mathbf{w}_0 \Vert^2} \left(\rho_{0}^{\perp}\right)^2 t
  \end{aligned}
  \end{equation*}
  Now, we can lower bound the effective learning rate $\hat{\eta}_t$: 
  \begin{equation*}
  \begin{aligned}
    \forall t \ge 0, \ \ \hat{\eta}_t & = \frac{\eta \alpha_t \rho_t}{\Vert \mathbf{w}_t \Vert^2}
    \ge 
    \frac{\eta \alpha_t \rho_0}{\Vert \mathbf{w}_0 \Vert^{2} + 
    \frac{\eta^{2} \Vert \hat{\mathbf{w}} \Vert^2 }{\Vert \mathbf{w}_0 \Vert^2} \left(\rho_{0}^{\perp}\right)^2 t}
    \ge 
    \frac{\eta \rho_0^2 - \eta \rho_0 \left(\rho_0 -\alpha_{0}\right) e^{-\eta_{\alpha}t}}
    {\Vert \mathbf{w}_0 \Vert^{2} + 
    \frac{\eta^{2} \Vert \hat{\mathbf{w}} \Vert^2 }{\Vert \mathbf{w}_0 \Vert^2} \left(\rho_{0}^{\perp}\right)^2 t}
  \end{aligned}
  \end{equation*}
  By Lemma.\ref{lemma:Dynamics_of_BN_Linear_Regression}, if $\hat{\eta}_t$ is larger than $\frac{2}{1 - (\rho_t^{\perp}/\rho_t)^2}$, 
  $\rho_{t+1}^{\perp} / \rho_{t+1}$ would be larger than $\rho_{t}^{\perp} / \rho_t$, which means that the delayed rising edge has started. 
  We now investigate whether the lower bound of $\hat{\eta}_t$ can exceed the threshold at some $t$. 
  Therefore, according to the Lemma.\ref{lemma:Dynamics_of_BN_Linear_Regression}, assuming $\rho_t^{\perp}$ is always decreasing leads to a contradiction, 
  indicating the presence of a delayed onset event. 
  We investigate if there exists a $t > 0$ such that 
  \begin{equation}\label{eq:linear_spike_trigger_threshold}
  \begin{aligned}
    \frac{\eta \rho_0^2 - \eta \rho_0 \left(\rho_0 -\alpha_{0}\right) e^{-\eta_{\alpha}t}}
    {\Vert \mathbf{w}_0 \Vert^{2} + 
    \frac{\eta^{2} \Vert \hat{\mathbf{w}} \Vert^2 }{\Vert \mathbf{w}_0 \Vert^2} \left(\rho_{0}^{\perp}\right)^2 t}
    \ge
    \frac{2}{1 - (\rho_t^{\perp})^2/\rho_t^2} 
  \end{aligned}
  \end{equation}
  Since $\rho_0^{\perp} \ge \rho_{t}^{\perp}$, we have 
  \begin{equation}\label{suff_1}
  \begin{aligned}
    (\ref{eq:linear_spike_trigger_threshold}) 
    \ \boldsymbol{\Leftarrow} \ 
    \frac{\eta \rho_0^2 - \eta \rho_0 \left(\rho_0 -\alpha_{0}\right) e^{-\eta_{\alpha}t}}
    {\Vert \mathbf{w}_0 \Vert^{2} + 
    \frac{\eta^{2} \Vert \hat{\mathbf{w}} \Vert^2 }{\Vert \mathbf{w}_0 \Vert^2} \left(\rho_{0}^{\perp}\right)^2 t} \ge \frac{2}{1 - (\rho_0^{\perp})^2/\rho_0^2}
  \end{aligned}
  \end{equation}
  Rearrange the inequality to obtain
  \begin{equation*}
  \begin{aligned}
    \eta \rho_0^2 
    -
    \frac{2 \Vert \mathbf{w}_0 \Vert^{2}}{1 - (\rho_0^{\perp})^2/\rho_0^2}
      >
    \frac{2}{1 - (\rho_0^{\perp})^2/\rho_0^2}
    \frac{\eta^{2} \Vert \hat{\mathbf{w}} \Vert^2 }{\Vert \mathbf{w}_0 \Vert^2} \left(\rho_{0}^{\perp}\right)^2 t
    +
    \eta \rho_0 \left(\rho_0 -\alpha_{0}\right) e^{-\eta_{\alpha}t} 
  \end{aligned}
  \end{equation*}
  Devide $\rho_0^2$ on both side to obtain
  \begin{equation}\label{eq:aaa}
  \begin{aligned}
    \eta 
    -
    \frac{2 \Vert \mathbf{w}_0 \Vert^{2}}{\rho_0^2 - (\rho_0^{\perp})^2}
    &  >
    \frac{2}{1 - (\rho_0^{\perp})^2/\rho_0^2}
    \frac{\eta^{2} \Vert \hat{\mathbf{w}} \Vert^2 }{\Vert \mathbf{w}_0 \Vert^2} \left(\rho_{0}^{\perp} / \rho_0\right)^2 t
    +
    \eta \left( 1 - \alpha_0/\rho_0 \right) e^{-\eta_{\alpha}t} 
    \\ & =
    \frac{2}{1 - (\rho_0^{\perp})^2/\rho_0^2}
    \frac{\eta^{2} \Vert \hat{\mathbf{w}} \Vert^2 }{\Vert \mathbf{w}_0 \Vert^2} \left(\rho_{0}^{\perp} / \rho_0\right)^2 t
    +
    \eta \left( 1 - k \right) e^{-\eta_{\alpha}t}
  \end{aligned}
  \end{equation}
  where $k = \alpha_0 / \rho_0 \in (0, 1)$.
  We next study the left-hand side of (\ref{eq:aaa}):
  \begin{equation}\label{eq:lhs_bbb}
  \begin{aligned}
    \text{LHS of }(\ref{eq:aaa})
    & = 
    \eta 
    -
    \frac{2 \Vert \mathbf{w}_0 \Vert^{2}}{\rho_0^2 - (\rho_0^{\perp})^2}
    \\ & = 
    \eta 
    -
    2 \frac{\Vert \mathbf{w}_0 \Vert^{2}}{\Vert \hat{\mathbf{w}} \Vert^{2}}
    \frac{1 + (\rho_0^{\perp} / \rho_0)^2}{1 - (\rho_0^{\perp} / \rho_0)^2}
    \\ & \ge
    \eta 
    -
    \frac{4 \Vert \mathbf{w}_0 \Vert^{2}}{\Vert \hat{\mathbf{w}} \Vert^2}
    \\ & \ge \frac{1}{2}\eta
  \end{aligned}
  \end{equation}
  Then we look at the right hand side of the (\ref{eq:aaa}):
  \begin{equation*}
  \begin{aligned}
    \text{RHS of }(\ref{eq:aaa}) & = 
    \frac{2 \left(\rho_{0}^{\perp} / \rho_0\right)^2 }{1 - (\rho_0^{\perp})^2/\rho_0^2}
    \frac{\eta^{2} \Vert \hat{\mathbf{w}} \Vert^2 }{\Vert \mathbf{w}_0 \Vert^2} t
    +
    \eta \left( 1 - k \right)
     e^{-\eta_{\alpha}t} 
    \\ & \le 
    \underbrace{
      4 \frac{\eta^{2} \Vert \hat{\mathbf{w}} \Vert^2 }{\Vert \mathbf{w}_0 \Vert^2} \left(\rho_{0}^{\perp} / \rho_0\right)^2 
    }_{B} t
    +
    \underbrace{
      \eta \left( 1 - k \right)
    }_{
      A
    }
     e^{-\eta_{\alpha}t}
  \end{aligned}
  \end{equation*}
  We discuss the minimum of the function $f(t; a, b) = A e^{-\eta_{\alpha} t} + B t$.
  Its minimum is achieved by $t_{\min} = \ln(A \eta_{\alpha} / B) / \eta_{\alpha}$. 
  As $t_{\min}$ is unlikely to be an integer, we calculate $f(t)$ at the integers closest to $t_{\min}$. 
  Since $f(t)$ is monotonically decreasing in $t < t_{\min}$. We have 
  \begin{equation*}
  \begin{aligned}
    f(\lfloor t_{\min} + 1 \rfloor) 
    \le 
    f(t_{\min} + 1)
    & = 
    f(t_{\min}) + f(t_{\min} + 1) - f(t_{\min})
    \\ & = 
    f(t_{\min})
    +
    \Bigl(
      A e^{-\eta_{\alpha} \left(t_{\min} + 1\right)}
      -
      A e^{-\eta_{\alpha} t_{\min}}
    \Bigr)
    \\ & \quad +
    \Bigl(
      B \left(t_{\min} + 1\right)
      -
      B t_{\min}
    \Bigr)
    \\ & = 
    f(t_{\min})
    +
    B \left(
      \frac{e^{-\eta_{\alpha}} - 1}{\eta_{\alpha}} + 1
    \right)
    \\ & \le 
    f(t_{\min})
    \\ & \quad + B
  \end{aligned}
  \end{equation*}
  where the inequality uses $\eta_{\alpha} \in (0, 1)$.
  And
  \begin{equation*}
  \begin{aligned}
    f_{\min} = f(t_{\min})
    =
    \frac{B}{\eta_{\alpha}}
    \ln\!\left(
      \frac{A \eta_{\alpha} e}{B}
    \right).
  \end{aligned}
  \end{equation*}
  \begin{equation}\label{eq:bbb}
  \begin{aligned}
    \inf_{t > 0} \left\{
      \text{RHS of }(\ref{eq:aaa})
    \right\}
    & \le
    f(\lfloor t_{\min} + 1 \rfloor) 
    \\ & \le 
    f(t_{\min}) + B 
    \\ & = 
    \frac{B}{\eta_{\alpha}}
    \ln\left( A \eta_{\alpha} e / B \right) + B 
    \\ & = 
    \frac{B}{\eta_{\alpha}}
    \ln\left( A \eta_{\alpha} e^{1+\eta_{\alpha}} / B \right)
    \\ & \le 
    \sqrt{e^{1+\eta_{\alpha}} A B / \eta_{\alpha}}
    \\ & =
    \sqrt{
      4 \frac{e \eta^{2} \Vert \hat{\mathbf{w}} \Vert^2 }{\eta_\alpha \Vert \mathbf{w}_0 \Vert^2} 
      \eta \left( 1 - k \right)
    }
    \left(\rho_{0}^{\perp} / \rho_0\right)
    \\ & =
    \frac{2\eta \Vert \hat{\mathbf{w}} \Vert }{\Vert \mathbf{w}_0 \Vert} 
    \sqrt{
      \frac{\eta e^{1+\eta_{\alpha}} \left( 1 - k \right)}{\eta_\alpha}
    }
    \left(\rho_{0}^{\perp} / \rho_0\right)
  \end{aligned}
  \end{equation}

  To verify that there exists a $t$ for which the inequality (\ref{eq:linear_spike_trigger_threshold}) holds, we have 
  \begin{equation*}
  \begin{aligned}
    (\ref{eq:linear_spike_trigger_threshold})
    \ \boldsymbol{\Leftarrow} \ 
    (\ref{suff_1})
    \boldsymbol{\Leftrightarrow}
    (\ref{eq:aaa})
    & \ \boldsymbol{\Leftrightarrow} \ 
    \text{LHS of }(\ref{eq:aaa}) \ge 
    \text{RHS of }(\ref{eq:aaa})
    \\ & \ \boldsymbol{\Leftarrow} \ 
    \text{LHS of }(\ref{eq:aaa}) \ge 
    \inf_{t > 0} \left\{
      \text{RHS of }(\ref{eq:aaa})
    \right\}
    \\ & \ \boldsymbol{\Leftarrow} \ 
    \frac{1}{2} \eta \ge 
    \frac{2\eta \Vert \hat{\mathbf{w}} \Vert }{\Vert \mathbf{w}_0 \Vert} 
    \sqrt{
      \frac{\eta e^{1+\eta_{\alpha}} \left( 1 - k \right)}{\eta_\alpha}
    }
    \left(\rho_{0}^{\perp} / \rho_0\right),
    \\ & \ \boldsymbol{\Leftrightarrow} \ 
    \frac{\eta}{\Vert {\mathbf{w}}_0 \Vert^2}
    \le 
    \frac{1}{16\Vert \hat{\mathbf{w}} \Vert^2}
    \frac{\eta_\alpha}{e^{1+\eta_{\alpha}} \left( 1 - k \right)}
    \left(
      \frac{\rho_0}{\rho_{0}^{\perp}}
    \right)^2
    \\ & \ \boldsymbol{\Leftarrow} \ 
    \frac{\eta}{\Vert {\mathbf{w}}_0 \Vert^2}
    \le 
    \frac{3}{16\Vert \hat{\mathbf{w}} \Vert^2}
    \frac{\eta_\alpha}{e^{1+\eta_{\alpha}} \left( 1 - k \right)}
    \\ & \ \boldsymbol{\Leftarrow} \ 
    \frac{\eta}{\Vert {\mathbf{w}}_0 \Vert^2}
    \le 
    \frac{3}{16\Vert \hat{\mathbf{w}} \Vert^2}
    \frac{\eta_\alpha}{e^{2} \left( 1 - k \right)},
  \end{aligned}
  \end{equation*}
  which means after at most $t_{\min}$ iterations, $\rho_t^{\perp}$ will increase. This completes the proof. 
\end{proof}

\newtheorem*{restate_Divergence_lr}{Theorem \ref{Divergence_lr} (Finite-Time Self-Stabilization of the Rising Edge)}
\definecolor{shadecolor}{rgb}{0.92,0.92,0.92}
\begin{shaded}
  \begin{restate_Divergence_lr}
    Let the assumptions of item (2) in Theorem \ref{Occurance_of_Spike_ls} hold, so that a delayed rising edge exists. 
  Then 
  the \textit{Rising Edge} will last for at most $\Delta T_1$ iterations, 
  then it returns to \textit{Falling Edge}. 
  Specifically, there exists a $t_2 \in (t_1, t_1 + \Delta T_1]$ such that 
  \begin{equation*}
  \begin{aligned}
    \frac{\rho_{t+1}^{\perp}}{\rho_{t+1}} \ge \frac{\rho_{t}^{\perp}}{\rho_{t}} \ \ \forall t \in [t_1, t_2) 
    \ \ \text{and} \ \ 
    \frac{\rho_{t_2+1}^{\perp}}{\rho_{t_2+1}} \le \frac{\rho_{t_2}^{\perp}}{\rho_{t_2}},
  \end{aligned}
  \end{equation*}
  where 
  \begin{equation*}
  \begin{aligned}
    \Delta T_1 = 
    \left\lceil
      \frac{1}{4}
      \frac{\Vert \hat{\mathbf{w}} \Vert^4 }{\alpha_{t_1}^2}
        \left(
          1 / (\rho_{t_1}^{\perp})^2 - 1 / \rho_{t_1}^2
        \right)^2
      \right\rceil
        +
        \left\lceil
      \frac{1}{4}
      \frac{\Vert \hat{\mathbf{w}} \Vert^2}{\rho_{t_1}^2}
      \left(
        \rho_{t_1} / \rho_{t_1}^{\perp} - \rho_{t_1}^{\perp} / \rho_{t_1}
      \right)^2
    \right\rceil
  \end{aligned}
  \end{equation*}
  Moreover, we define a time $\phi \in [t_1, t_2]$ as the first moment when $\alpha_t$ 
  catches up with $\rho_t$, \textit{i.e.}, the time such that 
  \( \alpha_t \le \rho_t \ \forall t \in [t_1, \phi] \) and \( \alpha_t \ge \rho_t \ \forall t \in (\phi, t_2) \). 
  Then, the dynamics of $\rho_t^{\perp}/\rho_t$ is given by:
  \begin{equation*}
  \begin{aligned}
    \forall t \in [t_1, \phi], \ \ 
    (\rho_{t}^{\perp}/\rho_{t})^2 
    \le 
    1 - 
    \frac{ 2 \rho_{t_1}^{\perp} \alpha_{t_1}}{\Vert \hat{\mathbf{w}} \Vert^2}
    \sqrt{t - t_1};
    \quad 
    \forall t \in (\phi, t_2], \ \ 
    (\rho_{t}^{\perp}/\rho_{t})^2 
    \le 
    1 - 
    \frac{ 2 \rho_{t_1}^{\perp}}{
      \Vert \hat{\mathbf{w}} \Vert
    }
    \sqrt{t - \phi}.
  \end{aligned}
  \end{equation*}
\end{restate_Divergence_lr}
\end{shaded}
\begin{proof}
  By Lemma \ref{Existence_of_Falling_Edge_lr}, we know that the \textit{Rising Edge} will eventually terminate. 
  We assume that \textit{Rising Edge} ends after \( \Delta T_1 \) iterations, i.e., 
  \[
    \frac{\rho_t^{\perp}}{\rho_t} \le \frac{\rho_{t+1}^{\perp}}{\rho_{t+1}} 
    \quad \forall t \in [t_1, t_1 + \Delta T_1)
    \quad \text{and} \quad 
    \frac{\rho_{t_1 + \Delta T_1}^{\perp}}{\rho_{t_1 + \Delta T_1}} \ge \frac{\rho_{t_1 + \Delta T_1+1}^{\perp}}{\rho_{t_1 + \Delta T_1+1}} .
  \]
  Since $\rho_t$ keeps increasing during $t \in [t_0, t_1)$ and $0 < \alpha_{t_0} < \rho_{t_0}$, 
  the update of $\alpha_t$ implies that $0 < \alpha_{t_1} < \rho_{t_1}$. 
  For simplicity, we shift time so that the \textit{Rising Edge} starts at $t=0$, namely,
  $t \leftarrow t - t_1$. 
  Therefore, we have 
  \[
    \frac{\rho_t^{\perp}}{\rho_t} \le \frac{\rho_{t+1}^{\perp}}{\rho_{t+1}} 
    \quad \forall t \in [0, \Delta T_1),
  \]
  and
  \[
    \frac{\rho_{\Delta T_1 }^{\perp}}{\rho_{\Delta T_1 }} \ge \frac{\rho_{\Delta T_1 +1}^{\perp}}{\rho_{\Delta T_1 +1}},
    \qquad
    0 < \alpha_0 < \rho_0,
  \]
  after the shift $t \leftarrow t - t_1$.
  It must hold that
  \[
  \forall t \in [0, \Delta T_1), \ \  \frac{\hat{\eta}_t - 1}{1 + \hat{\eta}_t \left( \frac{\rho_t^{\perp}}{\rho_t} \right)^2} \ge 1
  \]
  by Lemma \ref{lemma:Dynamics_of_BN_Linear_Regression}, 
  since \( {\rho_t^{\perp}}/{\rho_t} \) keeps increasing for \( t \in [0, \Delta T_1) \), 
  which can be rewritten as
  \begin{equation}
    \forall t < \Delta T_1, \ \  \left( \frac{\rho_t^{\perp}}{\rho_t} \right)^2 \le 1 - \frac{2}{\hat{\eta}_t} \le 1. 
    \label{eq:first}
  \end{equation}
  The inequality above can be used to evaluate the upper bound of \( \left( {\rho_t^{\perp}}/{\rho_t} \right)^2 \). 
  Since \( \alpha_0 < \rho_0 \) while \( \rho_t \) continues to decrease, we divide \( [0, \Delta T_1) \) 
  into two phases, \( P_1 = [0, \phi] \) and \( P_2 = (\phi, \Delta T_1) \), 
  such that for all \( t \in P_1 \), \( \alpha_t \le \rho_t \) and for all \( t \in P_2 \), \( \alpha_t \ge \rho_t \). 
  Note that \( P_2 \) may be empty, which means \( \alpha_t \) never exceeds \( \rho_t \) during the \textit{Rising Edge}.
  
  \noindent \textbf{The bound of \( \boldsymbol{\alpha_t} \)} 
  During \( P_1 \), since \( \alpha_t \) tracks \( \rho_t \) and \( \rho_t \) is always greater than \( \alpha_t \), it follows that \( \alpha_t \) is increasing. 
  From the definition of \( P_1 \), we have \( \alpha_t \le \rho_t \le \Vert \hat{\mathbf{w}} \Vert \), for all \( t \in P_1 \). Therefore, we obtain 
  \begin{equation*}
  \alpha_0 \le \alpha_t \le \Vert \hat{\mathbf{w}} \Vert, \quad \forall t \in P_1.
  \end{equation*}
  Next, consider \( P_2 \). For any \( t \in P_2 \), we have \( \alpha_t \ge \rho_t \). 
  Given the update rule for \( \alpha_t \), it follows that \( \alpha_t \) must decrease for \( t \in P_2 \). 
  Thus, we have 
  $\alpha_t \le \alpha_{\phi} \le \rho_{\phi} \le \rho_0 \le \Vert \hat{\mathbf{w}} \Vert.$
  Therefore, it holds that 
  \begin{equation*}
  \rho_t \le \alpha_t \le \Vert \hat{\mathbf{w}} \Vert, \quad \forall t \in P_2.
  \end{equation*}

  \noindent\textbf{The bound of $\boldsymbol{\Vert \mathbf{w}_t \Vert^2}$.}
  Given any $t \in P_1$, we have 
  \begin{equation*}
  \begin{aligned}
    \Vert \mathbf{w}_t \Vert^2 & = 
    \Vert \mathbf{w}_{0} \Vert^2 
    \\ & \quad
    +
    \eta^2 \sum_{\tau = 0}^{t}  
    \frac{\alpha_{\tau}^2 \left(\rho_{\tau}^{\perp}\right)^2}{\Vert \mathbf{w}_{\tau} \Vert^2}
    \\ & \ge
    \Vert \mathbf{w}_{0} \Vert^2
    +
    \frac{\eta^2}{\Vert \mathbf{w}_{t} \Vert^2}
    \sum_{\tau = 0}^{t}
      \alpha_{\tau}^2 \left(\rho_{\tau}^{\perp}\right)^2
    \\ & \ge
    \Vert \mathbf{w}_{0} \Vert^2
    +
    \frac{
      \eta^2 \left(\rho_{0}^{\perp}\right)^2 \alpha_{0}^2 t
    }{
      \Vert \mathbf{w}_{t} \Vert^2
    }
  \end{aligned}
  \end{equation*}
  We rearrange the inequality to obtain
  \begin{equation*}
  \begin{aligned}
    \Vert \mathbf{w}_t \Vert^4
    -
    \Vert \mathbf{w}_{0} \Vert^2 \Vert \mathbf{w}_t \Vert^2
    -
    \eta^2 (\rho_{0}^{\perp})^2 \alpha_0^2 t
    \ge 0
  \end{aligned}
  \end{equation*}
  Solving the range of $\Vert \mathbf{w}_t \Vert^2$, we have
  \begin{equation*}
  \begin{aligned}
    \Vert \mathbf{w}_t \Vert^2 \ge 
    \frac{1}{2}
    \left(
      \Vert \mathbf{w}_0 \Vert^2 
      +
      \sqrt{
        \Vert \mathbf{w}_0 \Vert^4
        +
        4  \eta^2 (\rho_{0}^{\perp})^2 \alpha_0^2 t
      }
    \right),
    \\ & \qquad \forall t \in P_1
  \end{aligned}
  \end{equation*}
  Then consider $t \in P_2$, we have 
  \begin{equation*}
  \begin{aligned}
    \Vert \mathbf{w}_t \Vert^2 
    \ge 
    \Vert \mathbf{w}_{0} \Vert^2 + 
    \frac{\eta^2}{\Vert \mathbf{w}_{t} \Vert^2}
    \sum_{\tau = 0}^{t} \alpha_{\tau}^2 \left(\rho_{\tau}^{\perp}\right)^2
    \\ & \ge 
    \Vert \mathbf{w}_{0} \Vert^2 + 
    \frac{\eta^2}{\Vert \mathbf{w}_{t} \Vert^2}
    \sum_{\tau = \phi}^{t} \alpha_{\tau}^2 \left(\rho_{\tau}^{\perp}\right)^2
    \\ & \ge
    \Vert \mathbf{w}_{0} \Vert^2
    +
    \frac{
      \eta^2 \left(\rho_{0}^{\perp}\right)^2 \rho_{t}^2
    }{
      \Vert \mathbf{w}_{t} \Vert^2
    }
    \left( t - \phi \right)
  \end{aligned}
  \end{equation*}
  and
  \begin{equation*}
  \begin{aligned}
    \Vert \mathbf{w}_t \Vert^2 \ge 
    \frac{1}{2}
    \left(
      \Vert \mathbf{w}_0 \Vert^2 + 
      \sqrt{
        \Vert \mathbf{w}_0 \Vert^4 + 4  
        \eta^2 \left(\rho_{0}^{\perp}\right)^2
        \rho_{t}^2 \left(
          t - \phi
        \right)
      }
    \right),
    \ \forall t \in P_2
  \end{aligned}
  \end{equation*}

  \noindent\textbf{The upper bound of $\boldsymbol{\hat{\eta}_t}$.} 
  Given any $t \in P_1$, we have 
  \begin{equation*}
  \begin{aligned}
    \hat{\eta}_t = \eta \frac{\alpha_t \rho_t}{\Vert \mathbf{w}_t \Vert^2}
    \le 
    \frac{\eta \Vert \hat{\mathbf{w}} \Vert^2}{
      \left(
        \Vert \mathbf{w}_0 \Vert^2 + 
        \sqrt{
          \Vert \mathbf{w}_0 \Vert^4 + 4 \eta^2 (\rho_{0}^{\perp})^2 \alpha_0^2 t
        }
      \right) / 2
    }
    \le 
    \frac{\eta \Vert \hat{\mathbf{w}} \Vert^2}{
      \sqrt{
        4 \eta^2 (\rho_{0}^{\perp})^2 \alpha_0^2 t
      }
    / 2}
    = 
    \frac{\Vert \hat{\mathbf{w}} \Vert^2 }{
        \rho_{0}^{\perp} \alpha_0
    }
    \frac{1}{ \sqrt{t}}
  \end{aligned}
  \end{equation*}
  And given any $t \in P_2$, we have 
  \begin{equation*}
  \begin{aligned}
    \hat{\eta}_t = \eta \frac{\alpha_t \rho_t}{\Vert \mathbf{w}_t \Vert^2}
    & \le 
    \frac{\eta \Vert \hat{\mathbf{w}} \Vert \rho_t}{
      \left(
        \Vert \mathbf{w}_0 \Vert^2 + 
        \sqrt{
          \Vert \mathbf{w}_0 \Vert^4 + 4  
          \eta^2 \left(\rho_{0}^{\perp}\right)^2
          \rho_{t}^2 \left(
            t - \phi
          \right)
        }
      \right) / 2
    }
    \\ & \le 
    \frac{\eta \Vert \hat{\mathbf{w}} \Vert \rho_t}{
      \sqrt{
        4 \eta^2 \left(\rho_{0}^{\perp}\right)^2
        \rho_{t}^2 \left(
          t - \phi 
        \right)
      }
    / 2}
    \\ & = 
    \frac{\Vert \hat{\mathbf{w}} \Vert}{\rho_{0}^{\perp}}
    \frac{1}{\sqrt{t - \phi }}
  \end{aligned}
  \end{equation*}

  \noindent\textbf{The minimum feasible $\boldsymbol{\Delta T_1}$.} 
  By Lemma \ref{lemma:Dynamics_of_BN_Linear_Regression}, we aim to find a $t$ such that 
  \begin{equation}\label{eq:falling_edge_feasibility}
  \begin{aligned}
    \hat{\eta}_t \le \frac{2}{1 - (\rho_t^{\perp})^2 / \rho_t^2} 
  \end{aligned},
  \end{equation}
  which is the feasible solution of $\Delta T_1$. 
  If $P_2$ is empty, which means $[0, \Delta T_1) = P_1$, then we have
  \begin{equation*}
  \begin{aligned}
    (\ref{eq:falling_edge_feasibility})
    & \ \boldsymbol{\Leftarrow} \ 
    \hat{\eta}_t \le \frac{2}{1 - (\rho_0^{\perp})^2 / \rho_0^2}
    \\ & \  \boldsymbol{\Leftarrow} \ 
    \frac{\Vert \hat{\mathbf{w}} \Vert^2 }{
      \rho_{0}^{\perp} \alpha_0
    }
    \frac{1}{ \sqrt{t}}
    \le \frac{2}{1 - (\rho_0^{\perp})^2 / \rho_0^2}
    \\ & \  \boldsymbol{\Leftarrow} \ 
    t \ge 
    \left\lceil
      \frac{1}{4}
      \left(
        1 / (\rho_{0}^{\perp})^2 - 1 / \rho_0^2
      \right)^2
      \frac{\Vert \hat{\mathbf{w}} \Vert^4 }{
        \alpha_0^2
      }
    \right\rceil.
  \end{aligned}
  \end{equation*}
  Next, consider the case that \( P_2 \) is not empty. 
  In this case, it must be that \( \phi < \left\lceil
  \frac{1}{4}
  \left(
    1 / (\rho_{0}^{\perp})^2 - 1 / \rho_0^2
  \right)^2
  \frac{\Vert \hat{\mathbf{w}} \Vert^4 }{
    \alpha_0^2
  }
  \right\rceil \), 
  because otherwise the \textit{Rising Edge} would terminate within \( P_1 \). 
  For \( t \in P_2 \), we have
  \begin{equation*}
  \begin{aligned}
    (\ref{eq:falling_edge_feasibility})
    & \ \boldsymbol{\Leftarrow} \ 
    \hat{\eta}_t \le \frac{2}{1 - (\rho_0^{\perp})^2 / \rho_0^2}
    \\ & \  \boldsymbol{\Leftarrow} \ 
    \frac{\Vert \hat{\mathbf{w}} \Vert}{\rho_{0}^{\perp}}
    \frac{1}{\sqrt{t - \phi}}
    \le \frac{2}{1 - (\rho_0^{\perp})^2 / \rho_0^2}
    \\ & \  \boldsymbol{\Leftrightarrow} \ 
    t \ge
    \phi
    +
    \frac{1}{4}
    \frac{\Vert \hat{\mathbf{w}} \Vert^2}{\rho_{0}^2}
    \left(
      \rho_0 / \rho_0^{\perp} - \rho_0^{\perp} / \rho_0
    \right)^2
    \\ & \  \boldsymbol{\Leftarrow} \ 
    t \ge
    \left\lceil
      \frac{1}{4}
      \left(
        1 / (\rho_{0}^{\perp})^2 - 1 / \rho_0^2
      \right)^2
      \frac{\Vert \hat{\mathbf{w}} \Vert^4 }{
        \alpha_0^2
      }
    \right\rceil
    +
    \left\lceil
    \frac{1}{4}
    \frac{\Vert \hat{\mathbf{w}} \Vert^2}{\rho_{0}^2}
    \left(
      \rho_0 / \rho_0^{\perp} - \rho_0^{\perp} / \rho_0
    \right)^2
    \right\rceil
  \end{aligned}
  \end{equation*}
  
  \noindent\textbf{The upper bound of $\boldsymbol{\rho_t^{\perp} / \rho_t}$.}
  We give the upper bound of $\rho_t^{\perp} / \rho_t$. 
  Recall $(\rho_{t}^{\perp}/\rho_{t})^2 \le 1 - 2/\hat{\eta}_t$, we have 
  \begin{equation*}
  \begin{aligned}
    (\rho_{t}^{\perp}/\rho_{t})^2 
    \le 1 - 2/\hat{\eta}_t 
    \le 
    1 - 
    \frac{ 2 \rho_{0}^{\perp} \alpha_0}{\Vert \hat{\mathbf{w}} \Vert^2}
    \sqrt{t}, \  \forall t \in P_1
  \end{aligned}
  \end{equation*}
  and
  \begin{equation*}
  \begin{aligned}
    (\rho_{t}^{\perp}/\rho_{t})^2 
    \le 1 - 2/\hat{\eta}_t 
    \le 
    1 - 
    \frac{ 2 \rho_{0}^{\perp}}{
      \Vert \hat{\mathbf{w}} \Vert
    }
    \sqrt{t - \phi}
    \ \forall t \in P_2
  \end{aligned}
  \end{equation*} 
\end{proof}

\definecolor{shadecolor}{rgb}{0.92,0.92,0.92}
\begin{shaded}
\begin{lemma}[Existence of the Falling Edge of Loss Spike]\label{Existence_of_Falling_Edge_lr}
  Let $\ell = \ell_{squ}$, $\hat{\mathbf{w}} = \mathbf{\Sigma}^{-1} \boldsymbol{\mu}$ and $\mathbf{\Sigma} = \mathbf{I}$. 
  Consider the gradient descent (\ref{gradient_updata}).
  If the initial state satisfies $\hat{\eta}_0 > \frac{2}{1 - \left(\rho_0^{\perp}/\rho_0 \right)^2}$ and $\eta_{\alpha} \in (0, 1)$, 
  then there must exist a $t > 0$ such that $\hat{\eta}_t < \frac{2}{1 - \left(\rho_t^{\perp}/\rho_t \right)^2}$.
\end{lemma}
\end{shaded}

\begin{proof}
  According to Lemma \ref{lemma:Dynamics_of_BN_Linear_Regression}, the condition 
  $\hat{\eta}_0 \ge \frac{2}{1 - ( {\rho_0^{\perp}}/{\rho_0} )^2}$ 
  implies that $\rho_0$ is decreasing. 
  We assume that $\rho_t^{\perp}$ is monotonically increasing.
  Then we prove the result by leading to a contradiction based on this assumption. 
  Since $\rho_t^{\perp}$ is bounded above by $\Vert \hat{\mathbf{w}} \Vert$, it must converge as $t \to \infty$. 
  There are two possible cases: either $\rho_t^{\perp}$ converges to $\Vert \hat{\mathbf{w}} \Vert$, or it does not. 

  \noindent \textbf{Case 1.} 
  If $\rho_t^{\perp} \rightarrow \rho_{\infty}^{\perp} ( \rho_{\infty}^{\perp} < \Vert \hat{\mathbf{w}}\Vert)$, 
  we prove $\Vert \mathbf{w}_t \Vert$ can be large arbitrarily. 
  \begin{equation*}
  \begin{aligned}
    \Vert \mathbf{w}_t \Vert^{2} 
    & = \Vert \mathbf{w}_{t-1} \Vert^2 + \eta^{2} \frac{\alpha_{t-1}^2 \left(\rho_{t-1}^{\perp}\right)^2 }{\Vert \mathbf{w}_{t-1} \Vert^2}
    \\ & =
    \Vert \mathbf{w}_{0} \Vert^2 + \eta^{2} \sum_{\tau=0}^{t-1} \frac{\alpha_{\tau}^2 \left(\rho_{\tau}^{\perp}\right)^2 }{\Vert \mathbf{w}_{\tau} \Vert^2}
  \end{aligned}
  \end{equation*}
  We can assume that $\inf_{t \ge 0} \{\alpha_t\} = \alpha_{\min}$ is strictly larger than $0$. 
  Because if not, it means that there exists a time $t > 0$ such that $\alpha_t$ can be arbitrarily small, 
  and therefore $\hat{\eta}_t$ can be smaller than $2$, which is a contradiction. 
  Therefore, we have 
  \begin{equation*}
    \begin{aligned}
      \Vert \mathbf{w}_t \Vert^{2} 
      & =
      \Vert \mathbf{w}_{0} \Vert^2 + \eta^{2} \sum_{\tau=0}^{t-1} \frac{\alpha_{\tau}^2 \left(\rho_{\tau}^{\perp}\right)^2 }{\Vert \mathbf{w}_{\tau} \Vert^2}
      \\ & \ge 
      \Vert \mathbf{w}_{0} \Vert^2 + \eta^{2} \alpha_{\min}^2 \sum_{\tau=0}^{t-1} \frac{ \left(\rho_{\tau}^{\perp}\right)^2 }{\Vert \mathbf{w}_{\tau} \Vert^2}
      \\ & \ge 
      \Vert \mathbf{w}_{0} \Vert^2 + 
      \frac{ \eta^{2} \alpha_{\min}^2 }{\Vert \mathbf{w}_{t} \Vert^2}
      \sum_{\tau=0}^{t-1} \left(\rho_{\tau}^{\perp}\right)^2
      \\ & \ge 
      \Vert \mathbf{w}_{0} \Vert^2 + 
      \frac{ \eta^{2} \alpha_{\min}^2 }{\Vert \mathbf{w}_{t} \Vert^2}
      \left(\rho_{0}^{\perp}\right)^2  t
    \end{aligned}
    \end{equation*}
    Rearrange the above inequality to obtain 
    \begin{equation*}
    \begin{aligned}
      \Vert \mathbf{w}_t \Vert^{4} - \Vert \mathbf{w}_{t} \Vert^2 \Vert \mathbf{w}_{0} \Vert^2 -
      \eta^{2} \inf_{t \ge 0} \left\{\alpha_{\tau}^2\right\}
      \left(\rho_{0}^{\perp}\right)^2  t
      \ge 0
    \end{aligned}
    \end{equation*}
    Then we solve the above inequality:
    \begin{equation*}
    \begin{aligned}
      \Vert \mathbf{w}_t \Vert^{2} \ge 
      \frac{1}{2} \left(
        \Vert \mathbf{w}_{0} \Vert^2 + \sqrt{
          \Vert \mathbf{w}_{0} \Vert^4 + 4 \eta^{2} \inf_{t \ge 0} \left\{\alpha_{\tau}^2\right\}
          \left(\rho_{0}^{\perp}\right)^2  t
        }
      \right)
    \end{aligned}
    \end{equation*}
    Therefore, there always exists a large enough $t$ such that $\Vert \mathbf{w}_t \Vert$ is very large and $\hat{\eta}_t \le 2$, which means 
    after that, $\rho_t^{\perp}$ is no longer increasing. This leads to contradiction. 

    \noindent \textbf{Case 2.} 
    If $\boldsymbol{\rho_t^{\perp} \rightarrow \Vert \hat{\mathbf{w}} \Vert}$, we have $\rho_t \rightarrow 0$. 
    Therefore $\hat{\eta}_t = \eta \frac{\alpha_t \rho_t}{\Vert \mathbf{w}_t \Vert^2} \rightarrow 0$. 
    And before that, $\hat{\eta}_t$ should have already been less than $2$ and $\rho_t^{\perp}$ is decreasing. 
\end{proof}

\section{Detailed Proof for Logistic Regression}\label{Detailed_Proof_for_Logistic_Regression}
\newtheorem*{restate_logit_bound}{Lemma \ref{logit_bound}}
\definecolor{shadecolor}{rgb}{0.92,0.92,0.92}
\begin{shaded}
\begin{restate_logit_bound}
  Let $\ell$ be $\ell_{log}$ and $\hat{\mathbf{w}}$ be the SVM solution as presented in Definition \ref{margin}. 
  Suppose Assumptions \ref{Overparameterization} and \ref{Logistic_Setup} hold.
  Consider the gradient descent (\ref{gradient_updata}), for any $t \ge 0$,
  if $\rho_t > 0$, it holds that 
  \begin{equation*}
  \begin{aligned}
    (1). &
    \left|
      \left\Vert \mathbf{w}_t \right\Vert_{\mathbf{\Sigma}}
      -
      \gamma^2 \left\Vert \mathbf{w}_t \right\Vert
      \cdot
      \rho_t
    \right|
    \le
    2 \sqrt{2} \lambda_{\max} \cdot \gamma \frac{\Vert \mathbf{w}_t \Vert^2}{\Vert \mathbf{w}_t \Vert_{\mathbf{\Sigma}}}
    \cdot \rho_t^{\perp}; 
    \\ (2). &
    \left| y_i \langle \mathbf{w}_t, \mathbf{x}_i \rangle  - \gamma^2 \Vert \mathbf{w}_t \Vert \cdot \rho_t \right|
    \le 
    \sqrt{\lambda_{\max}} \gamma \Vert \mathbf{w}_t \Vert \cdot \rho_t^{\perp}, 
    \ \forall i \in [n],
  \end{aligned}
  \end{equation*}
\end{restate_logit_bound}
\end{shaded}
\begin{proof}
  In the proof, we ignore the subscript of $\mathbf{w}_t$. 
  For later use, we introduce $\tilde{\rho} \left(\mathbf{w}\right)$
  and $\tilde{\rho}^{\perp} \left(\mathbf{w}\right)$. 
  \begin{equation*}
  \begin{aligned}
    \tilde{{\rho}}(\mathbf{w}) := \frac{\langle \mathbf{w}, \hat{\mathbf{w}} \rangle}{\Vert \hat{\mathbf{w}} \Vert};
    \quad
    \tilde{{\rho}}^{\perp}(\mathbf{w}) := 
    \left\Vert {\mathbf{w}} - \tilde{\rho}(\mathbf{w}) \frac{\hat{\mathbf{w}}}{\Vert \hat{\mathbf{w}} \Vert} \right\Vert,
  \end{aligned}
  \end{equation*}
  It is easy to see that $\tilde{\rho}(\mathbf{w})$ represents the length of $\mathbf{w}$ in $\text{span} \{\hat{\mathbf{w}}\}$, 
  while $\tilde{\rho}^{\perp}(\mathbf{w})$ corresponds to the length of $\mathbf{w}$ in $\text{span}^{\perp} \{\hat{\mathbf{w}}\}$. 
  Moreover, the definitions imply
  $\rho(\mathbf{w}) \Vert \hat{\mathbf{w}} \Vert = \tilde{\rho}(\mathbf{w}) \Vert \mathbf{w} \Vert$
  and
  $\rho^{\perp}(\mathbf{w}) \Vert \hat{\mathbf{w}} \Vert = \tilde{\rho}^{\perp}(\mathbf{w}) \Vert \mathbf{w} \Vert$.
  Therefore, we have 
  \begin{equation*}
  \begin{aligned}
    & \cos \angle \left(\mathbf{w},  \hat{\mathbf{w}}\right) 
    = 
    {\rho}(\mathbf{w}) / \Vert \hat{\mathbf{w}} \Vert 
    =
    {\tilde{\rho}}(\mathbf{w}) / \Vert {\mathbf{w}} \Vert;
    \\ &
    \sin \angle \left(\mathbf{w},  \hat{\mathbf{w}}\right) 
    = {\rho}^{\perp}(\mathbf{w}) / \Vert \hat{\mathbf{w}} \Vert
    =
    \tilde{{\rho}}^{\perp}(\mathbf{w}) / \Vert {\mathbf{w}} \Vert.
  \end{aligned}
  \end{equation*}
  We decompose $\mathbf{w}$ orthogonally into two components: one in the same direction as $\hat{\mathbf{w}}$, and the other in a direction orthogonal to it. 
  Recall the definition of $\tilde{{\rho}}(\mathbf{w})$ and $\tilde{{\rho}}^{\perp}(\mathbf{w})$, we can define 
  \begin{equation*}
  \begin{aligned}
    \tilde{\mathbf{e}}_1 = \hat{\mathbf{w}} / \Vert \hat{\mathbf{w}} \Vert;
    \quad
    \tilde{\mathbf{e}}_2 = \left( \mathbf{w} - \tilde{{\rho}}(\mathbf{w}) \tilde{\mathbf{e}}_1 \right) / \tilde{\rho}^{\perp} \left(\mathbf{w}\right),
  \end{aligned}
  \end{equation*}
  and by Cauchy Inequality we have the following bounds
  \begin{equation}\label{eq:span_sigma_eigen_bounds}
  \begin{aligned}
    \left\Vert \tilde{\mathbf{e}}_1 \right\Vert_{\mathbf{\Sigma}}^2 = 
    \left(
      \frac{
        \left\Vert \hat{\mathbf{w}} \right\Vert_{\mathbf{\Sigma}}
      }{
        \left\Vert \hat{\mathbf{w}} \right\Vert
      }
    \right)^2 = \gamma^2;
    \\ 
    | \langle \tilde{\mathbf{e}}_1, \tilde{\mathbf{e}}_2 \rangle_{\mathbf{\Sigma}} | \le \lambda_{\max};
    \\ 
    \lambda_{\min} \le \left\Vert  \tilde{\mathbf{e}}_2 \right\Vert_{\mathbf{\Sigma}}^2 \le \lambda_{\max}.
  \end{aligned}
  \end{equation}

  \noindent
  \textbf{(1).}
  Therefore 
  \begin{equation*}
  \begin{aligned}
    \left\Vert \mathbf{w} \right\Vert_{\mathbf{\Sigma}}^2
    & = \left\Vert  
    \tilde{{\rho}}(\mathbf{w}) \tilde{\mathbf{e}}_1 + 
    \tilde{\rho}^{\perp} \left(\mathbf{w}\right) \tilde{\mathbf{e}}_2
    \right\Vert_{\mathbf{\Sigma}}^2
    \\ & =
    \left(
      \tilde{{\rho}}(\mathbf{w})
    \right)^2
    \left\Vert \tilde{\mathbf{e}}_1 \right\Vert_{\mathbf{\Sigma}}^2
    +
    \left(
      \tilde{\rho}^{\perp} \left(\mathbf{w}\right)
    \right)^2
    \left\Vert \tilde{\mathbf{e}}_2 \right\Vert_{\mathbf{\Sigma}}^2
    + 
    2 \tilde{\rho} \left(\mathbf{w}\right) \tilde{\rho}^{\perp} \left(\mathbf{w}\right) 
    \langle \tilde{\mathbf{e}}_1, \tilde{\mathbf{e}}_2 \rangle_{\mathbf{\Sigma}}
    \\ & =
    \gamma^{2}
    \left(
      \tilde{{\rho}}(\mathbf{w})
    \right)^2
    +
    \left(
      \tilde{\rho}^{\perp} \left(\mathbf{w}\right)
    \right)^2
    \left\Vert \tilde{\mathbf{e}}_2 \right\Vert_{\mathbf{\Sigma}}^2
    + 
    2 \tilde{\rho} \left(\mathbf{w}\right) \tilde{\rho}^{\perp} \left(\mathbf{w}\right) 
    \langle \tilde{\mathbf{e}}_1, \tilde{\mathbf{e}}_2 \rangle_{\mathbf{\Sigma}}
  \end{aligned}
  \end{equation*}
  Plug the bounds in (\ref{eq:span_sigma_eigen_bounds}) into above equation, we have 
  \begin{equation*}
  \begin{aligned}
    \left\Vert \mathbf{w} \right\Vert_{\mathbf{\Sigma}}^2 & \le 
    \left(
      \tilde{{\rho}}(\mathbf{w})
    \right)^2
    \gamma^2
    +
    \lambda_{\max}
    \left(
      \tilde{\rho}^{\perp} \left(\mathbf{w}\right)
    \right)^2
    + 
    2 \lambda_{\max} 
    \tilde{\rho} \left(\mathbf{w}\right) \tilde{\rho}^{\perp} \left(\mathbf{w}\right)
    \\ 
    \left\Vert \mathbf{w} \right\Vert_{\mathbf{\Sigma}}^2 & \ge
    \left(
      \tilde{{\rho}}(\mathbf{w})
    \right)^2
    \gamma^2
    +
    \lambda_{\min}
    \left(
      \tilde{\rho}^{\perp} \left(\mathbf{w}\right)
    \right)^2
    - 
    2 \lambda_{\max} 
    \tilde{\rho} \left(\mathbf{w}\right) \tilde{\rho}^{\perp} \left(\mathbf{w}\right)
    \\ & \ge
    \left(
      \tilde{{\rho}}(\mathbf{w})
    \right)^2
    \gamma^2
    - 
    \lambda_{\max}
    \left(
      \tilde{\rho}^{\perp} \left(\mathbf{w}\right)
    \right)^2
    - 
    2 \lambda_{\max} 
    \tilde{\rho} \left(\mathbf{w}\right) \tilde{\rho}^{\perp} \left(\mathbf{w}\right)
  \end{aligned}
  \end{equation*}
  Combine them to give 
  \begin{equation*}
  \begin{aligned}
    \left|
      \left\Vert \mathbf{w} \right\Vert_{\mathbf{\Sigma}}^2
      -
      \left(
        \tilde{{\rho}}(\mathbf{w})
      \right)^2
      \gamma^2
    \right|
    & \le 
    \lambda_{\max}
    \left(
      \tilde{\rho}^{\perp} \left(\mathbf{w}\right)
    \right)^2
    +
    2 \lambda_{\max} 
    \tilde{\rho} \left(\mathbf{w}\right) \tilde{\rho}^{\perp} \left(\mathbf{w}\right)
    \\ & \le 
    2 \lambda_{\max} 
    \tilde{\rho}^{\perp} \left(\mathbf{w}\right)
    \left(
      \tilde{\rho}^{\perp} \left(\mathbf{w}\right)
      +
      \tilde{\rho} \left(\mathbf{w}\right)
    \right)
  \end{aligned}
  \end{equation*}
  Recall the definitions of $\sin \angle \left(\mathbf{w},  \hat{\mathbf{w}}\right)$ and $\cos \angle \left(\mathbf{w},  \hat{\mathbf{w}}\right)$, 
  we have 
  \begin{equation*}
  \begin{aligned}
    \left|
      \left\Vert \mathbf{w} \right\Vert_{\mathbf{\Sigma}}^2
      -
      \cos^2 \angle \left(\mathbf{w},  \hat{\mathbf{w}}\right)
      \left\Vert \mathbf{w} \right\Vert^2
      \gamma^2
    \right|
    & \le
    2 \lambda_{\max} 
    \Vert \mathbf{w} \Vert^2
    \sin \angle \left(\mathbf{w},  \hat{\mathbf{w}}\right)
    \left(
      \sin \angle \left(\mathbf{w},  \hat{\mathbf{w}}\right)
      +
      \cos \angle \left(\mathbf{w},  \hat{\mathbf{w}}\right)
    \right)
    \\ & \le
    2 \sqrt{2} \lambda_{\max} 
    \Vert \mathbf{w} \Vert^2
    \sin \angle \left(\mathbf{w},  \hat{\mathbf{w}}\right),
  \end{aligned}
  \end{equation*}
  where the second inequality is by the range of $\angle \left(\mathbf{w}, \hat{\mathbf{w}}\right)$. By our definition, 
  $\angle \left(\mathbf{w}, \hat{\mathbf{w}}\right) \in [0, \pi]$. 
  Next, we apply $\sin \angle \left(\mathbf{w},  \hat{\mathbf{w}}\right) = \rho^{\perp} \left(\mathbf{w}\right) / \hat{\mathbf{w}}$
  and 
  $\cos \angle \left(\mathbf{w},  \hat{\mathbf{w}}\right) = \rho \left(\mathbf{w}\right) / \hat{\mathbf{w}}$
  to give 
  \begin{equation*}
  \begin{aligned}
    \left|
      \left\Vert \mathbf{w} \right\Vert_{\mathbf{\Sigma}}^2
      -
      \left(
        \gamma^2 \left\Vert \mathbf{w} \right\Vert
        \cdot
        \rho \left(\mathbf{w}\right)
      \right)^2
    \right|
    \le
    2 \sqrt{2} \lambda_{\max} 
    \gamma \Vert \mathbf{w} \Vert^2
    \cdot \rho^{\perp} \left(\mathbf{w}\right)
  \end{aligned}
  \end{equation*}
  Then we bound $\left|
      \left\Vert \mathbf{w} \right\Vert_{\mathbf{\Sigma}}
      -
      \gamma^2 \left\Vert \mathbf{w} \right\Vert
      \cdot
      \rho \left(\mathbf{w}\right)
  \right|$,
  \begin{equation*}
  \begin{aligned}
    \left|
      \left\Vert \mathbf{w} \right\Vert_{\mathbf{\Sigma}}
      -
      \gamma^2 \left\Vert \mathbf{w} \right\Vert
      \cdot
      \rho \left(\mathbf{w}\right)
    \right|
    & = 
    \frac{
      \left|
        \left\Vert \mathbf{w} \right\Vert_{\mathbf{\Sigma}}^2
          -
          \left(
            \gamma^2 \left\Vert \mathbf{w} \right\Vert
            \cdot
            \rho \left(\mathbf{w}\right)
          \right)^2
      \right|
    }{
      \left|
        \left\Vert \mathbf{w} \right\Vert_{\mathbf{\Sigma}}
        +
        \gamma^2 \left\Vert \mathbf{w} \right\Vert
        \cdot
        \rho \left(\mathbf{w}\right)
    \right|
    }
    \\ & \le 
    \left|
      \left\Vert \mathbf{w} \right\Vert_{\mathbf{\Sigma}}^2
        -
        \left(
          \gamma^2 \left\Vert \mathbf{w} \right\Vert
          \cdot
          \rho \left(\mathbf{w}\right)
        \right)^2
    \right| / \left\Vert \mathbf{w} \right\Vert_{\mathbf{\Sigma}}
    \\ & \le 
    2 \sqrt{2} \lambda_{\max} 
    \gamma \Vert \mathbf{w} \Vert^2
    \cdot \rho^{\perp} \left(\mathbf{w}\right)
     / \left\Vert \mathbf{w} \right\Vert_{\mathbf{\Sigma}}
  \end{aligned}
  \end{equation*}

  \noindent \textbf{(2).}
  For the second result, by Cauchy inequality, for any $i \in [n]$, it holds that 
  \begin{equation*}
  \begin{aligned}
    y_i \langle \mathbf{w}, \mathbf{x}_i \rangle 
    = 
    y_i 
    {\tilde{\rho}}(\mathbf{w}) {\tilde{\rho}}(\mathbf{x}_i)
    + 
    y_i
    \left\langle \mathcal{P}^{\perp}(\mathbf{w}), \mathcal{P}^{\perp}(\mathbf{x}_i) \right\rangle
    \left\{ \begin{array}{cl}
      \le {\tilde{\rho}}(\mathbf{w}) \gamma + \Vert \mathbf{x}_i \Vert  \cdot {\tilde{\rho}}^{\perp} (\mathbf{w}) \\
      \ge {\tilde{\rho}}(\mathbf{w}) \gamma - \Vert \mathbf{x}_i \Vert  \cdot {\tilde{\rho}}^{\perp} (\mathbf{w}) \\
    \end{array} \right.
  \end{aligned},
  \end{equation*}
  we have 
  \begin{equation*}
  \begin{aligned}
    \left| y_i \langle \mathbf{w}, \mathbf{x}_i \rangle  - {\tilde{\rho}}(\mathbf{w}) \gamma\right|
    \le 
    \Vert \mathbf{x}_i \Vert  \cdot {\tilde{\rho}}^{\perp} (\mathbf{w})
    \le 
    \sqrt{\lambda_{\max}}  \cdot {\tilde{\rho}}^{\perp} (\mathbf{w})
  \end{aligned}
  \end{equation*}
  Recall that 
  $\tilde{\rho}^{\perp} \left(\mathbf{w}\right) \cdot \Vert \hat{\mathbf{w}} \Vert 
  = \rho^{\perp} \left(\mathbf{w}\right) \cdot \Vert \mathbf{w} \Vert$ and 
  $\tilde{\rho} \left(\mathbf{w}\right) \cdot \Vert \hat{\mathbf{w}} \Vert 
  = \rho \left(\mathbf{w}\right) \cdot \Vert \mathbf{w} \Vert$, we have the second result.
\end{proof}

\newtheorem*{restate_risk_upperbound_logit_reg}{Lemma \ref{risk_upperbound_logit_reg} (Upper Bound of Logistic Loss)}
\definecolor{shadecolor}{rgb}{0.92,0.92,0.92}
\begin{shaded}
\begin{restate_risk_upperbound_logit_reg}
  Let $\ell$ be $\ell_{log}$ and $\hat{\mathbf{w}}$ be the SVM solution as presented in Definition \ref{margin}. 
  Suppose Assumptions \ref{Overparameterization} and \ref{Logistic_Setup} hold.
  Consider the gradient descent (\ref{gradient_updata}), for any $t \ge 0$, 
  if $\rho_t > 0$ and $\alpha_t > 0$, it holds that 
  \begin{equation*}
  \begin{aligned}
    \mathcal{R}_t \le 
    \ell \left(\alpha_t\right) +
    \alpha_t \left| \ell^{\prime} \left(
        \left[ 1 -  C_0 \gamma \cdot \rho_t^{\perp} \right] \alpha_t
      \right)
    \right|
    \cdot C_0 \gamma \cdot \rho_t^{\perp},
  \end{aligned}
  \end{equation*}
  where $C_0$ (defined in \ref{logit_bound}) is 
  \begin{equation*}
  \begin{aligned}
    C_0 := \frac{\sqrt{\lambda_{\max}}}{\sqrt{\lambda_{\min} }} + 2 \sqrt{2} \frac{\lambda_{\max} }{\lambda_{\min}}
  \end{aligned}.
  \end{equation*}
\end{restate_risk_upperbound_logit_reg}
\end{shaded}

\begin{proof}
  In the proof, we ignore the subscript of $\mathbf{w}_t$ and $\alpha_t$. 
  We first bound the difference between
  $\left|  \ell \left(
      \alpha \cdot
      \langle \mathbf{w}, \mathbf{x}_i y_i \rangle / \left\Vert \mathbf{w} \right\Vert_{\mathbf{\Sigma}}
    \right)
  \right|$ 
  and $\left|\ell \left(\alpha\right)\right|$. For any $i \in [n]$, we have
  \begin{equation*}
  \begin{aligned}
    \Big|
      \left| 
      \ell
      \left(
        \alpha \cdot
        \langle \mathbf{w}, \mathbf{x}_i y_i \rangle / \left\Vert \mathbf{w} \right\Vert_{\mathbf{\Sigma}}
      \right)
      \right|
      - 
      \left|\ell \left(\alpha\right)\right|
    \Big|
    & \le 
    \max \left(
      \left|  \ell^{\prime} \left( \alpha \right) \right|,
      \left| \ell^{\prime} \left(
          \alpha \frac{\langle \mathbf{w}, \mathbf{x}_i y_i \rangle}
          {\left\Vert \mathbf{w} \right\Vert_{\mathbf{\Sigma}}} 
        \right)
      \right|
    \right)
    \\ & \qquad \cdot \alpha
    \left|
      \frac{\langle \mathbf{w}, \mathbf{x}_i y_i \rangle}{
        \left\Vert \mathbf{w} \right\Vert_{\mathbf{\Sigma}}
      }
      - 1
    \right|
  \end{aligned}
  \end{equation*}
  Since $\alpha > 0$, we have 
  \begin{equation}\label{eq:logit_margin_deviation_bound}
  \begin{aligned}
    \left| 
      \alpha
      \frac{\langle \mathbf{w}, \mathbf{x}_i y_i \rangle}{
        \left\Vert \mathbf{w} \right\Vert_{\mathbf{\Sigma}}
      }
      - 
      \alpha
    \right|
    & =
    \alpha
    \left| 
      \langle \mathbf{w}, \mathbf{x}_i y_i \rangle 
      - 
      \left\Vert \mathbf{w} \right\Vert_{\mathbf{\Sigma}}
    \right|
    / \left\Vert \mathbf{w} \right\Vert_{\mathbf{\Sigma}}
    \\ & \le
    \alpha\frac{
      \left| 
      \langle \mathbf{w}, \mathbf{x}_i y_i \rangle 
      - 
      \gamma^2 \Vert \mathbf{w} \Vert \cdot \rho \left(\mathbf{w}\right)
    \right|
    }{\left\Vert \mathbf{w} \right\Vert_{\mathbf{\Sigma}}}
    +
    \alpha\frac{
    \left| 
      \left\Vert \mathbf{w} \right\Vert_{\mathbf{\Sigma}}
      -
      \gamma^2 \Vert \mathbf{w} \Vert \cdot \rho \left(\mathbf{w}\right)
    \right|
    }{
      \left\Vert \mathbf{w} \right\Vert_{\mathbf{\Sigma}}
    }
    \\ & \le 
    \alpha\sqrt{\lambda_{\max}} 
    \frac{
      \Vert \mathbf{w} \Vert
    }{\left\Vert \mathbf{w} \right\Vert_{\mathbf{\Sigma}}} \gamma \cdot \rho^{\perp} \left(\mathbf{w}\right)
    +
    \alpha \cdot 2 \sqrt{2} \lambda_{\max}
    \frac{\Vert \mathbf{w} \Vert^2}{
      \left\Vert \mathbf{w} \right\Vert_{\mathbf{\Sigma}}^2
    }
    \gamma 
    \cdot \rho^{\perp} \left(\mathbf{w}\right)
    \\ & \le 
    \alpha \frac{\sqrt{\lambda_{\max}} }{\sqrt{\lambda_{\min}} }
    \frac{
      \Vert \mathbf{w} \Vert
    }{\left\Vert \mathbf{w} \right\Vert_{\mathbf{\Sigma}}} \gamma \cdot \rho^{\perp} \left(\mathbf{w}\right)
    +
    \alpha \cdot 2 \sqrt{2} \gamma
    \frac{\lambda_{\max}}{\lambda_{\min}}
    \rho^{\perp} \left(\mathbf{w}\right)
    \\ & =
    \alpha
    C_0
    \gamma \cdot \rho^{\perp} \left(\mathbf{w}\right)
    ,
  \end{aligned}
  \end{equation}
  where
  \[
    C_0 :=
    \frac{\sqrt{\lambda_{\max}}}{\sqrt{\lambda_{\min}}}
    +
    2 \sqrt{2} \frac{\lambda_{\max}}{\lambda_{\min}}.
  \]
  We have 
  \begin{equation*}
  \begin{aligned}
    \left| 
      \ell^{\prime} \left(
        \alpha \frac{\langle \mathbf{w}, \mathbf{x}_i y_i \rangle}{
          \left\Vert \mathbf{w} \right\Vert_{\mathbf{\Sigma}}
        } 
      \right)
    \right|
    & = 
    \left| 
      \ell^{\prime} \left(
        \alpha \frac{\langle \mathbf{w}, \mathbf{x}_i y_i \rangle}{
          \left\Vert \mathbf{w} \right\Vert_{\mathbf{\Sigma}}
        } 
        - \alpha + \alpha
      \right)
    \right|
    \\ & \le 
    \left| 
      \ell^{\prime} \left(
        \alpha - 
        \alpha
        C_0
        \gamma \cdot \rho^{\perp} \left(\mathbf{w}\right)
      \right)
    \right|
    \\ & = 
    \left| 
      \ell^{\prime} \left(
        \left[
          1
          - 
        C_0 
        \gamma \cdot \rho^{\perp} \left(\mathbf{w}\right)
        \right]
        \alpha
      \right)
    \right|
  \end{aligned}
  \end{equation*}
  Therefore, we have 
  \begin{equation}\label{eq:logit_derivative_shift_bound}
  \begin{aligned}
    \max \left(
      \left| 
      \ell^{\prime} \left(
        \alpha
      \right)
    \right|,
    \left| 
      \ell^{\prime} \left(
        \alpha \frac{\langle \mathbf{w}, \mathbf{x}_i y_i \rangle}{
          \left\Vert \mathbf{w} \right\Vert_{\mathbf{\Sigma}}
        } 
      \right)
    \right|
    \right)
    & \le 
    \max \left(
      \left| 
      \ell^{\prime} \left(
        \alpha
      \right)
    \right|,
    \left| 
      \ell^{\prime} \left( \left[1- C_0 \gamma \cdot \rho^{\perp} \left(\mathbf{w}\right) \right] \alpha \right)
    \right|
    \right)
    \\ & \le 
    \left| 
    \ell^{\prime} \left(
      \left[
        1
        - 
      C_0
      \gamma \cdot \rho^{\perp} \left(\mathbf{w}\right)
      \right]
      \alpha
    \right)
  \right|
  \end{aligned}
  \end{equation}
  Now we are ready to bound $\mathcal{R} \left(\mathbf{w}, \alpha\right)$:
  \begin{equation*}
  \begin{aligned}
    & \left|
      \mathcal{R} \left(\mathbf{w}, \alpha\right) - \ell \left(\alpha \right)
    \right|
    \\ = &
    \left|
      \frac{1}{n}
      \sum_{i=1}^{n}
      \ell \left(
        \alpha \cdot \frac{ \langle \mathbf{w}, \mathbf{x}_i y_i\rangle}{\Vert \mathbf{w} \Vert_{\mathbf{\Sigma}}}
      \right)
      - \ell \left(\alpha \right)
    \right|
    \\ \le &
    \frac{1}{n}
    \sum_{i=1}^{n}
    \left|
      \ell \left(
        \alpha \cdot \frac{ \langle \mathbf{w}, \mathbf{x}_i y_i\rangle}{\Vert \mathbf{w} \Vert_{\mathbf{\Sigma}}}
      \right)
      - \ell \left(\alpha \right)
    \right|
    \\ \le &
    \max\left(
      \left| \ell^{\prime} \left( \alpha \cdot \frac{ \langle \mathbf{w}, \mathbf{x}_1 y_1\rangle}{\Vert \mathbf{w} \Vert_{\mathbf{\Sigma}}} \right) \right|, 
      \cdots,
      \left| \ell^{\prime} \left( \alpha \cdot \frac{ \langle \mathbf{w}, \mathbf{x}_n y_n\rangle}{\Vert \mathbf{w} \Vert_{\mathbf{\Sigma}}} \right) \right|, 
      \left| \ell^{\prime} \left( \alpha \right) \right|
    \right)
    \frac{1}{n}
    \sum_{i=1}^{n}
    \left|
      \alpha \cdot \frac{ \langle \mathbf{w}, \mathbf{x}_i y_i\rangle}{\Vert \mathbf{w} \Vert_{\mathbf{\Sigma}}}
      - \alpha
    \right|
    \\ \le &
    \frac{1}{n}
    \left| 
      \ell^{\prime} \left(
        \left[
          1
          - 
        C_0
        \gamma \cdot \rho^{\perp} \left(\mathbf{w}\right)
        \right]
        \alpha
      \right)
    \right|
    \sum_{i=1}^{n}
    \left|
      \alpha \cdot \frac{ \langle \mathbf{w}, \mathbf{x}_i y_i\rangle}{\Vert \mathbf{w} \Vert_{\mathbf{\Sigma}}}
      - \alpha
    \right|
    \\ \le &
    \left| 
      \ell^{\prime} \left(
        \left[
          1
          - 
        C_0
        \gamma \cdot \rho^{\perp} \left(\mathbf{w}\right)
        \right]
        \alpha
      \right)
    \right|
    \alpha
    C_0
    \gamma \cdot \rho^{\perp} \left(\mathbf{w}\right)
  \end{aligned}
  \end{equation*}
\end{proof}

\noindent\textbf{Optional appendix-only loss bridge for logistic regression.}
To complement the main directional theorem, we record here the additional condition used only to convert directional deviation into a lower bound on the logistic loss. This material is intentionally kept out of the main text because it is not needed for Theorem \ref{logis_spike}.

\begin{assumption}[Non-degenerate data]\label{Non_degenerate_data}
  Let $\hat{\mathbf{w}}$ be the SVM solution as presented in Definition \ref{margin}, and let $\mathcal{S}$ denote the support-vector set. Assume that there exist coefficients $\beta_i > 0$ for all $i \in \mathcal{S}$ such that
  \[
    \hat{\mathbf{w}} = \sum_{i \in \mathcal{S}} \beta_i y_i \mathbf{x}_i.
  \]
\end{assumption}

\begin{definition}[Margin Offset]\label{Margin_Offset}
  Suppose Assumptions \ref{Overparameterization}, \ref{Logistic_Setup}, and \ref{Non_degenerate_data} hold. Define the margin offset $b > 0$ by
  \begin{equation*}
  \begin{aligned}
    -b := \max_{\mathbf{w} \in \text{span}^{\perp}\{\hat{\mathbf{w}}\} \cap \text{span} \{\mathbf{x}_1, \cdots, \mathbf{x}_n\}}
    \min_{i \in [n]} \frac{y_i \langle \mathbf{x}_i, \mathbf{w} \rangle}{\Vert \mathbf{w} \Vert}.
  \end{aligned}
  \end{equation*}
\end{definition}
Under Assumption \ref{Non_degenerate_data}, every nonzero direction in $\text{span}^{\perp}\{\hat{\mathbf{w}}\} \cap \text{span}\{\mathbf{x}_1,\ldots,\mathbf{x}_n\}$ has a strictly negative signed margin for at least one sample, so the maximized minimum above is negative and the quantity $b$ is well-defined and strictly positive.

We now prove Lemma \ref{risk_lowerbound_logit_reg}. The proof uses Assumption \ref{Non_degenerate_data} and Definition \ref{Margin_Offset}; see Section 3.1 of \cite{eos_logistic_regression} for related techniques.
\begin{lemma}[Lower Bound of Logistic Loss]\label{risk_lowerbound_logit_reg}
    Let $\ell$ be $\ell_{log}$ and $\hat{\mathbf{w}}$ be the SVM solution as presented in Definition \ref{margin}. 
  Suppose Assumptions \ref{Overparameterization}, \ref{Logistic_Setup}, and \ref{Non_degenerate_data} hold.
  Consider the gradient descent (\ref{gradient_updata}), 
  for all $t \ge 0$, if $\alpha_t > 0$, it holds that 
  \begin{equation*}
  \begin{aligned}
    \mathcal{R}_t \ge \frac{1}{n} \ell \left( 
      \frac{\alpha_t}{\sqrt{\lambda_{\min}}}
      \left(
        \rho_t \gamma^2  
      - \rho_t^{\perp} b \gamma
      \right)
    \right),
  \end{aligned}
  \end{equation*}
  where $b>0$ is the margin offset (see Definition \ref{Margin_Offset} in Appendix).
\end{lemma}

\begin{proof}
  Consider $\mathcal{R}_t$, we have 
  \begin{equation*}
  \begin{aligned}
    \mathcal{R}_t = \frac{1}{n} \sum_{i=1}^{n} \ell \left( \alpha_t \frac{\langle \mathbf{w}_t, \mathbf{x}_i y_i \rangle}{\Vert \mathbf{w}_t \Vert_{\mathbf{\Sigma}}} \right)
  \end{aligned}
  \end{equation*}
  By Definition \ref{Margin_Offset}, there exists a $j \in [n]$ such that 
  \begin{equation*}
  \begin{aligned}
    \left\langle y_j \mathbf{x}_j, 
    \left( \mathbf{I} - \frac{ \hat{\mathbf{w}} \hat{\mathbf{w}}^T }{\Vert \hat{\mathbf{w}} \Vert^2} \right)
    \mathbf{w}_t \right\rangle 
    & \leq -b \left\Vert \left( \mathbf{I} - \frac{ \hat{\mathbf{w}} \hat{\mathbf{w}}^T }{\Vert \hat{\mathbf{w}} \Vert^2} \right)
    \mathbf{w}_t \right\Vert 
    \\ & =
    -b 
    \frac{\left\Vert \left( \mathbf{I} - \frac{ \hat{\mathbf{w}} \hat{\mathbf{w}}^T }{\Vert \hat{\mathbf{w}} \Vert^2} \right)
    \mathbf{w}_t \right\Vert }{
      \left\Vert \left( \mathbf{I} - \frac{ {\mathbf{w}_t} {\mathbf{w}_t}^T }{\Vert {\mathbf{w}_t} \Vert^2} \right)
    \hat{\mathbf{w}} \right\Vert 
    }
    \left\Vert \left( \mathbf{I} - \frac{ {\mathbf{w}_t} {\mathbf{w}_t}^T }{\Vert {\mathbf{w}_t} \Vert^2} \right)
    \hat{\mathbf{w}} \right\Vert 
    \\ & =
    -b 
    \frac{\Vert \mathbf{w}_t \Vert}{\Vert \hat{\mathbf{w}} \Vert}
    \rho_t^{\perp}
    =
    - \rho_t^{\perp} b \gamma \Vert \mathbf{w}_t \Vert
  \end{aligned}
  \end{equation*}
  Then, we have 
  \begin{equation*}
  \begin{aligned}
    \langle \mathbf{w}_t, \mathbf{x}_j y_j \rangle
    & = 
    \left\langle \left( \frac{ \hat{\mathbf{w}} \hat{\mathbf{w}}^T }{\Vert \hat{\mathbf{w}} \Vert^2} \right)\mathbf{w}_t, \mathbf{x}_j y_j  \right\rangle + 
    \left\langle \left( \mathbf{I} - \frac{ \hat{\mathbf{w}} \hat{\mathbf{w}}^T }{\Vert \hat{\mathbf{w}} \Vert^2} \right)\mathbf{w}_t, \mathbf{x}_j y_j \right \rangle
    \\ & \le 
    \gamma^2 \mathbf{w}_t^T \hat{\mathbf{w}}
    - \rho_t^{\perp} b \gamma \Vert \mathbf{w}_t \Vert
    \\ & =
    \gamma^2 \Vert \mathbf{w}_t \Vert \rho_t
    - \rho_t^{\perp} b \gamma \Vert \mathbf{w}_t \Vert
  \end{aligned}
  \end{equation*}
  Then
  \begin{equation*}
  \begin{aligned}
    \mathcal{R}_t \ge 
    \frac{1}{n}
    \ell\!\left(
      \alpha_t
      \frac{\langle \mathbf{w}_t, \mathbf{x}_j y_j \rangle}{
        \Vert \mathbf{w}_t \Vert_{\mathbf{\Sigma}}
      }
    \right)
    \ge 
    \frac{1}{n} \ell \left(
      \alpha_t
      \frac{\Vert \mathbf{w}_t \Vert}{\Vert \mathbf{w}_t \Vert_{\mathbf{\Sigma}}}
      \left(
        \rho_t \gamma^2  
        - \rho_t^{\perp} b \gamma
      \right)
    \right)
    \ge 
    \frac{1}{n}\ell\!\Bigl(
      \frac{\alpha_t}{\sqrt{\lambda_{\min}}}
      \bigl(
        \rho_t \gamma^2  
        - \rho_t^{\perp} b \gamma
      \bigr)
    \Bigr)
  \end{aligned}
  \end{equation*}
\end{proof}

\definecolor{shadecolor}{rgb}{0.92,0.92,0.92}
\newtheorem*{restate_inner_product_bound}{Lemma \ref{inner_product_bound}}
\begin{shaded}
\begin{restate_inner_product_bound}
  Let $\ell$ be $\ell_{log}$ and $\hat{\mathbf{w}}$ be the SVM solution as defined in Definition \ref{margin}. 
  Suppose Assumptions \ref{Overparameterization}, \ref{Logistic_Setup}, and \ref{Active_margin_data} hold.
  Consider the gradient descent (\ref{gradient_updata}), for any $t \ge 0$, it holds that
  $- \alpha_t \cdot \left\langle \hat{\mathbf{w}}, \nabla_{\mathbf{w}} \mathcal{R} \left(\mathbf{w}, \alpha_t\right) \right\rangle \ge 0$; 
  and if $\alpha_t > 0$, we have 
  \begin{equation*}
  \begin{aligned}
    - \left\langle \hat{\mathbf{w}}, \nabla_{\mathbf{w}} \mathcal{R}_t \right\rangle 
    \ge 
    \frac{\lambda_{\min}}{8}
    \frac{\alpha_t  e^{-\alpha_t }}{\Vert \mathbf{w}_t  \Vert_{\mathbf{\Sigma}}}
    \left( \rho_t^{\perp} \right)^2.
  \end{aligned}
  \end{equation*}
\end{restate_inner_product_bound}
\end{shaded}
\begin{proof}
  In this proof, we ignore the subscript of $\mathbf{w}_t$ and $\alpha_t$.  
  Recall that 
  \begin{equation*}
  \begin{aligned}
    \nabla_{\mathbf{w}} \mathcal{R}
    \left(\mathbf{w}, \alpha\right)
     = \frac{\alpha}{n \Vert \mathbf{w} \Vert_{\mathbf{\Sigma}}}
      \left(
        \mathbf{I} - \frac{\mathbf{\Sigma} \mathbf{w} \mathbf{w}^{T}}{\Vert \mathbf{w} \Vert_{\mathbf{\Sigma}}^{2}}
      \right)
      \tilde{\mathbf{X}} 
      \boldsymbol{\ell^\prime}\left(
        \frac{
        \alpha \tilde{\mathbf{X}}^T \mathbf{w}
        }{\Vert \mathbf{w} \Vert_\mathbf{\Sigma}}
      \right)
  \end{aligned}
  \end{equation*}
  We have 
  \begin{equation*}
  \begin{aligned}
    - \left\langle \hat{\mathbf{w}}, \nabla_{\mathbf{w}} \mathcal{R} \right\rangle 
    & = 
    - \frac{\alpha}{n \Vert \mathbf{w} \Vert_{\mathbf{\Sigma}}}
    \hat{\mathbf{w}}^T
    \left(
      \mathbf{I} - \frac{\mathbf{\Sigma} \mathbf{w} \mathbf{w}^{T}}{\Vert \mathbf{w} \Vert_{\mathbf{\Sigma}}^{2}}
    \right)
    \tilde{\mathbf{X}} 
    \boldsymbol{\ell^\prime}
    \\ & = 
    - \frac{\alpha}{n \Vert \mathbf{w} \Vert_{\mathbf{\Sigma}}}
    \hat{\mathbf{w}}^T
    \tilde{\mathbf{X}} 
    \left(
      \mathbf{I} - \frac{ \tilde{\mathbf{X}}^{T} \mathbf{w} \left( \tilde{\mathbf{X}}^{T} \mathbf{w} \right)^T}{\Vert \tilde{\mathbf{X}}^{T} \mathbf{w} \Vert^{2}}
    \right)
    \boldsymbol{\ell^\prime}
    \\ & = 
    - \frac{\alpha}{n  \Vert \mathbf{w} \Vert_{\mathbf{\Sigma}}}
    \boldsymbol{1}^{T}
    \left(
      \mathbf{I} - \frac{ \mathbf{m} \mathbf{m}^T}{\Vert \mathbf{m} \Vert^{2}}
    \right)
    \boldsymbol{\ell^\prime},
  \end{aligned}
  \end{equation*}
  where we denote $\mathbf{m} := \tilde{\mathbf{X}}^{T} \mathbf{w}$.
  Note that $\Vert \mathbf{m} \Vert^{2} = n \Vert \mathbf{w} \Vert_{\mathbf{\Sigma}}^2 $. We have 
  \begin{equation}\label{eq:inner_product_trace_form}
  \begin{aligned}
    - \left\langle \hat{\mathbf{w}}, \nabla_{\mathbf{w}} \mathcal{R} \right\rangle 
    & = 
    - \frac{\alpha}{ n^2 \Vert \mathbf{w} \Vert_{\mathbf{\Sigma}}^3}
    \boldsymbol{1}^{T}
    \left(
      \mathbf{m}^T \mathbf{m} \mathbf{I} - \mathbf{m} \mathbf{m}^T
    \right)
    \boldsymbol{\ell^\prime}
    \\ & = 
    - \frac{\alpha}{2  n^2 \Vert \mathbf{w} \Vert_{\mathbf{\Sigma}}^3}
    \text{Tr}
    \left(
      \boldsymbol{\ell^{\prime}} \mathbf{m}^T
      -
      \mathbf{m} {\boldsymbol{\ell^{\prime}}}^{T}
    \right)^T  
    \left(
      \boldsymbol{1} \mathbf{m}^T -  \mathbf{m} \boldsymbol{1}^T 
    \right)
    \\ & = 
    \frac{\alpha}{2  n^2 \Vert \mathbf{w} \Vert_{\mathbf{\Sigma}}^3}
    \sum_{i=1}^{n} \sum_{j=1}^{n}
    \left(
      |[\boldsymbol{\ell^{\prime}}]_i| \mathbf{m}_j
      -
      \mathbf{m}_i |[\boldsymbol{\ell^{\prime}}]_j|
    \right)
    \left(
       \mathbf{m}_j -  \mathbf{m}_i
    \right)
    \\ & = 
    \frac{\alpha }{2 n^2 \Vert \mathbf{w} \Vert_{\mathbf{\Sigma}}^3}
    \sum_{i=1}^{n} \sum_{j=1}^{n}
    |[\boldsymbol{\ell^{\prime}}]_i|
    |[\boldsymbol{\ell^{\prime}}]_j|
    \left(
      \mathbf{m}_j -  \mathbf{m}_i 
    \right)^2
    \frac{
      |[\boldsymbol{\ell^{\prime}}]_j|^{-1} \mathbf{m}_j
      -
      \mathbf{m}_i |[\boldsymbol{\ell^{\prime}}]_i|^{-1}
    }{\mathbf{m}_j -  \mathbf{m}_i }
  \end{aligned}
  \end{equation}
  For the last term of above equation, we have 
  \begin{equation*}
  \begin{aligned}
    \frac{
      |[\boldsymbol{\ell^{\prime}}]_j|^{-1} \mathbf{m}_j
      -
      \mathbf{m}_i |[\boldsymbol{\ell^{\prime}}]_i|^{-1}
    }{\mathbf{m}_j -  \mathbf{m}_i }
    & = 
    \frac{
      \left(
        1 + \exp( \alpha \mathbf{m}_j / \Vert \mathbf{w} \Vert_{\mathbf{\Sigma}} )
      \right)
       \mathbf{m}_j
      -
      \left(
        1 + \exp( \alpha \mathbf{m}_i / \Vert \mathbf{w} \Vert_{\mathbf{\Sigma}} )
      \right)
      \mathbf{m}_i 
    }{\mathbf{m}_j -  \mathbf{m}_i }
    \\ & = 
    1 + 
    \frac{
      \exp( \alpha \mathbf{m}_j / \Vert \mathbf{w} \Vert_{\mathbf{\Sigma}} )
      \mathbf{m}_j
      -
      \exp( \alpha \mathbf{m}_i / \Vert \mathbf{w} \Vert_{\mathbf{\Sigma}} )
      \mathbf{m}_i 
    }{\mathbf{m}_j -  \mathbf{m}_i }.
  \end{aligned}
  \end{equation*}
  Multiply both the numerator and denominator by $\alpha / \Vert \mathbf{w} \Vert_{\mathbf{\Sigma}}$.
  Set
  $a := \mathbf{m}_i \alpha / \Vert \mathbf{w} \Vert_{\mathbf{\Sigma}}$
  and
  $b := \mathbf{m}_j \alpha / \Vert \mathbf{w} \Vert_{\mathbf{\Sigma}}$.
  Then
  \begin{equation}\label{eq:logit_pairwise_secant_lower}
  \begin{aligned}
    \frac{
      |[\boldsymbol{\ell^{\prime}}]_j|^{-1} \mathbf{m}_j
      -
      \mathbf{m}_i |[\boldsymbol{\ell^{\prime}}]_i|^{-1}
    }{\mathbf{m}_j -  \mathbf{m}_i }
    \\ & =
    1
    +
    \frac{b e^{b}-a e^{a}}{b - a} 
    \\ & \ge 
    1 + \exp \left(\max(a, b)\right) 
    \\ & =
    \max\left(
      | [\boldsymbol{\ell^{\prime}}]_{i} |^{-1}, 
      | [\boldsymbol{\ell^{\prime}}]_{j} |^{-1}
    \right).
  \end{aligned}
  \end{equation}
  Plugging (\ref{eq:logit_pairwise_secant_lower}) into (\ref{eq:inner_product_trace_form}) gives
  \begin{equation*}
  \begin{aligned}
    - \left\langle \hat{\mathbf{w}}, \nabla_{\mathbf{w}} \mathcal{R} \right\rangle 
    &= 
    \frac{\alpha }{2 n^2 \Vert \mathbf{w} \Vert_{\mathbf{\Sigma}}^3}
    \sum_{i=1}^{n} \sum_{j=1}^{n}
    \left(
      \mathbf{m}_j -  \mathbf{m}_i 
    \right)^2
    \max\left(
      | [\boldsymbol{\ell^{\prime}}]_{i} |, 
      | [\boldsymbol{\ell^{\prime}}]_{j} |
    \right)
  \end{aligned}
  \end{equation*}
  Now we proved that $- \left\langle \hat{\mathbf{w}}, \nabla_{\mathbf{w}} \mathcal{R} \right\rangle$ and 
  $\alpha$ always has the same sign. 
  In the left part of this proof, we only focus on the case that $\alpha > 0$.
  Further, by Lemma C.1 of \citep{implicit_bias_BN}, we have 
  \begin{equation}\label{eq:inner_product_pairwise_lower}
  \begin{aligned}
    - \left\langle \hat{\mathbf{w}}, \nabla_{\mathbf{w}} \mathcal{R} \right\rangle 
    \ge 
    \frac{\alpha}{8 n^2 \Vert \mathbf{w} \Vert_{\mathbf{\Sigma}}^3}
    \frac{\sum_{i=1}^{n} | [\boldsymbol{\ell^{\prime}}]_{i} |}{n}
    \sum_{i=1}^{n} \sum_{j=1}^{n}
    \left(
      \mathbf{m}_j -  \mathbf{m}_i 
    \right)^2
    \ge 0
  \end{aligned}
  \end{equation}
  Then consider ${\sum_{i=1}^{n} | [\boldsymbol{\ell^{\prime}}]_{i} |}/{n}$.
  \begin{equation}\label{eq:avg_logit_derivative_lower}
  \begin{aligned}
    \frac{\sum_{i=1}^{n} | [\boldsymbol{\ell^{\prime}}]_{i} |}{n}
    & = 
    \frac{1}{n} \sum_{i=1}^{n} \left[ 1 + \exp \left( \frac{\alpha \cdot \langle \mathbf{w}, \mathbf{x}_i y_i \rangle}
      {\Vert \mathbf{w} \Vert_{\mathbf{\Sigma}}} 
    \right)
    \right]^{-1}
    \\ & \geq 
    \frac{1}{n} \sum_{i=1}^{n} \left[ 1 + \exp \left( \frac{\alpha \cdot |\langle \mathbf{w}, \mathbf{x}_i y_i \rangle|}
    {\Vert \mathbf{w} \Vert_{\mathbf{\Sigma}}} \right) \right]^{-1}
    \\ & \geq \left[ 1 + \exp \left( \alpha \cdot \frac{1}{n} \sum_{i=1}^{n} 
    \frac{|\langle \mathbf{w}, \mathbf{x}_i y_i \rangle|}{\Vert \mathbf{w} \Vert_{\mathbf{\Sigma}}} \right) \right]^{-1}
    \\ & \geq [1 + \exp(\alpha)]^{-1}
    \\ & \geq \exp(- \alpha) / 2.
  \end{aligned}
  \end{equation}
  Then we relate $\sum_{i=1}^{n} \sum_{j=1}^{n}
  \left(
    \mathbf{m}_j -  \mathbf{m}_i 
  \right)^2$ with $\rho_t^{\perp}$.
  \begin{equation*}
  \begin{aligned}
    \sum_{i=1}^{n} \sum_{j=1}^{n}
    \left(
      \mathbf{m}_j -  \mathbf{m}_i 
    \right)^2
    & = 
    \text{Tr} \left( \mathbf{m} \mathbf{1}^{T} - \mathbf{1} \mathbf{m}^T \right)^T \left(
      \mathbf{m} \mathbf{1}^{T} - \mathbf{1} \mathbf{m}^T
    \right)
    \\ & = 
    2 \left(
      \Vert \mathbf{1} \Vert^2 \Vert \mathbf{m} \Vert^2 - \left( \mathbf{1}^T \mathbf{m} \right)^2
    \right)
    \\ & = 
    2 \Vert \mathbf{m} \Vert^2 \mathbf{1}^{T} \left(
      \mathbf{I} - \mathbf{m} \mathbf{m}^{T} / \Vert \mathbf{m} \Vert^{2}
    \right)
    \mathbf{1}
  \end{aligned}
  \end{equation*}
  Note that $\mathbf{I} - \mathbf{m} \mathbf{m}^{T} / \Vert \mathbf{m} \Vert^{2}$ is a projection matrix, we have 
  \begin{equation*}
  \begin{aligned}
    \sum_{i=1}^{n} \sum_{j=1}^{n}
    \left(
      \mathbf{m}_j -  \mathbf{m}_i 
    \right)^2
    & = 
    2 \Vert \mathbf{m} \Vert^2 
    \left\Vert
      \left(
        \mathbf{I} - \frac{\mathbf{m} \mathbf{m}^{T}}{ \Vert \mathbf{m} \Vert^{2}}
      \right)
      \mathbf{1}
    \right\Vert^2
  \end{aligned}
  \end{equation*}
  Recall that $\tilde{\mathbf{X}}^T \hat{\mathbf{w}} = \mathbf{1}$ by Assumption \ref{Active_margin_data}, and $\mathbf{m} = \tilde{\mathbf{X}}^{T} \mathbf{w}$. Therefore,
  \begin{equation}\label{eq:pairwise_margin_sum_projection}
  \begin{aligned}
    \sum_{i=1}^{n} \sum_{j=1}^{n}
    \left(
      \mathbf{m}_j -  \mathbf{m}_i 
    \right)^2
    & = 
    2 n \Vert \mathbf{w} \Vert_{\mathbf{\Sigma}}^2 
    \left\Vert
      \left(
        \mathbf{I} - \frac{\tilde{\mathbf{X}}^{T} \mathbf{w} \left(\tilde{\mathbf{X}}^{T} \mathbf{w}\right)^{T}}{ \Vert \tilde{\mathbf{X}}^{T} \mathbf{w} \Vert^{2}}
      \right)
      \tilde{\mathbf{X}}^T \hat{\mathbf{w}}
    \right\Vert^2
    \\ & = 
    2 n \Vert \mathbf{w} \Vert_{\mathbf{\Sigma}}^2 
    \left\Vert
    \tilde{\mathbf{X}}^T
      \left(
        \hat{\mathbf{w}} - \frac{\mathbf{w}^{T} \tilde{\mathbf{X}} \tilde{\mathbf{X}}^{T} \hat{\mathbf{w}} }{ \Vert \tilde{\mathbf{X}}^{T} \mathbf{w} \Vert^{2}} \mathbf{w}
      \right)
    \right\Vert^2
    \\ & = 
    2 n \Vert \mathbf{w} \Vert_{\mathbf{\Sigma}}^2 
    \left\Vert
    \tilde{\mathbf{X}}^T
      \left(
        \hat{\mathbf{w}} - \frac{\langle \mathbf{w},  \hat{\mathbf{w}} \rangle_{\mathbf{\Sigma}} }{ \Vert \mathbf{w} \Vert_{\mathbf{\Sigma}}^{2}} \mathbf{w}
      \right)
    \right\Vert^2
    \\ & = 
    2 n^2 \Vert \mathbf{w} \Vert_{\mathbf{\Sigma}}^2 
    \left\Vert
      \hat{\mathbf{w}} - \frac{\langle \mathbf{w},  \hat{\mathbf{w}} \rangle_{\mathbf{\Sigma}} }{ \Vert \mathbf{w} \Vert_{\mathbf{\Sigma}}^{2}} \mathbf{w}
    \right\Vert_{\mathbf{\Sigma}}^2
  \end{aligned}
  \end{equation}
  We note that
  $\mathbf{w} \cdot \langle \mathbf{w}, \hat{\mathbf{w}}\rangle / \Vert \mathbf{w} \Vert^2$
  is the Euclidean projection of $\hat{\mathbf{w}}$ onto $\text{span}\{\mathbf{w}\}$.
  Therefore,
  we have 
  \begin{equation}\label{eq:orthogonal_projection_lower}
  \begin{aligned}
    \left\Vert
      \hat{\mathbf{w}}
      -
      \frac{\langle \mathbf{w},  \hat{\mathbf{w}} \rangle_{\mathbf{\Sigma}} }{ \Vert \mathbf{w} \Vert_{\mathbf{\Sigma}}^{2}} \mathbf{w}
    \right\Vert_{\mathbf{\Sigma}}^2
    \\ & \ge 
    \lambda_{\min}
    \left\Vert
      \hat{\mathbf{w}}
      -
      \frac{\langle \mathbf{w},  \hat{\mathbf{w}} \rangle_{\mathbf{\Sigma}} }{ \Vert \mathbf{w} \Vert_{\mathbf{\Sigma}}^{2}} \mathbf{w}
    \right\Vert^2
    \\ & \ge 
    \lambda_{\min}
    \left\Vert
      \hat{\mathbf{w}}
      -
      \frac{\langle \mathbf{w},  \hat{\mathbf{w}} \rangle }{ \Vert \mathbf{w} \Vert^{2}} \mathbf{w}
    \right\Vert^2
    \\ & =
    \lambda_{\min}
    \left( {\rho}^{\perp} \left(\mathbf{w}\right)\right)^2
  \end{aligned}
  \end{equation}
  Plugging (\ref{eq:avg_logit_derivative_lower}), (\ref{eq:pairwise_margin_sum_projection}), and (\ref{eq:orthogonal_projection_lower}) into (\ref{eq:inner_product_pairwise_lower}) gives
  \begin{equation*}
  \begin{aligned}
    - \left\langle \hat{\mathbf{w}}, \nabla_{\mathbf{w}} \mathcal{R} \right\rangle 
    \ge 
    \frac{ \lambda_{\min}}{8} 
    \frac{\alpha e^{-\alpha}}{ \Vert \mathbf{w} \Vert_{\mathbf{\Sigma}}}
    \big( {\rho}^{\perp} \left(\mathbf{w}\right)\big)^2
  \end{aligned}
  \end{equation*}
\end{proof}

\definecolor{shadecolor}{rgb}{0.92,0.92,0.92}
\newtheorem*{restate_Gradient_Norm_Bound}{Lemma \ref{Gradient_Norm_Bound}}
\begin{restate_Gradient_Norm_Bound}\label{restate_Gradient_Norm_Bound_ref}
  Let $\ell$ be $\ell_{log}$ and $\hat{\mathbf{w}}$ be the SVM solution as defined in Definition \ref{margin}. 
  Suppose Assumptions \ref{Overparameterization}, \ref{Logistic_Setup}, and \ref{Active_margin_data} hold. 
  Consider the gradient descent (\ref{gradient_updata}), for any $t \ge 0$, if $\alpha_t>0$, it holds that 
  \begin{equation*}
  \begin{aligned}
    \frac{\lambda_{\min}}{4}
    \frac{\alpha_t e^{-\alpha_t}}{\Vert \mathbf{w}_t \Vert_{\mathbf{\Sigma}}}
    \rho_t^{\perp}
    \\ & \le 
    \frac{\sqrt{\lambda_{\min}}}{4}
    \frac{\alpha_t e^{-\alpha_t}}{\Vert \mathbf{w}_t \Vert_{\mathbf{\Sigma}}}
    \rho_t^{\perp, \mathbf{\Sigma}}
    \\ & \le 
    \left\Vert 
      \nabla_{\mathbf{w}} \mathcal{R}_t
    \right\Vert
    \le 
    \frac{\alpha_t}{\Vert \mathbf{w}_t \Vert_{\mathbf{\Sigma}}}
    \max \Big(
      \sqrt{\lambda_{\max}},
      \\
      & \qquad \lambda_{\max}
      \left( \alpha_t + 1 \right)
      \rho_t^{\perp}
    \Big),
  \end{aligned}
  \end{equation*}
  where $\rho_t^{\perp, \mathbf{\Sigma}}$ is defined as 
  \begin{equation*}
  \begin{aligned}
    \rho_t^{\perp, \mathbf{\Sigma}} 
    := \rho^{\perp, \mathbf{\Sigma}} \left( \mathbf{w}_t \right)
    \\ & :=
    \left\Vert
      \hat{\mathbf{w}}
      -
      \rho_t
      \frac{\mathbf{w}_t}{\Vert \mathbf{w}_t \Vert_\mathbf{\Sigma}}
    \right\Vert_\mathbf{\Sigma}.
  \end{aligned} 
  \end{equation*}    
\end{restate_Gradient_Norm_Bound}
\begin{proof}
  \noindent\textbf{First Upper Bound}
  We have 
  \begin{equation*}
  \begin{aligned}
    \left\Vert 
      \nabla_{\mathbf{w}} \mathcal{R}
    \right\Vert
    & =
    \frac{\alpha}{n \Vert \mathbf{w} \Vert_{\mathbf{\Sigma}}}
    \left\Vert
      \left(
        \mathbf{I} - \frac{\mathbf{\Sigma} \mathbf{w} \mathbf{w}^{T}}{\Vert \mathbf{w} \Vert_{\mathbf{\Sigma}}^{2}}
      \right)
      \tilde{\mathbf{X}} 
      \boldsymbol{\ell^\prime}
    \right\Vert
    =
    \frac{\alpha}{n \Vert \mathbf{w} \Vert_{\mathbf{\Sigma}}}
    \left\Vert
      \tilde{\mathbf{X}}
      \left(
        \mathbf{I} - \frac{\mathbf{m} \mathbf{m}^{T}}{\Vert \mathbf{m} \Vert^{2}}
      \right)
      \boldsymbol{\ell^\prime}
    \right\Vert
    \\ & \le
    \frac{\alpha}{\sqrt{n} \Vert \mathbf{w} \Vert_{\mathbf{\Sigma}}}
    \sqrt{\lambda_{\max}}
    \left\Vert
    \boldsymbol{\ell^\prime}
    \right\Vert
    \\ & \le
    \frac{\alpha}{\Vert \mathbf{w} \Vert_{\mathbf{\Sigma}}}
    \sqrt{\lambda_{\max}}
    ,
  \end{aligned}
  \end{equation*}
  where the second inequality is since $|\ell^{\prime} \left(\cdot\right)| \le 1$ and 
  $\left\Vert
    \boldsymbol{\ell^\prime}
  \right\Vert \le \sqrt{n}$.

  \noindent\textbf{Second Upper Bound}
  We introduce the vector of absolute logistic derivatives:
  \begin{equation*}
  \begin{aligned}
    |\boldsymbol{\ell^\prime}|
    :=
    \left(
      \left|
        \ell^{\prime}\!\left(
          \alpha \cdot \langle \mathbf{w}, y_i \mathbf{x}_i \rangle / \Vert \mathbf{w} \Vert_{\mathbf{\Sigma}}
        \right)
      \right|
    \right)_{i=1}^{n}
  \end{aligned}
  \end{equation*}
  Next, we bound the difference between $|\boldsymbol{\ell^\prime}|$ and $|\hat{\boldsymbol{\ell^\prime}}|$, 
  where $|\hat{\boldsymbol{\ell^\prime}}|$ is 
  \begin{equation*}
  \begin{aligned}
    & |\hat{\boldsymbol{\ell^\prime}}| := 
    \left| \boldsymbol{\ell^\prime} \right| \left(
      \cos_{\mathbf{\Sigma}} \angle (\mathbf{w}, \hat{\mathbf{w}}) \cdot 
      \frac{\alpha \tilde{\mathbf{X}}^T \hat{\mathbf{w}}}{\Vert \hat{\mathbf{w}} \Vert_{\mathbf{\Sigma}}}
    \right)
    =
    \left|
      \ell^{\prime}
      \left(
        \cos_{\mathbf{\Sigma}} \angle (\mathbf{w}, \hat{\mathbf{w}}) \cdot \alpha
      \right)
    \right|
    \mathbf{1},
    \\ & 
    \ \text{ and }
    \cos_{\mathbf{\Sigma}} \angle (\mathbf{w}, \hat{\mathbf{w}}) := 
    \frac{
      \langle \mathbf{w}, \hat{\mathbf{w}} \rangle_{\mathbf{\Sigma}}
    }{
      \Vert \mathbf{w} \Vert_{\mathbf{\Sigma}} \Vert \hat{\mathbf{w}} \Vert_{\mathbf{\Sigma}}
    }.
  \end{aligned}.
  \end{equation*}
  Then, we have
  \begin{equation}\label{eq:gradient_norm_upper_decomposition}
  \begin{aligned}
    \left\Vert 
      \nabla_{\mathbf{w}} \mathcal{R}
    \right\Vert
    & =
    \frac{\alpha}{n \Vert \mathbf{w} \Vert_{\mathbf{\Sigma}}}
    \left\Vert
      \left(
        \mathbf{I} - \frac{\mathbf{\Sigma} \mathbf{w} \mathbf{w}^{T}}{\Vert \mathbf{w} \Vert_{\mathbf{\Sigma}}^{2}}
      \right)
      \tilde{\mathbf{X}} 
      \boldsymbol{\ell^\prime}
    \right\Vert
    \\ & =
    \frac{\alpha}{n \Vert \mathbf{w} \Vert_{\mathbf{\Sigma}}}
    \left\Vert
      \tilde{\mathbf{X}}
      \left(
        \mathbf{I} - \frac{\mathbf{m} \mathbf{m}^{T}}{\Vert \mathbf{m} \Vert^{2}}
      \right)
      \boldsymbol{\ell^\prime}
    \right\Vert
    \\ & \le 
    \sqrt{\lambda_{\max}}
    \frac{\alpha}{\sqrt{n} \Vert \mathbf{w} \Vert_{\mathbf{\Sigma}}}
    \left(
      \left\Vert
      \left(
        \mathbf{I} - \frac{\mathbf{m} \mathbf{m}^{T}}{\Vert \mathbf{m} \Vert^{2}}
      \right)
      \left(
        |\boldsymbol{\ell^\prime}| 
        - 
        |\hat{\boldsymbol{\ell^\prime}}|
      \right)
      \right\Vert
      +
      \right. 
      \\ & \ \ \ \ \ \ \ 
      \left.
      \left|
        \ell^{\prime}
        \left(
          \cos_{\mathbf{\Sigma}} \angle (\mathbf{w}, \hat{\mathbf{w}}) \cdot 
          \alpha
        \right)
      \right|
      \left\Vert
      \left(
        \mathbf{I} - \frac{\mathbf{m} \mathbf{m}^{T}}{\Vert \mathbf{m} \Vert^{2}}
      \right)
        \mathbf{1}
      \right\Vert
    \right),
  \end{aligned}
  \end{equation}
  where we denote $\mathbf{m} = \tilde{\mathbf{X}}^T \mathbf{w}$. 
  For the first term in above inequality, we have 
  \begin{equation*}
  \begin{aligned}
    & \left\Vert
    \left(
      \mathbf{I} - \frac{ \mathbf{m} \mathbf{m}^T}{\Vert \mathbf{m} \Vert^{2}}
    \right)
    \left(
      |\boldsymbol{\ell^\prime}| 
      - 
      |\hat{\boldsymbol{\ell^\prime}}|
    \right)
  \right\Vert
  \\ \le & 
  \left\Vert
    |\boldsymbol{\ell^\prime}| - |\hat{\boldsymbol{\ell^\prime}}|
  \right\Vert
  \\  = &
  |\ell^{\prime \prime}\left( z \right)|
  \left\Vert 
    \frac{
      \alpha \tilde{\mathbf{X}}^T \mathbf{w}
      }{\Vert \mathbf{w} \Vert_\mathbf{\Sigma}}
    -
    \cos_{\mathbf{\Sigma}} \angle (\mathbf{w}, \hat{\mathbf{w}}) \cdot 
    \frac{
    \alpha \tilde{\mathbf{X}}^T \hat{\mathbf{w}}
    }{\Vert \hat{\mathbf{w}} \Vert_\mathbf{\Sigma}}
  \right\Vert
  \\ \le  & 
  \max \left\{ 
    |[\boldsymbol{\ell^\prime}]_1|,
    \cdots, 
    |[\boldsymbol{\ell^\prime}]_n|, 
    |\ell^{\prime}(\cos_{\mathbf{\Sigma}} \angle (\mathbf{w}, \hat{\mathbf{w}}) \cdot \alpha)|
  \right\}
  \cdot 
  \left\Vert 
    \frac{
      \alpha \tilde{\mathbf{X}}^T \mathbf{w}
      }{\Vert \mathbf{w} \Vert_\mathbf{\Sigma}}
    -
    \cos_{\mathbf{\Sigma}} \angle (\mathbf{w}, \hat{\mathbf{w}}) \cdot 
    \frac{
    \alpha \tilde{\mathbf{X}}^T \hat{\mathbf{w}}
    }{\Vert \hat{\mathbf{w}} \Vert_\mathbf{\Sigma}}
  \right\Vert,
  \end{aligned}
  \end{equation*}
  where the equation is by Mean Value Theorem and $z$ is between
  $\alpha \cdot \langle \mathbf{w} , \mathbf{x}_i y_i\rangle$ and $\cos_{\mathbf{\Sigma}} \angle (\mathbf{w}, \hat{\mathbf{w}}) \cdot \alpha$.
  Next, we have 
  \begin{equation*}
  \begin{aligned}
    \left\Vert 
      \frac{
        \alpha \tilde{\mathbf{X}}^T \mathbf{w}
        }{\Vert \mathbf{w} \Vert_\mathbf{\Sigma}}
      -
      \cos_{\mathbf{\Sigma}} \angle (\mathbf{w}, \hat{\mathbf{w}})
      \cdot
      \frac{
      \alpha \tilde{\mathbf{X}}^T \hat{\mathbf{w}}
      }{\Vert \hat{\mathbf{w}} \Vert_\mathbf{\Sigma}}
    \right\Vert
    & \le 
    \frac{\alpha}{\Vert \mathbf{w} \Vert_\mathbf{\Sigma}} 
    \left\Vert 
      \tilde{\mathbf{X}}^{T} \mathbf{w}
      -
      \frac{
        \langle \mathbf{w}, \hat{\mathbf{w}} \rangle_{\mathbf{\Sigma}}
      }{
        \Vert \hat{\mathbf{w}} \Vert_{\mathbf{\Sigma}}^2
      }
      \tilde{\mathbf{X}}^{T} \hat{\mathbf{w}}
    \right\Vert
    \\ & =
    \frac{\alpha}{\Vert \mathbf{w} \Vert_\mathbf{\Sigma}} 
    \left\Vert 
      \left(
        \mathbf{I} - 
        \frac{
          \mathbf{1} \mathbf{1}^{T}
        }{
          \Vert \mathbf{1} \Vert^2
        }
      \right)
      \mathbf{m}
    \right\Vert
    \\ & =
    \frac{\alpha}{\Vert \mathbf{w} \Vert_\mathbf{\Sigma}} 
    \frac{\Vert \mathbf{m} \Vert}{\Vert \mathbf{1} \Vert}
    \left\Vert 
      \left(
        \mathbf{I} - 
        \frac{
          \mathbf{m} \mathbf{m}^{T}
        }{
          \Vert \mathbf{m} \Vert^2
        }
      \right)
      \mathbf{1}
    \right\Vert
    \\ & =
    \alpha
    \left\Vert 
      \left(
        \mathbf{I} - 
        \frac{
          \mathbf{m} \mathbf{m}^{T}
        }{
          \Vert \mathbf{m} \Vert^2
        }
      \right)
      \mathbf{1}
    \right\Vert
  \end{aligned}
  \end{equation*}
  Now take all of above bounds into (\ref{eq:gradient_norm_upper_decomposition}) together
  \begin{equation*}
  \begin{aligned}
    \left\Vert 
      \nabla_{\mathbf{w}} \mathcal{R}
    \right\Vert
    & \le 
    \sqrt{\lambda_{\max}}
    \frac{\alpha}{\sqrt{n} \Vert \mathbf{w} \Vert_{\mathbf{\Sigma}}}
    \Big(
      \alpha
      \max \left\{ 
        |[\boldsymbol{\ell^\prime}]_1|,
        \cdots, 
        |[\boldsymbol{\ell^\prime}]_n|, 
        |\ell^{\prime}(\cos_{\mathbf{\Sigma}} \angle (\mathbf{w}, \hat{\mathbf{w}}) \cdot \alpha)|
      \right\} +
    \\ & \ \ \ \ \ \ \ \ \ \ \ \ \ \ \ \ \ \ \ \ \ \ 
      \left|
        \ell^{\prime}
        \left(
          \cos_{\mathbf{\Sigma}} \angle (\mathbf{w}, \hat{\mathbf{w}}) \cdot 
          \alpha
        \right)
      \right|
    \Big)
    \cdot 
    \left\Vert
    \left(
      \mathbf{I} - \frac{\mathbf{m} \mathbf{m}^{T}}{\Vert \mathbf{m} \Vert^{2}}
    \right)
      \mathbf{1}
    \right\Vert
    \\ & \le 
    \sqrt{\lambda_{\max}}
    \frac{\alpha}{\sqrt{n} \Vert \mathbf{w} \Vert_{\mathbf{\Sigma}}}
    \left( \alpha + 1 \right)
    \left\Vert
    \left(
      \mathbf{I} - \frac{\mathbf{m} \mathbf{m}^{T}}{\Vert \mathbf{m} \Vert^{2}}
    \right)
      \mathbf{1}
    \right\Vert
  \end{aligned}
  \end{equation*}
  By (\ref{eq:pairwise_margin_sum_projection}), we have 
  \begin{equation*}
  \begin{aligned}
    \left\Vert
    \left(
      \mathbf{I} - \frac{\mathbf{m} \mathbf{m}^{T}}{\Vert \mathbf{m} \Vert^{2}}
    \right)
      \mathbf{1}
    \right\Vert^2 
    = n
    \left\Vert
      \hat{\mathbf{w}} - \frac{\langle \mathbf{w},  \hat{\mathbf{w}} \rangle_{\mathbf{\Sigma}} }{ \Vert \mathbf{w} \Vert_{\mathbf{\Sigma}}^{2}} \mathbf{w}
    \right\Vert_{\mathbf{\Sigma}}^2
  \end{aligned}
  \end{equation*}
  And note that $\mathbf{w} \cdot \langle \mathbf{w}, \hat{\mathbf{w}}\rangle / \Vert \mathbf{w} \Vert^2$ is the projection of $\hat{\mathbf{w}}$ onto $\text{span}\{\mathbf{w}\}$ under the
  Euclidean inner product $\langle\cdot, \cdot\rangle$, therefore we have 
  \begin{equation}\label{eq:orthogonal_projection_upper}
  \begin{aligned}
    \left\Vert
      \hat{\mathbf{w}} - \frac{\langle \mathbf{w},  \hat{\mathbf{w}} \rangle_{\mathbf{\Sigma}} }{ \Vert \mathbf{w} \Vert_{\mathbf{\Sigma}}^{2}} \mathbf{w}
    \right\Vert_{\mathbf{\Sigma}}^2
    \le 
    \left\Vert
      \hat{\mathbf{w}} - 
      \frac{\langle \mathbf{w},  \hat{\mathbf{w}} \rangle }
      { \Vert \mathbf{w} \Vert^{2}} \mathbf{w}
    \right\Vert_{\mathbf{\Sigma}}^2
    \le 
    \lambda_{\max}
    \left\Vert
      \hat{\mathbf{w}} - 
      \frac{\langle \mathbf{w},  \hat{\mathbf{w}} \rangle }
      { \Vert \mathbf{w} \Vert^{2}} \mathbf{w}
    \right\Vert^2
    = 
    \lambda_{\max}
    \left(
      \rho_{t}^{\perp}
    \right)^2
  \end{aligned}
  \end{equation}
    
  \noindent\textbf{Lower Bound}
  We have 
  \begin{equation}\label{eq:gradient_norm_lower_start}
  \begin{aligned}
    \left\Vert 
      \nabla_{\mathbf{w}} \mathcal{R}
    \right\Vert
    =
    \frac{\alpha}{n \Vert \mathbf{w} \Vert_{\mathbf{\Sigma}}}
    \left\Vert
      \tilde{\mathbf{X}}
      \left(
        \mathbf{I} - \frac{\mathbf{m} \mathbf{m}^{T}}{\Vert \mathbf{m} \Vert^{2}}
      \right)
      \boldsymbol{\ell^\prime}
    \right\Vert
    \ge 
    \sqrt{\lambda_{\min}}
    \frac{\alpha}{\sqrt{n} \Vert \mathbf{w} \Vert_{\mathbf{\Sigma}}}
    \left\Vert
      \left(
        \mathbf{I} - \frac{\mathbf{m} \mathbf{m}^{T}}{\Vert \mathbf{m} \Vert^{2}}
      \right)
      \boldsymbol{\ell^\prime}
    \right\Vert
  \end{aligned}
  \end{equation}
  By Lagrange's identity, we have
  \begin{equation*}
  \begin{aligned}
    \left\Vert
      \left(
        \mathbf{I} - \frac{\mathbf{m} \mathbf{m}^{T}}{\Vert \mathbf{m} \Vert^{2}}
      \right)
      \boldsymbol{\ell^\prime}
    \right\Vert^2 
    & = 
    \frac{1}{2 \Vert \mathbf{m} \Vert^2}
    \sum_{i=1}^n \sum_{j=1}^{n}
    \left(
      [\boldsymbol{\ell^\prime}]_i
      \mathbf{m}_j
      -
      [\boldsymbol{\ell^\prime}]_j
      \mathbf{m}_i
    \right)^2 
    \\ & = 
    \frac{1}{2 \Vert \mathbf{m} \Vert^2}
    \sum_{i=1}^n \sum_{j=1}^{n}
    \left( [\boldsymbol{\ell^\prime}]_i [\boldsymbol{\ell^\prime}]_j \right)^2
    \left(
      \mathbf{m}_j - \mathbf{m}_i
    \right)^2
    \left(
      \frac{
        [\boldsymbol{\ell^\prime}]_j^{-1}
        \mathbf{m}_j
        -
        [\boldsymbol{\ell^\prime}]_i^{-1}
        \mathbf{m}_i
      }{
        \mathbf{m}_j - \mathbf{m}_i
      }
    \right)^2
    \\ & \ge 
    \frac{1}{2 \Vert \mathbf{m} \Vert^2}
    \sum_{i=1}^n \sum_{j=1}^{n}
    \max \left(
      [\boldsymbol{\ell^\prime}]_i^2, 
      [\boldsymbol{\ell^\prime}]_j^2
    \right)
    \left(
      \mathbf{m}_j - \mathbf{m}_i
    \right)^2
  \end{aligned},
  \end{equation*}
  By Lemma C.1 of \citep{implicit_bias_BN}, we have 
  \begin{equation}\label{eq:gradient_norm_lower_projection}
  \begin{aligned}
    \left\Vert
      \left(
        \mathbf{I} - \frac{\mathbf{m} \mathbf{m}^{T}}{\Vert \mathbf{m} \Vert^{2}}
      \right)
      \boldsymbol{\ell^\prime}
    \right\Vert^2
    & \ge 
    \frac{1}{2 \Vert \mathbf{m} \Vert^2}
    \frac{1}{4n}
    \sum_{i=1}^{n} [\boldsymbol{\ell^\prime}]_i^2
    \sum_{i=1}^n \sum_{j=1}^{n}
    \left(
      \mathbf{m}_j - \mathbf{m}_i
    \right)^2
    \\ & \ge 
    \frac{1}{4}
    \left(
      \sum_{i=1}^{n} [\boldsymbol{\ell^\prime}]_i^2
    \right)
    \left\Vert
      \hat{\mathbf{w}} - \frac{\langle \mathbf{w},  \hat{\mathbf{w}} \rangle_{\mathbf{\Sigma}} }{ \Vert \mathbf{w} \Vert_{\mathbf{\Sigma}}^{2}} \mathbf{w}
    \right\Vert_{\mathbf{\Sigma}}^2
    \\ & \ge 
    \frac{\lambda_{\min}}{4}
    \left(
      \sum_{i=1}^{n} [\boldsymbol{\ell^\prime}]_i^2
    \right)
    \left(
      \rho_{t}^{\perp}
    \right)^2
  \end{aligned}
  \end{equation}
  Plugging (\ref{eq:pairwise_margin_sum_projection}) and (\ref{eq:orthogonal_projection_lower}) into (\ref{eq:gradient_norm_lower_projection}), and then using (\ref{eq:gradient_norm_lower_start}), gives 
  \begin{equation*}
  \begin{aligned}
    \left\Vert 
      \nabla_{\mathbf{w}} \mathcal{R}
    \right\Vert
    & \ge 
    \frac{\lambda_{\min}}{2}
    \frac{\alpha}{\Vert \mathbf{w} \Vert_{\mathbf{\Sigma}}}
    \cdot
    \frac{1}{\sqrt{n}}
    \Vert \boldsymbol{\ell^\prime} \Vert
    \cdot 
    \rho_{t}^{\perp}
  \end{aligned}
  \end{equation*}
  Then we look at $\frac{1}{\sqrt{n}}
      \Vert \boldsymbol{\ell^\prime} \Vert$
  \begin{equation*}
  \begin{aligned}
    \frac{1}{n}
    \Vert \boldsymbol{\ell^\prime} \Vert^2 
    & = 
    \frac{1}{n} \sum_{i=1}^{n} \left[ 1 + \exp \left( \frac{\alpha \cdot \langle \mathbf{w}, \mathbf{x}_i y_i \rangle}
      {\Vert \mathbf{w} \Vert_{\mathbf{\Sigma}}} 
    \right)
    \right]^{-2}
    \\ & \geq 
    \frac{1}{n} \sum_{i=1}^{n} \left[ 1 + \exp \left( \frac{\alpha \cdot |\langle \mathbf{w}, \mathbf{x}_i y_i \rangle|}
    {\Vert \mathbf{w} \Vert_{\mathbf{\Sigma}}} \right) \right]^{-2}
    \\ & \geq \left[ 1 + \exp \left( \alpha \cdot \frac{1}{n} \sum_{i=1}^{n} 
    \frac{|\langle \mathbf{w}, \mathbf{x}_i y_i \rangle|}{\Vert \mathbf{w} \Vert_{\mathbf{\Sigma}}} \right) \right]^{-2}
    \\ & \geq [1 + \exp(\alpha)]^{-2}
    \\ & \geq \exp(-2 \alpha) / 4.
  \end{aligned}
  \end{equation*}
\end{proof}

\definecolor{shadecolor}{rgb}{0.92,0.92,0.92}
\newtheorem*{restate_First_Phase}{Lemma \ref{First_Phase}}
\begin{shaded}
\begin{restate_First_Phase}
  Let $\ell$ be $\ell_{log}$ and $\hat{\mathbf{w}}$ be the SVM solution as defined in Definition \ref{margin}. 
  Suppose Assumptions \ref{Overparameterization}, \ref{Logistic_Setup}, and \ref{Active_margin_data} hold and $\lambda_{\max} > 1$.
  Consider the gradient descent (\ref{gradient_updata}), 
  for any $\tan_{\min} > 0$, if
  $0 < \alpha_0 \le \frac{1}{3} \log \left(\lambda_{\max}\right)$, 
  \begin{equation*}
  \begin{aligned}
    \frac{\eta}{\Vert \mathbf{w}_0 \Vert^2} \ge 
    C_1 \cdot \left(1 + \frac{1}{\tan_{\min}^2}\right) \gamma
    \ \text{ and } \ 
    \eta_{\alpha} \le C_3 \cdot \frac{\tan_{\min}^2}{1 + \tan_{\min}^2}
    \frac{\Vert \mathbf{w}_0 \Vert^2}{\eta \gamma},
  \end{aligned}
  \end{equation*}
  then there exists $t_0 < T_0$ such that 
  $\left(\rho_{t_0}^{\perp} / \rho_{t_0}\right)^2 \le \tan_{\min}^2$, 
  and for any $t < T_0$, 
  we have $\frac{1}{2}\alpha_0 \le \alpha_t \le \frac{3}{2}\alpha_0$, 
  where $C_1, C_2, C_3$ are some constants and 
  \begin{equation*}
  \begin{aligned}
    & 
    T_0 := C_2 \cdot \left(1 + \frac{1}{\tan_{\min}^2}\right) \frac{\eta \gamma}{\Vert \mathbf{w}_0 \Vert^2};
    \quad C := 
    \frac{32 \sqrt{\lambda_{\max}}}{\sqrt{\lambda_{\min}}}
    \frac{ \exp \left({3 \alpha_0 /2 }\right)}{\alpha_0};
    \\ & C_1 := 
    2 C \cdot
      \left(
      1 + 36 
      \frac{ \lambda_{\max}^2}{\lambda_{\min}^2}
      e^{2\alpha_0}
      \left(
        \rho_0^{\perp, \mathbf{\Sigma}}
      \right)^{-2}
    \right);
    \quad
    C_2 := 
    2 C \cdot
    \alpha_{0}^2 \left( \frac{\lambda_{\max}}{\lambda_{\min}} \right)^{3/2};
    \quad
    C_3 := \frac{\alpha_0}{2 C_2}.
  \end{aligned}
  \end{equation*}
  We mention that $\rho_0^{\perp, \mathbf{\Sigma}}$ is quantity related to initial parameter, which is defined as 
  \begin{equation*}
  \begin{aligned}
    \rho_0^{\perp, \mathbf{\Sigma}} = \left\Vert
      \hat{\mathbf{w}} - \frac{\langle \mathbf{w}_0, \hat{\mathbf{w}}\rangle_{\mathbf{\Sigma}} }
      {\Vert \mathbf{w}_0 \Vert_{\mathbf{\Sigma}}^2 } \mathbf{w}_{0}
    \right\Vert_{\mathbf{\Sigma}}.
  \end{aligned}
  \end{equation*}
\end{restate_First_Phase}
\end{shaded}

\begin{proof}
  Given any $T_0 > 0$, we choose a small enough $\eta_{\alpha}$ to control the growth of $\alpha_{t}$. 
  We have 
  \begin{equation*}
  \begin{aligned}
    \left| \alpha_{t+1} - \alpha_t \right|
    &= 
    \left| - \eta_{\alpha}\frac{\partial \mathcal{R}}{\partial \mathbf{\alpha}} \left(\mathbf{w}_t, \alpha_t \right) \right|
    =
    \frac{ \eta_{\alpha} }{n \Vert \mathbf{w}_t \Vert_\mathbf{\Sigma}}
    \left| 
      \mathbf{w}_t^T \tilde{\mathbf{X}}
      \boldsymbol{\ell^\prime}\left(
        \frac{
        \alpha_t \tilde{\mathbf{X}}^T \mathbf{w}_t
        }{\Vert \mathbf{w}_t \Vert_\mathbf{\Sigma}}
      \right)
    \right|
    \\ & \le
    \frac{\eta_{\alpha}}{n} \Vert \boldsymbol{\ell^{\prime}}_t \Vert 
    \frac{\left\Vert \tilde{\mathbf{X}}^T \mathbf{w}_t \right\Vert}{\Vert \mathbf{w}_t \Vert_{\mathbf{\Sigma}}}
    \le
    \frac{\eta_{\alpha}}{n} \Vert \mathbf{1} \Vert 
    \frac{\left\Vert \tilde{\mathbf{X}}^T \mathbf{w}_t \right\Vert}{\Vert \mathbf{w}_t \Vert_{\mathbf{\Sigma}}}
    = \eta_{\alpha}.
  \end{aligned}
  \end{equation*}
  If $\eta_{\alpha} \le \frac{\alpha_0}{2 T_0}$, for any $t \in [0, T_0)$, we have 
  \begin{equation*}
  \begin{aligned}
    | \alpha_{t} - \alpha_0 | \le t \eta_{\alpha} \le T_0 \eta_{\alpha} \le \frac{\alpha_0}{2},
    \quad 
    \frac{1}{2} \alpha_0 \le \alpha_{t} \le \frac{3}{2} \alpha_0
  \end{aligned}
  \end{equation*}
  Then, we calculate the bounds of $\Vert \mathbf{w}_{T_0} \Vert$ and $\Vert \mathbf{w}_{1} \Vert$. 
  By Lemma \ref{Gradient_Norm_Bound} (refer to the version in Appendix), we have the lower bound of $\Vert \mathbf{w}_{1} \Vert^2$
  \begin{equation*}
  \begin{aligned}
    \Vert \mathbf{w}_{1} \Vert^2 
    & = 
    \Vert \mathbf{w}_{0} \Vert^2 
    + \eta^2 
    \left\Vert 
        \frac{\partial \mathcal{R}}{\partial \mathbf{w}} \left(\mathbf{w}_0, \alpha_0\right)
    \right\Vert^2
    \\ & \ge 
    \Vert \mathbf{w}_{0} \Vert^2  
    + 
    \frac{\eta^2 }{16}
    \lambda_{\min}
    \frac{1}{\Vert \mathbf{w}_0 \Vert_{\mathbf{\Sigma}}^2}
    \frac{\alpha_0^2}{e^{2\alpha_0}}
    \left(
      \rho_0^{\perp, \mathbf{\Sigma}}
    \right)^2
    \\ & \ge 
    \Vert \mathbf{w}_{0} \Vert^2  
    + 
    \frac{\eta^2 }{16}
    \frac{\lambda_{\min} }{\lambda_{\max}}
    \frac{1}{\Vert \mathbf{w}_0 \Vert^2}
    \frac{\alpha_0^2}{e^{2\alpha_0}}
    \left(
      \rho_0^{\perp, \mathbf{\Sigma}}
    \right)^2
    \\ & \ge 
    \frac{\eta^2 }{16}
    \frac{\lambda_{\min} }{\lambda_{\max}}
    \frac{1}{\Vert \mathbf{w}_0 \Vert^2}
    \frac{\alpha_0^2}{e^{2\alpha_0}}
    \left(
      \rho_0^{\perp, \mathbf{\Sigma}}
    \right)^2
  \end{aligned}
  \end{equation*}
  And we apply the first upper bound of gradient in Lemma \ref{Gradient_Norm_Bound} 
  to obtain the upper bound of $\Vert \mathbf{w}_{T_0} \Vert$:
  \begin{equation*}
  \begin{aligned}
    \Vert \mathbf{w}_{1} \Vert^2 
    = 
    \Vert \mathbf{w}_{0} \Vert^2 
    + \eta^2 
    \left\Vert 
        \frac{\partial \mathcal{R}}{\partial \mathbf{w}} \left(\mathbf{w}_0, \alpha_0\right)
    \right\Vert^2
    \le 
    \Vert \mathbf{w}_{0} \Vert^2 
    +
    \frac{\eta^2 \lambda_{\max}}{\lambda_{\min}}
    \frac{\alpha_{0}^2}{\Vert \mathbf{w}_0 \Vert^2},
  \end{aligned}
  \end{equation*}
  and the upper bound of $\Vert \mathbf{w}_{1} \Vert^2$:
  \begin{equation*}
  \begin{aligned}
    \Vert \mathbf{w}_{T_0} \Vert^2 
    & = 
    \Vert \mathbf{w}_{1} \Vert^2 
    + \eta^2 
    \sum_{\tau=1}^{T_0 - 1}
    \left\Vert 
        \frac{\partial \mathcal{R}}{\partial \mathbf{w}} \left(\mathbf{w}_{\tau}, \alpha_{\tau}\right)
    \right\Vert^2
    \\ & \le 
    \Vert \mathbf{w}_{1} \Vert^2 
    + 
    \frac{\eta^2 \lambda_{\max}}{\lambda_{\min}}
    \sum_{\tau=1}^{T_0 - 1}
    \frac{\alpha_{\tau}^2}{\Vert \mathbf{w}_\tau \Vert^2}
    \\ & \le 
    \Vert \mathbf{w}_{1} \Vert^2 
    + 
    \left(T_0 - 1\right) \eta^2 \frac{ \lambda_{\max}}{\lambda_{\min}}
    \max_{\tau \in [1, T_0)} \left\{ \alpha_{\tau}^2 \right\}
    \frac{1}{\Vert \mathbf{w}_{1} \Vert^2}
    \\ & \le 
    \Vert \mathbf{w}_{1} \Vert^2 
    + 
     \frac{9}{4} \frac{ \alpha_{0}^2 \lambda_{\max}}{\lambda_{\min}}
    \frac{1}{\Vert \mathbf{w}_{1} \Vert^2} T_0 \eta^2.
  \end{aligned}
  \end{equation*}

  Recall the bound in Lemma \ref{inner_product_bound}, we have 
  \begin{equation*}
    \begin{aligned}
      - \left\langle \hat{\mathbf{w}}, \frac{\partial \mathcal{R}}{\partial \mathbf{w}} \left(\mathbf{w}, \alpha\right) \right\rangle 
      \ge 
      \lambda_{\min} \frac{\alpha }{\Vert \mathbf{w} \Vert_{\mathbf{\Sigma}}}
      \frac{e^{-\alpha}}{8}
      \left(
        \rho^{\perp}\left(\mathbf{w}\right)
      \right)^2 
  \end{aligned}
  \end{equation*}
  Therefore, by gradient descent update, we have for any $t \in [0, T_0)$
  \begin{equation*}
  \begin{aligned}
    \left\langle \mathbf{w}_{t+1}, \hat{\mathbf{w}} \right\rangle - \left\langle \mathbf{w}_{t}, \hat{\mathbf{w}} \right\rangle 
    & = 
    -
    \eta \left\langle 
      \frac{\partial \mathcal{R}}{\partial \mathbf{w}} \left(\mathbf{w}_t, \alpha_t\right)
    , \hat{\mathbf{w}} \right\rangle 
    \\ & \ge 
    \eta \frac{\lambda_{\min}}{8}
    \frac{1}{\Vert \mathbf{w}_t \Vert_{\mathbf{\Sigma}}}
      \frac{\alpha_t }{ e^{\alpha_t}}
    \left(
      \rho_t^{\perp}
    \right)^2 
    \\ & \ge 
    \eta \frac{\lambda_{\min}}{8 \sqrt{\lambda_{\max}}}
    \frac{1}{\Vert \mathbf{w}_t \Vert}
      \frac{\alpha_t }{ e^{\alpha_t}}
    \left(
      \rho_t^{\perp}
    \right)^2 
    \\ & \ge 
    \eta \frac{\alpha_0}{16 \exp\left(\frac{3}{2}\alpha_0\right)} \frac{\lambda_{\min}}{\sqrt{\lambda_{\max}}}
    \frac{1}{\Vert \mathbf{w}_t \Vert}
    \left(
      \rho_t^{\perp}
    \right)^2 
  \end{aligned}
  \end{equation*}
  The expression can be rearranged to obtain
  \begin{equation*}
  \begin{aligned}
    \left(
      \rho_t^{\perp}
    \right)^2 
    \le 
    \frac{16 \sqrt{\lambda_{\max}} }{\lambda_{\min}}
    \frac{ \exp \left({3 \alpha_0 /2 }\right)}{\alpha_0 }
    \frac{\Vert \mathbf{w}_t \Vert}{\eta}
    \left(
      \left\langle \mathbf{w}_{t+1}, \hat{\mathbf{w}} \right\rangle - \left\langle \mathbf{w}_{t}, \hat{\mathbf{w}} \right\rangle 
    \right)
  \end{aligned}
  \end{equation*}
  Further, we have 
  \begin{equation*}
  \begin{aligned}
    \min_{\tau \in [0, T_0)} 
    \left(
      \rho_{\tau}^{\perp}
    \right)^2 
    & \le 
    \frac{1}{T_0}
    \sum_{\tau = 0}^{T_0-1}
    \left(
      \rho_{\tau}^{\perp}
    \right)^2 
    \\ & \le 
    \frac{16 \sqrt{\lambda_{\max}}}{\lambda_{\min}}
    \frac{ \exp \left({3 \alpha_0 /2 }\right)}{\alpha_0 }
    \frac{1}{\eta}
    \sum_{\tau = 0}^{T_0-1}
    \Vert \mathbf{w}_\tau \Vert
    \left(
      \left\langle \mathbf{w}_{\tau+1}, \hat{\mathbf{w}} \right\rangle - \left\langle \mathbf{w}_{\tau}, \hat{\mathbf{w}} \right\rangle 
    \right)
    \\ & \le 
    \frac{16 \sqrt{\lambda_{\max}}}{\lambda_{\min}}
    \frac{ \exp \left({3 \alpha_0 /2 }\right)}{\alpha_0 }
    \max_{\tau \in [0, T_0)} \left\{ \Vert \mathbf{w}_\tau \Vert \right\}
    \frac{1}{\eta}
    \sum_{\tau = 0}^{T_0-1}
    \left(
      \left\langle \mathbf{w}_{\tau+1}, \hat{\mathbf{w}} \right\rangle - \left\langle \mathbf{w}_{\tau}, \hat{\mathbf{w}} \right\rangle 
    \right)
    \\ & = 
    \frac{16 \sqrt{\lambda_{\max}}}{\lambda_{\min}}
    \frac{ \exp \left({3 \alpha_0 /2 }\right)}{\alpha_0 }
    \max_{\tau \in [0, T_0)} \left\{ \Vert \mathbf{w}_\tau \Vert \right\}
    \frac{1}{T_0 \eta}
    \left(
      \left\langle \mathbf{w}_{T_0}, \hat{\mathbf{w}} \right\rangle - \left\langle \mathbf{w}_{0}, \hat{\mathbf{w}} \right\rangle 
    \right)
  \end{aligned}
  \end{equation*}
  By Cauchy inequalities, we have 
  \begin{equation*}
  \begin{aligned}
    \min_{\tau \in [0, T_0)} 
    \left(
      \rho_{\tau}^{\perp}
    \right)^2  & \le
    \frac{16 \sqrt{\lambda_{\max}}}{\gamma \lambda_{\min}}
    \frac{ \exp \left({3 \alpha_0 /2 }\right)}{\alpha_0 }
    \max_{\tau \in [0, T_0)} \left\{ \Vert \mathbf{w}_\tau \Vert \right\}
    \frac{1}{T_0 \eta}
    \left(
      \Vert \mathbf{w}_{0} \Vert + \Vert \mathbf{w}_{T_0} \Vert
    \right)
    \\ & \le  
    \frac{32 \sqrt{\lambda_{\max}}}{\gamma \lambda_{\min}}
    \frac{ \exp \left({3 \alpha_0 /2 }\right)}{\alpha_0 }
    \max_{\tau \in [0, T_0)} \left\{ \Vert \mathbf{w}_\tau \Vert \right\}
    \frac{\Vert \mathbf{w}_{T_0} \Vert}{T_0\eta}
    \\ & \le
    \frac{32 \sqrt{\lambda_{\max}}}{\lambda_{\min}}
    \frac{ \exp \left({3 \alpha_0 /2 }\right)}{\alpha_0 }
    \frac{\Vert \mathbf{w}_{T_0} \Vert^2}{T_0 \eta \gamma}
    \\ & \le
    \frac{ 32 }{\alpha_0 }
    \frac{ \lambda_{\max}}{\lambda_{\min}}
    \frac{\Vert \mathbf{w}_{T_0} \Vert^2}{T_0 \eta \gamma}
    \\ & =
    {C}
    \frac{\Vert \mathbf{w}_{T_0} \Vert^2}{T_0 \eta \gamma}
    ,
  \end{aligned}
  \end{equation*}
  where
  \[
    C = \frac{ 32 }{\alpha_0 }
    \frac{ \lambda_{\max}}{\lambda_{\min}}.
  \]
  Now, let us check $\frac{\Vert \mathbf{w}_{T_0} \Vert^2}{T_0 \eta \gamma}$. 
  Recall just derived bounds of $\Vert \mathbf{w}_{T_0} \Vert^2$ and $\Vert \mathbf{w}_{1} \Vert^2$, we have 
  \begin{equation*}
  \begin{aligned}
    \frac{\Vert \mathbf{w}_{T_0} \Vert^2}{T_0 \eta \gamma} 
    & \le 
    \frac{1}{T_0 \eta \gamma} 
    \left(
      \Vert \mathbf{w}_{1} \Vert^2 
      + 
       \frac{9}{4} \frac{ \alpha_{0}^2 \lambda_{\max}}{\lambda_{\min}}
      \frac{1}{\Vert \mathbf{w}_{1} \Vert^2} T_0 \eta^2
    \right)
    \\& = 
    \frac{\Vert \mathbf{w}_{1} \Vert^2 }{T_0 \eta \gamma} 
    +
    \frac{9}{4} \frac{ \alpha_{0}^2 \lambda_{\max}}{\lambda_{\min}}
    \frac{\eta}{\Vert \mathbf{w}_{1} \Vert^2 \gamma}
    \\ & \le 
    \frac{1}{T_0 \eta \gamma} 
    \left(
      \Vert \mathbf{w}_{0} \Vert^2 
      +
      \frac{\eta^2 \lambda_{\max}}{\lambda_{\min}}
      \frac{\alpha_{0}^2}{\Vert \mathbf{w}_0 \Vert^2}
    \right)
    \\ & \quad
    +
    \frac{9}{4} \frac{ \alpha_{0}^2 \lambda_{\max}}{\lambda_{\min}}
    \frac{\eta}{\gamma}
    \left(
      \frac{\eta^2 }{16}
      \frac{\lambda_{\min} }{\lambda_{\max}}
      \frac{1}{\Vert \mathbf{w}_0 \Vert^2}
      \frac{\alpha_0^2}{e^{2\alpha_0}}
      \left(
        \rho_0^{\perp, \mathbf{\Sigma}}
      \right)^2
    \right)^{-1}
    \\ & =
    \Vert \mathbf{w}_{0} \Vert^2 \frac{1}{T_0 \eta \gamma} 
    + \frac{\lambda_{\max}}{\lambda_{\min}}
    \frac{\alpha_{0}^2}{\Vert \mathbf{w}_0 \Vert^2}
    \frac{\eta}{T_0 \gamma} 
    \\ & \quad
    + 36 e^{2\alpha_0}
    \frac{ \lambda_{\max}^2}{\lambda_{\min}^2}
    \Vert \mathbf{w}_0 \Vert^2
    \frac{1}{\eta \gamma}
    \left(
      \rho_0^{\perp, \mathbf{\Sigma}}
    \right)^{-2}
    \\ & \le A_0 + C_{\rho,0} B_0
  \end{aligned}
  \end{equation*}
  where the last inequality uses $T_0 \ge 1$ and
  \begin{equation*}
  \begin{aligned}
    A_0
    & :=
    \frac{\alpha_{0}^2 \lambda_{\max}}{\lambda_{\min}}
    \\ & \quad \cdot
    \frac{\eta}{\Vert \mathbf{w}_{0} \Vert^2 T_0 \gamma},
  \end{aligned}
  \end{equation*}
  \begin{equation*}
    C_{\rho,0}
    :=
    1
    +
    36 e^{2\alpha_0}
    \frac{\lambda_{\max}^2}{\lambda_{\min}^2}
    \left(
      \rho_0^{\perp, \mathbf{\Sigma}}
    \right)^{-2}.
  \end{equation*}
  \begin{equation*}
    B_0
    :=
    \frac{\Vert \mathbf{w}_{0} \Vert^2}{\eta\,\gamma}.
  \end{equation*}
  \begin{equation*}
    Q_0
    :=
    \frac{\Vert \mathbf{w}_{T_0} \Vert^2}{T_0 \eta \gamma}.
  \end{equation*}
  Therefore, given any $\rho_{\min}^{\perp} > 0$, a sufficient condition is
  \begin{equation*}
  \begin{aligned}
    & \ \min_{\tau \in [0, T_0)} 
    \left(
      \rho_{\tau}^{\perp}
    \right)^2 
     \le
    \left(
      \rho_{\min}^{\perp}
    \right)^2 
    \boldsymbol{\Leftarrow} \ 
    C \cdot 
    \frac{\Vert \mathbf{w}_{T_0} \Vert^2}{T_0 \eta \gamma} 
    \le
    \left(
      \rho_{\min}^{\perp}
    \right)^2
  \end{aligned}
  \end{equation*}
  A sufficient way to enforce the last display is
  \begin{equation*}
  \begin{aligned}
    \frac{\eta}{\Vert \mathbf{w}_0 \Vert^2}
    & \ge C_1 \cdot \frac{1}{ \left(\rho_{\min}^{\perp}\right)^2 }
    \frac{1}{\gamma},
    \\
    T_0
    & \ge C_2 \cdot \frac{\eta}{\Vert \mathbf{w}_0 \Vert^2} \cdot \frac{1}{ \left(\rho_{\min}^{\perp}\right)^2 } 
    \frac{1}{\gamma},
  \end{aligned}
  \end{equation*}
  where $C_1$ and $C_2$ are explicit constants determined by the initialization and the data:
  \begin{equation*}
  \begin{aligned}
    C_1 := 
    2 C \cdot
      \left(
      1 + 36 
      \frac{ \lambda_{\max}^2}{\lambda_{\min}^2}
      e^{2\alpha_0}
      \left(
        \rho_0^{\perp, \mathbf{\Sigma}}
      \right)^{-2}
    \right);
    \ \ 
    C_2 := 
    2 C \cdot
    \alpha_{0}^2 \frac{\lambda_{\max}}{\lambda_{\min}}
    .
  \end{aligned}
  \end{equation*}
  To obtain the bound for $\rho_t^{\perp} / \rho_t$, we set $\left(\rho_{\min}^{\perp}\right)^2 = 
  {\tan_{\min}^2}/\left(\gamma^2 \cdot \left(1 + \tan_{\min}^2\right)\right)$. 
  Using $1/\gamma^2 = \left(\rho_t^{\perp}\right)^2 + \rho_t^2$, if $\rho_t^{\perp} \leq \rho_{\min}^{\perp}$ 
  as defined above, then $\rho_t^{\perp} / \rho_t \leq \tan_{\min}$. Therefore, we have 
  \begin{equation*}
  \begin{aligned}
    & \ \min_{\tau \in [0, T_0)} 
    \left(
      \rho_{\tau}^{\perp} / \rho_{\tau}
    \right)^2 
     \le
      \tan_{\min}^{2}
    \\  \boldsymbol{\Leftarrow} & \ 
    \frac{\eta}{\Vert \mathbf{w}_0 \Vert^2} \ge C_1 \cdot \left(1 + \frac{1}{\tan_{\min}^2}\right)
    \gamma
    \text{ and }
    T_0 \ge C_2 \cdot \left(1 + \frac{1}{\tan_{\min}^2}\right) \frac{\eta \gamma}{\Vert \mathbf{w}_0 \Vert^2}.
  \end{aligned}
  \end{equation*}
  Finally, remember that all the sufficient conditions above hold under the condition of $\eta_{\alpha} \le \frac{1}{2 T_0} \alpha_0$. 
  Therefore, we have the condition for $\eta_{\alpha}$
  \begin{equation*}
  \begin{aligned}
    \eta_{\alpha} \le 
    \frac{1}{2 T_0} \alpha_0
    =
    C_3 \cdot \frac{\tan_{\min}^2}{1 + \tan_{\min}^2}
    \frac{\Vert \mathbf{w}_0 \Vert^2}{\eta \gamma}
    ,
    \text{ where }
    C_3 := \frac{\alpha_0}{2 C_2}
  \end{aligned}
  \end{equation*}
\end{proof}

\newtheorem*{restate_logis_spike}{Theorem \ref{logis_spike}}
\definecolor{shadecolor}{rgb}{0.92,0.92,0.92}
\begin{shaded}
\begin{restate_logis_spike}
  Let $\ell$ be $\ell_{log}$ and $\hat{\mathbf{w}}$ be the SVM solution as defined in Definition \ref{margin}. 
  Suppose Assumptions \ref{Overparameterization}, \ref{Logistic_Setup}, and \ref{Active_margin_data} hold and $\lambda_{\max} > 1$. If 
  $0 < \alpha_0 \le \frac{1}{3} \log \left( \lambda_{\max} \right)$ and 
  \begin{equation*}
  \begin{aligned}
    \max \left( 
      C_1 \frac{16 \lambda_{\max}^2  }{ \lambda_{\min}^2 }, 
      C_4
    \right)  \cdot \gamma
    \le \frac{\eta}{\Vert \mathbf{w}_0 \Vert^2} \le C_5 \cdot \gamma^{-1}
    \text{  and  }
    \eta_{\alpha} \le 
    C_3 \cdot \frac{\lambda_{\min}^2}{16\lambda_{\max}^2} \cdot \frac{\Vert \mathbf{w}_0 \Vert^2}{\eta \gamma}, 
  \end{aligned}
  \end{equation*}
  it holds that during $t \in [0, T_0)$
  \begin{itemize}
    \item (1). $\rho_t^{\perp}$ keeps decreasing as long as $\left( {\rho_t^{\perp}}/{\rho_t} \right)^2 \ge 
    {4}/{\left(\sqrt{ \Phi^2 + 4} - \Phi\right)^2} - 1$ and $\rho_t > 0$; 

    \item (2). $\exists t_0$ such that $\left(\rho_{t_0}^{\perp}/\rho_{t_0}\right)^2 \le 
    \tan_{\min}^2$; 
    
    \item (3). if, after item (2) is reached, the directional dynamics enters a \textit{Rising Edge} segment on the positive-alignment branch $\rho_t > 0$, then every iterate on that segment satisfying $\left(\rho_t^{\perp} / \rho_t\right)^2 \ge C_6 \cdot \frac{\eta^2}{\Vert \mathbf{w}_0 \Vert^4} \cdot \frac{1}{\gamma^2} - 1$
    already meets the convergence condition of Lemma \ref{Direction_Convegence_and_Divergence}, so the next iterate leaves the \textit{Rising Edge}.
  \end{itemize}
  We mention $T_0$ and $\tan_{\min}$ are defined as:
  \begin{equation*}
  \begin{aligned}
    T_0 = C_2 \cdot \left( 1 + \frac{1}{\tan_{\min}^2}\right) \cdot \frac{\eta \gamma}{\Vert \mathbf{w}_0 \Vert^2}, 
    \quad 
    \tan_{\min}^2 = \frac{\gamma^2 \lambda_{\min}}{8 \lambda_{\max}^2}.
  \end{aligned}
  \end{equation*}
  And $C_1, C_2, C_3 (\text{defined in Lemma \ref{First_Phase}})$ are constants, depending in particular on $\alpha_0$ and $\rho_0^{\perp, \mathbf{\Sigma}}$, and 
  \begin{equation*}
  \begin{aligned}
    & \Phi := \frac{6\lambda_{\max}^2 }{\lambda_{\min}^2 }
    \alpha_0 \cdot \frac{\eta \gamma}{\Vert \mathbf{w}_{0} \Vert^2}; 
    \quad
    \tilde{C} := \frac{3}{256} \frac{\lambda_{\min}}{\lambda_{\max}} \alpha_0;
    \quad C_4 := 2 / \tilde{C}; 
    \\ & 
    C_5 := 
    \sqrt{\frac{\tilde{C}}{2} / 
    \left(
      36 C_2 \frac{ \alpha_{0}^2 \lambda_{\max}^3}{\lambda_{\min}^3}
    \right)}, 
    \quad 
    C_6 := 
    \left(
      6
      \frac{\lambda_{\max}^{5/2}}{\lambda_{\min}^{5/2} }
      \alpha_0 e^{3 \alpha_0 / 2}\left( \frac{3}{2} \alpha_0 + 1 \right)^2
    \right)^2.
  \end{aligned}
  \end{equation*}
\end{restate_logis_spike}
\end{shaded}
\begin{proof}
  We observe that $\frac{\lambda_{\min}^2 }{8 \lambda_{\max}^2} \le \tan_{\min}^2 \le 1/4 < 1$ since $\lambda_{\min} \le \gamma^2 \le \lambda_{\max}$. 
  Then, by the lower bound of $\frac{\eta}{\Vert \mathbf{w}_0 \Vert^2}$, we have 
  \begin{equation*}
  \begin{aligned}
    \frac{\eta}{\Vert \mathbf{w}_0 \Vert^2}
    \ge
    C_1 \frac{16 \lambda_{\max}^2  }{ \lambda_{\min}^2 } \gamma 
    \ge
    \frac{2 C_1 }{ \tan_{\min}^2 } \gamma 
    \ge 
    C_1 \left(1 + \frac{1}{\tan_{\min}^2}\right) \gamma.
  \end{aligned}
  \end{equation*} 
  Next, we consider the bound for $\eta_{\alpha}$
  \begin{equation*}
  \begin{aligned}
    \eta_{\alpha} 
    \le 
    C_3 \cdot \frac{\lambda_{\min}^2}{16\lambda_{\max}^2} \cdot \frac{\Vert \mathbf{w}_0 \Vert^2}{\eta \gamma}
    \le 
    C_3 \cdot \frac{\tan_{\min}^2}{2} \cdot \frac{\Vert \mathbf{w}_0 \Vert^2}{\eta \gamma}
    \le 
    C_3 \cdot \frac{\tan_{\min}^2}{1 + \tan_{\min}^2} \cdot \frac{\Vert \mathbf{w}_0 \Vert^2}{\eta \gamma}
  \end{aligned}
  \end{equation*}
  Therefore, by Lemma \ref{First_Phase}, we know there exists a $t_0 < T_0$ such that 
  $\left( {\rho_{t_0}^\perp}/{\rho_{t_0}} \right)^2 \le \tan_{\min}^2 \le 1/4$,
  and for any $t < T_0$, $\frac{1}{2} \alpha_0 \le \alpha_{t} \le \frac{3}{2} \alpha_0$.
  Item (2) is exactly the conclusion of Lemma \ref{First_Phase}.

  We next explain the auxiliary step-size restriction $C_4 \gamma \le \eta / \Vert \mathbf{w}_0 \Vert^2 \le C_5 \gamma^{-1}$ that appears in the theorem statement. Starting from a time $t_0$ on the positive-alignment branch with $\rho_{t_0} > 0$, the second conclusion of Lemma \ref{Direction_Convegence_and_Divergence} applies once
  \begin{equation}\label{eq:positive_branch_exit_condition}
  \begin{aligned}
    \frac{\eta }{\Vert \mathbf{w}_{t_0} \Vert} 
      \left\Vert 
        \nabla_{\mathbf{w}} \mathcal{R} \left(\mathbf{w}_{t_0}, \alpha_{t_0}\right)
      \right\Vert
      \ge
      \frac{2  {\rho}_{t_0} {\rho}_{t_0}^{\perp}}{
        {\rho}_{t_0}^2 
        - \left({\rho}_{t_0}^{\perp}\right)^2
      }
  \end{aligned}
  \end{equation}
  By Lemma \ref{Gradient_Norm_Bound}, we have
  \begin{equation}\label{eq:positive_branch_exit_reduction}
  \begin{aligned}
    (\ref{eq:positive_branch_exit_condition}) 
    \ \boldsymbol{\Leftarrow} \ &
    \frac{1}{4}
    \lambda_{\min}
    \frac{\alpha_{t_0}}{e^{\alpha_{t_0}}}
    \frac{\eta}{\Vert \mathbf{w}_{t_0} \Vert \Vert \mathbf{w}_{t_0} \Vert_{\mathbf{\Sigma}}}
    \rho_{t_0}^{\perp}
    \ge
    \frac{2  {\rho}_{t_0} {\rho}_{t_0}^{\perp}}{
      {\rho}_{t_0}^2 
      - \left({\rho}_{t_0}^{\perp}\right)^2
    }
    \\ \ \boldsymbol{\Leftrightarrow} \ &
    \frac{1}{8}
    \lambda_{\min}
    \frac{\alpha_{t_0} }{e^{\alpha_{t_0}}}
    \left(
      1
      - 
      \left( \frac{{\rho}_{t_0}^{\perp}}{{\rho}_{t_0}}\right)^2
    \right)
    \ge
    \frac{\Vert \mathbf{w}_{t_0} \Vert \Vert \mathbf{w}_{t_0} \Vert_{\mathbf{\Sigma}}}{\eta \rho_{t_0}}
  \end{aligned}
  \end{equation}
  Since $1/2 \alpha_0 \le \alpha_{t} \le 3/2 \alpha_0 \ \forall t < T_0$ and $\rho_{t_0}^{\perp} / \rho_{t_0} \le \frac{1}{2}$, we have 
  \begin{equation}\label{eq:positive_branch_exit_alpha_reduction}
  \begin{aligned}
    (\ref{eq:positive_branch_exit_reduction}) 
    \ \boldsymbol{\Leftarrow} \ &
    \frac{1}{16}
    \lambda_{\min}
    \frac{\alpha_0 }{e^{3 \alpha_0 / 2 }}
    \left(
      1
      - 
      \left( \frac{{\rho}_{t_0}^{\perp}}{{\rho}_{t_0}}\right)^2
    \right)
    \ge
    \frac{\Vert \mathbf{w}_{t_0} \Vert \Vert \mathbf{w}_{t_0} \Vert_{\mathbf{\Sigma}}}{\eta \rho_{t_0}}
    \\ \ \boldsymbol{\Leftarrow} \ &
    \frac{3}{64}
    \lambda_{\min}
    \frac{\alpha_0 }{e^{3 \alpha_0 / 2 }}
    \ge
    \frac{\Vert \mathbf{w}_{t_0} \Vert \Vert \mathbf{w}_{t_0} \Vert_{\mathbf{\Sigma}}}{\eta \rho_{t_0}}
  \end{aligned}
  \end{equation}
  Note that 
  by Lemma \ref{logit_bound}, we have 
  \begin{equation*}
  \begin{aligned}
    \frac{\Vert \mathbf{w}_{t_0} \Vert_{\mathbf{\Sigma}} }{\Vert \mathbf{w}_{t_0} \Vert}
    & \le 
    \gamma^2 \rho_{t_0} + 2 \sqrt{2} \cdot \lambda_{\max} \cdot \gamma \frac{\Vert \mathbf{w}_{t_0} \Vert }{\Vert \mathbf{w}_{t_0} \Vert_{\mathbf{\Sigma}}} \rho_{t_0}^{\perp}
    \\ & \le 
    \gamma^2 \rho_{t_0} + 2 \sqrt{2} \cdot \frac{\lambda_{\max}}{\sqrt{\lambda_{\min}}} \cdot \gamma \rho_{t_0}^{\perp}
    \\ & \le 
    2 \gamma^2 \rho_{t_0}
  \end{aligned}
  \end{equation*}
  Therefore, we have 
  \begin{equation}\label{eq:positive_branch_exit_norm_reduction}
  \begin{aligned}
    (\ref{eq:positive_branch_exit_alpha_reduction}) 
    \ \boldsymbol{\Leftarrow} \ &
    \frac{3}{256}
    \lambda_{\min}
    \frac{\alpha_0 }{e^{3 \alpha_0 / 2 }}
    \ge
    \frac{\gamma^2}{\eta} \Vert \mathbf{w}_{t_0} \Vert^2
  \end{aligned}
  \end{equation}
  Since $\Vert \mathbf{w}_{T_0} \Vert \ge \Vert \mathbf{w}_{t} \Vert \ \forall t < T_0$, we have 
  \begin{equation}\label{eq:positive_branch_exit_terminal_norm}
  \begin{aligned}
    (\ref{eq:positive_branch_exit_norm_reduction}) 
    \ \boldsymbol{\Leftarrow} \ &
    \frac{3}{256}
    \lambda_{\min}
    \frac{\alpha_0 }{e^{3 \alpha_0 / 2 }}
    \ge
    \frac{\gamma^2}{\eta} \Vert \mathbf{w}_{T_0} \Vert^2
  \end{aligned}
  \end{equation}
  For $\Vert \mathbf{w}_{T_0} \Vert^2$, we apply the first upper bound of gradient in Lemma \ref{Gradient_Norm_Bound} to obtain
  \begin{equation*}
  \begin{aligned}
    \Vert \mathbf{w}_{T_0} \Vert^2 
    & = 
    \Vert \mathbf{w}_{0} \Vert^2 
    + \eta^2 
    \sum_{\tau=0}^{T_0 - 1}
    \left\Vert 
        \nabla_{\mathbf{w}} \mathcal{R} \left(\mathbf{w}_{\tau}, \alpha_{\tau}\right)
    \right\Vert^2
    \\ & \le 
    \Vert \mathbf{w}_{0} \Vert^2 
    + 
    \frac{\eta^2 \lambda_{\max}}{\lambda_{\min}}
    \sum_{\tau=0}^{T_0 - 1}
    \frac{\alpha_{\tau}^2}{\Vert \mathbf{w}_\tau \Vert^2}
    \\ & \le 
    \Vert \mathbf{w}_{0} \Vert^2 
    + 
    T_0 \eta^2 \frac{ \lambda_{\max}}{\lambda_{\min}}
    \max_{\tau \in [0, T_0)} \left\{ \alpha_{\tau}^2 \right\}
    \frac{1}{\Vert \mathbf{w}_{0} \Vert^2}
    \\ & \le 
    \Vert \mathbf{w}_{0} \Vert^2 
    + 
      \frac{9}{4} \frac{ \alpha_{0}^2 \lambda_{\max}}{\lambda_{\min}}
    \frac{1}{\Vert \mathbf{w}_{0} \Vert^2} T_0 \eta^2
  \end{aligned}
  \end{equation*}
  Therefore, it holds that 
  \begin{equation}\label{eq:positive_branch_exit_initial_sufficient}
  \begin{aligned}
    (\ref{eq:positive_branch_exit_terminal_norm}) 
    \ \boldsymbol{\Leftarrow} \ &
    \frac{3}{256}
    \lambda_{\min}
    \frac{\alpha_0 }{e^{3 \alpha_0 / 2 }}
    \ge
    \frac{\gamma^2}{\eta}
    \left(
      \Vert \mathbf{w}_{0} \Vert^2 
      + 
      \frac{9}{4} \frac{ \alpha_{0}^2 \lambda_{\max}}{\lambda_{\min}}
      \frac{1}{\Vert \mathbf{w}_{0} \Vert^2} T_0 \eta^2
    \right)
    \\ 
    \ \boldsymbol{\Leftrightarrow} \ &
    \frac{3}{256}
    \lambda_{\min}
    \frac{\alpha_0 }{e^{3 \alpha_0 / 2 }}
    \ge 
    \gamma^2
    \frac{\Vert \mathbf{w}_{0} \Vert^2 }{\eta}
    + 
    \frac{9}{4} \frac{ \alpha_{0}^2 \lambda_{\max}}{\lambda_{\min}}
    \frac{1}{\Vert \mathbf{w}_{0} \Vert^2} 
    T_0 \eta \gamma^2
    \\ 
    \ \boldsymbol{\Leftrightarrow} \ &
    \frac{3}{256}
    \lambda_{\min}
    \frac{\alpha_0 }{e^{3 \alpha_0 / 2 }}
    \ge 
    \gamma^2 \frac{\Vert \mathbf{w}_{0} \Vert^2}{\eta} 
    + 
    \frac{9}{4} C_2 \frac{ \alpha_{0}^2 \lambda_{\max}}{\lambda_{\min}}
    \left( 1 + \frac{1}{\tan_{\min}^2}\right)
    \frac{\eta^2 \gamma^3}{\Vert \mathbf{w}_{0} \Vert^4} 
    \\ 
    \ \boldsymbol{\Leftarrow} \ &
    \frac{3}{256}
    \lambda_{\min}
    \frac{\alpha_0 }{e^{3 \alpha_0 / 2 }}
    \ge 
    \gamma^2 \frac{\Vert \mathbf{w}_{0} \Vert^2}{\eta}
    + 
    \frac{9}{4} C_2 \frac{ \alpha_{0}^2 \lambda_{\max}}{\lambda_{\min}}
    \left(
      1+ 
      \frac{8 \lambda_{\max}^2}{\lambda_{\min}^2}
    \right)
    \frac{\eta^2 \gamma^3}{\Vert \mathbf{w}_{0} \Vert^4}  
    \\ 
    \ \boldsymbol{\Leftarrow} \ &
    \frac{3}{256}
    \lambda_{\min}
    \frac{\alpha_0 }{e^{3 \alpha_0 / 2 }}
    \ge 
    \gamma^2 \frac{\Vert \mathbf{w}_{0} \Vert^2}{\eta}
    + 
    36 C_2 \frac{ \alpha_{0}^2 \lambda_{\max}^3}{\lambda_{\min}^3}
    \frac{\eta^2 \gamma^3 }{\Vert \mathbf{w}_{0} \Vert^4 } 
    \\ 
    \ \boldsymbol{\Leftarrow} \ &
    \frac{3}{256}
    \frac{\lambda_{\min}}{\sqrt{\lambda_{\max}}}
    \alpha_0
    \ge 
    \gamma^2 \frac{\Vert \mathbf{w}_{0} \Vert^2}{\eta}
    + 
    36 C_2 \frac{ \alpha_{0}^2 \lambda_{\max}^3}{\lambda_{\min}^3}
    \frac{\eta^2 \gamma^3 }{\Vert \mathbf{w}_{0} \Vert^4 } 
    \\ 
    \ \boldsymbol{\Leftrightarrow} \ &
    \frac{3}{256}
    \frac{{\lambda_{\min}}}{{\lambda_{\max}}}
    \frac{\sqrt{\lambda_{\max}}}{\gamma}
    \alpha_0
    \ge 
    \gamma \frac{\Vert \mathbf{w}_{0} \Vert^2}{\eta}
    + 
    36 C_2 \frac{ \alpha_{0}^2 \lambda_{\max}^3}{\lambda_{\min}^3}
    \frac{\eta^2 \gamma^2 }{\Vert \mathbf{w}_{0} \Vert^4 }
    \\ 
    \ \boldsymbol{\Leftarrow} \ &
    \frac{3}{256}
    \frac{{\lambda_{\min}}}{{\lambda_{\max}}}
    \alpha_0
    \ge 
    \gamma \frac{\Vert \mathbf{w}_{0} \Vert^2}{\eta}
    + 
    36 C_2 \frac{ \alpha_{0}^2 \lambda_{\max}^3}{\lambda_{\min}^3}
    \frac{\eta^2 \gamma^2 }{\Vert \mathbf{w}_{0} \Vert^4 }
  \end{aligned}
  \end{equation}
  Here we used the definition of $T_0$, the bound $\tan_{\min}^2 \ge \lambda_{\min}^2 / (8 \lambda_{\max}^2)$, the initialization constraint $\alpha_0 \le \frac{1}{3} \log(\lambda_{\max})$, and $\gamma \le \sqrt{\lambda_{\max}}$.
  Combining (\ref{eq:positive_branch_exit_reduction})--(\ref{eq:positive_branch_exit_initial_sufficient}), condition (\ref{eq:positive_branch_exit_condition}) is therefore implied by
  \begin{equation*}
    C_4 \cdot \gamma \le 
    \frac{\eta}{\Vert \mathbf{w}_0 \Vert^2} 
    \le C_5 \cdot \gamma^{-1}.
  \end{equation*}
  This auxiliary implication explains the constants $C_4$ and $C_5$ appearing in the theorem statement, where
  \begin{equation*}
  \begin{aligned}
    \tilde{C} & := \frac{3}{256} \frac{\lambda_{\min}}{\lambda_{\max}} \alpha_0,
    \\
    C_4 & := 2 / \tilde{C},
    \\
    C_5 & := 
    \sqrt{
      \frac{\tilde{C}}{2}
      \Big/
      \left(
        36 C_2 \frac{ \alpha_{0}^2 \lambda_{\max}^3}{\lambda_{\min}^3}
      \right)
    }.
  \end{aligned}
  \end{equation*}

  \noindent \textbf{Proof of (1)}
  Next, we prove that during $t \in [0, T_0)$, the convergence condition is satisfied whenever
  \[
    \left(
      \frac{\rho_t^{\perp}}{\rho_t}
    \right)^2
    \ge
    \frac{4}{\left(\sqrt{ \Phi^2 + 4} - \Phi\right)^2} - 1.
  \]
  To see if $\rho_t^{\perp}$ is decreasing, we need to verify the convergence 
  condition in Lemma \ref{Direction_Convegence_and_Divergence} during $t \in [0, T_0)$:
  \begin{equation}\label{con_11}
  \begin{aligned}
    \frac{\eta {\rho}_{t}}{\Vert \mathbf{w}_{t} \Vert}
    \left\Vert
      \nabla_{\mathbf{w}} \mathcal{R} \left(\mathbf{w}_{t}, \alpha_{t}\right)
    \right\Vert^2
    \le 
    - 2
    \left\langle 
    \hat{\mathbf{w}},
    \nabla_{\mathbf{w}} \mathcal{R} \left(\mathbf{w}_{t}, \alpha_{t}\right)
    \right\rangle
  \end{aligned}
  \end{equation}
  By Lemma \ref{inner_product_bound} and \ref{Gradient_Norm_Bound}, we have
  \begin{equation*}
  \begin{aligned}
    & - \left\langle \hat{\mathbf{w}}, \nabla_{\mathbf{w}} \mathcal{R}_t\right\rangle 
    \ge 
    \frac{\lambda_{\min} }{8 \sqrt{\lambda_{\max} }}
    \frac{\alpha_t  e^{-\alpha_t }}{\Vert \mathbf{w}_t \Vert}
    \left( \rho_t^{\perp} \right)^2;
    \\ &
    \left\Vert 
    \nabla_{\mathbf{w}} \mathcal{R}_t
    \right\Vert
    \le 
    \sqrt{
      \frac{\lambda_{\max}}{\lambda_{\min}}
    }
    \frac{\alpha_t}{\Vert \mathbf{w}_t \Vert}.
  \end{aligned}
  \end{equation*}
  By above bounds, we have
  \begin{equation}\label{eq:positive_branch_descent_sufficient}
  \begin{aligned}
    (\ref{con_11})
    \ \boldsymbol{\Leftarrow} \ \ & 
    \frac{\lambda_{\max} }{\lambda_{\min} }
    \frac{\eta {\rho}_{t}}{\Vert \mathbf{w}_{t} \Vert}
    \frac{\alpha_t^2}{\Vert \mathbf{w}_t \Vert^{2} }
    \le 
    \frac{\lambda_{\min}}{4}
    \frac{1}{\Vert \mathbf{w}_{t} \Vert_{\mathbf{\Sigma}}}
    \frac{\alpha_{t}}{e^{\alpha_{t}}}
    \left(
      \rho_{t}^{\perp}
    \right)^2 
    \\ \ \boldsymbol{\Leftarrow} \ \ & 
    \frac{\lambda_{\max} }{\lambda_{\min} }
    \frac{\eta {\rho}_{t}}{\Vert \mathbf{w}_{t} \Vert}
    \frac{\alpha_t^2}{\Vert \mathbf{w}_t \Vert^{2} }
    \le 
    \frac{\lambda_{\min}}{4 \sqrt{\lambda_{\max}}}
    \frac{1}{\Vert \mathbf{w}_{t} \Vert}
    \frac{\alpha_{t}}{e^{\alpha_{t}}}
    \left(
      \rho_{t}^{\perp}
    \right)^2 
    \\ \ \boldsymbol{\Leftrightarrow} \ \ & 
    \frac{4\lambda_{\max}^{3/2} }{\lambda_{\min}^2 }
    \frac{\eta}{\Vert \mathbf{w}_{t} \Vert^2}
    \alpha_te^{\alpha_{t}}
    {\rho}_{t}
    \le 
    \left(
      \rho_{t}^{\perp}
    \right)^2 
    \\ \ \boldsymbol{\Leftarrow} \ \ & 
    \frac{6\lambda_{\max}^{3/2} }{\lambda_{\min}^2 }
    \frac{\eta}{\Vert \mathbf{w}_{t} \Vert^2}
    \alpha_0 e^{3\alpha_{0}/2}
    {\rho}_{t}
    \le 
    \left(
      \rho_{t}^{\perp}
    \right)^2 
    \\ \ \boldsymbol{\Leftarrow} \ \ & 
    \frac{6\lambda_{\max}^{3/2} }{\lambda_{\min}^2 }
    \frac{\eta}{\Vert \mathbf{w}_{0} \Vert^2}
    \alpha_0 e^{3\alpha_{0}/2}
    {\rho}_{t}
    \le 
    \left(
      \rho_{t}^{\perp}
    \right)^2 
    \\ \ \boldsymbol{\Leftarrow} \ \ & 
    \frac{6\lambda_{\max}^{2} }{\lambda_{\min}^2 }
    \frac{\eta}{\Vert \mathbf{w}_{0} \Vert^2}
    \alpha_0
    {\rho}_{t}
    \le 
    \left( \rho_{t}^{\perp} \right)^2 
  \end{aligned}
  \end{equation}
  Here we used $\frac{1}{2}\alpha_0 \le \alpha_t \le \frac{3}{2}\alpha_0$ for $t < T_0$, monotonic growth of $\Vert \mathbf{w}_t \Vert$, and $\alpha_0 \le \frac{1}{3}\log(\lambda_{\max})$.
  Then we denote 
  \begin{equation*}
  \begin{aligned}
    \Phi := \frac{6\lambda_{\max}^2 }{\lambda_{\min}^2 }
    \alpha_0 \cdot \frac{\eta \gamma}{\Vert \mathbf{w}_{0} \Vert^2}.
  \end{aligned},
  \end{equation*}
  to simplify the notation. Further, we have 
  \begin{equation*}
  \begin{aligned}
    (\ref{eq:positive_branch_descent_sufficient})
    \ \boldsymbol{\Leftrightarrow} \ \ & 
    \frac{\Phi}{\gamma}
    {\rho}_{t}
    \le 
    \left(
      \rho_{t}^{\perp}
    \right)^2 
    \ \boldsymbol{\Leftrightarrow} \ 
    {\rho}_{t}^2 + \frac{\Phi}{\gamma} {\rho}_{t} - \frac{1}{\gamma^2} \le  0
  \end{aligned}
  \end{equation*}
  Since $\rho_t > 0$, we solve the range of $\rho_t$:
  \begin{equation*}
  \begin{aligned}
    \rho_{t} 
    \le 
    \frac{\sqrt{ \Phi^2 + 4} - \Phi}{2}
    \cdot \frac{1}{\gamma}
    \ \boldsymbol{\Leftrightarrow} \ 
    \left(
      \frac{\rho_t^{\perp}}{\rho_t}
    \right)^2
    \ge 
    \frac{4}{\left(\sqrt{ \Phi^2 + 4} - \Phi\right)^2} - 1
  \end{aligned},
  \end{equation*}
  where the range of $\frac{\rho_t^{\perp}}{\rho_t}$ is by $1/\gamma^2 = \left(\rho_t^{\perp}\right)^2 + \rho_t^2$.
  Therefore, we have
  \begin{equation*}
  \begin{aligned}
    (\ref{con_11}) 
    \ \boldsymbol{\Leftarrow} \
    (\ref{eq:positive_branch_descent_sufficient}) 
    \ \boldsymbol{\Leftrightarrow} \
    \left(
      \frac{\rho_t^{\perp}}{\rho_t}
    \right)^2
    \ge 
    \frac{4}{\left(\sqrt{ \Phi^2 + 4} - \Phi\right)^2} - 1
  \end{aligned}
  \end{equation*}

  \noindent \textbf{Proof of item (3): explicit exit threshold on a rising-edge segment} 
  We now prove the explicit exit threshold for a \textit{Rising Edge} on the positive-alignment branch by checking the
  convergence condition of Lemma \ref{Direction_Convegence_and_Divergence}. 
  To obtain an explicit threshold, we upper bound $\Vert \nabla_{\mathbf{w}} \mathcal{R}_t \Vert$ by $\rho_t^\perp$, 
  using the second upper bound of Lemma \ref{Gradient_Norm_Bound}: 
  \begin{equation*}
  \begin{aligned}
    \left\Vert 
    \nabla_{\mathbf{w}} \mathcal{R}_t
    \right\Vert
    \le 
    \lambda_{\max}
    \frac{\alpha_t}{\Vert \mathbf{w}_t \Vert_{\mathbf{\Sigma}}}
    \left( \alpha_t + 1 \right)
    \rho_t^{\perp}
  \end{aligned}
  \end{equation*}
  Therefore, we have 
  \begin{equation*}
  \begin{aligned}
    (\ref{con_11})
    & \ \boldsymbol{\Leftarrow} \ 
    4
    \frac{\lambda_{\max}^{5/2}}{\lambda_{\min} }
    \frac{\alpha_t e^{\alpha_t} }{\Vert \mathbf{w}_t \Vert_{\mathbf{\Sigma}}^2}
    \left( \alpha_t + 1 \right)^2
    \eta {\rho}_{t}
    \le 
    1
    \\ & \ \boldsymbol{\Leftarrow} \ 
    4
    \frac{\lambda_{\max}^{5/2}}{\lambda_{\min}^2 }
    \frac{\alpha_t e^{\alpha_t} }{\Vert \mathbf{w}_t \Vert^2}
    \left( \alpha_t + 1 \right)^2
    \eta {\rho}_{t}
    \le 
    1
    \\ & \ \boldsymbol{\Leftarrow} \ 
    4
    \frac{\lambda_{\max}^{5/2}}{\lambda_{\min}^2 }
    \frac{\alpha_t e^{\alpha_t} }{\Vert \mathbf{w}_0 \Vert^2}
    \left( \alpha_t + 1 \right)^2
    \eta {\rho}_{t}
    \le 
    1
    \\ & \ \boldsymbol{\Leftarrow} \ 
    6
    \frac{\lambda_{\max}^{5/2}}{\lambda_{\min}^2 }
    \alpha_0 e^{3 \alpha_0 / 2}\left( \frac{3}{2} \alpha_0 + 1 \right)^2
    \frac{\eta {\rho}_{t}}{\Vert \mathbf{w}_0 \Vert^2}
    \le 
    1
    \\ & \ \boldsymbol{\Leftrightarrow} \ 
    6
    \frac{\lambda_{\max}^{5/2}}{\lambda_{\min}^{5/2} }
    \alpha_0 e^{3 \alpha_0 / 2}\left( \frac{3}{2} \alpha_0 + 1 \right)^2
    \frac{\eta \gamma}{\Vert \mathbf{w}_0 \Vert^2} {\rho}_{t}
    \le 
    1
    \\ & \ \boldsymbol{\Leftrightarrow} \ 
    {\rho}_{t}^2
    \le 
    \frac{\Vert \mathbf{w}_0 \Vert^4}{\gamma^2 \eta^2}
    C_6^{-1}
    \\ & \ \boldsymbol{\Leftrightarrow} \ 
    \left(
      \frac{\rho_t^\perp}{\rho_t}
    \right)^2
    \ge
    C_6
    \frac{\eta^2}{\Vert \mathbf{w}_0 \Vert^4}
    \frac{1}{\gamma^2}
    - 1
  \end{aligned}
  \end{equation*}
  Here we used $\Vert \mathbf{w}_t \Vert \ge \Vert \mathbf{w}_0 \Vert$, $\frac{1}{2}\alpha_0 \le \alpha_t \le \frac{3}{2}\alpha_0$ on $[0, T_0)$, $\gamma \ge \sqrt{\lambda_{\min}}$, and the identity $1/\gamma^2 = \rho_t^2 + (\rho_t^\perp)^2$.
  where $C_6$ is 
  \begin{equation*}
  \begin{aligned}
    C_6 := 
    \left(
      6
      \frac{\lambda_{\max}^{5/2}}{\lambda_{\min}^{5/2} }
      \alpha_0 e^{3 \alpha_0 / 2}\left( \frac{3}{2} \alpha_0 + 1 \right)^2
    \right)^2
  \end{aligned}
  \end{equation*}
  Therefore, every iterate on the positive-alignment branch satisfying
  \[
    \left(
      \frac{\rho_t^\perp}{\rho_t}
    \right)^2
    \ge
    C_6
    \frac{\eta^2}{\Vert \mathbf{w}_0 \Vert^4}
    \frac{1}{\gamma^2}
    - 1
  \]
  already meets the convergence condition (\ref{con_11}). Lemma \ref{Direction_Convegence_and_Divergence} then gives $\left(\rho_{t+1}^\perp\right)^2 \le \left(\rho_t^\perp\right)^2$. Since we are on the branch $\rho_t > 0$ and $\rho_t^2 + \left(\rho_t^\perp\right)^2 = 1/\gamma^2$, this implies $\rho_{t+1} \ge \rho_t$ and hence
  \[
    \left(
      \frac{\rho_{t+1}^\perp}{\rho_{t+1}}
    \right)^2
    \le
    \left(
      \frac{\rho_t^\perp}{\rho_t}
    \right)^2.
  \]
  Thus the next iterate leaves the \textit{Rising Edge}, proving item (3).
\end{proof}

\end{appendices}

\bibliography{sn-bibliography}

\end{document}